\documentclass{article}

\usepackage{enumitem}
\usepackage{microtype}
\usepackage{graphicx}
\usepackage{subfigure}
\usepackage{booktabs} 
\usepackage{amsmath,amsfonts,amssymb,amsthm}
\usepackage{comment}
\usepackage[round]{natbib}

\usepackage{url}            
\usepackage{nicefrac}       
\usepackage[dvipsnames]{xcolor}
\usepackage[backref=page,colorlinks=true,linkcolor=RedOrange, citecolor=NavyBlue]{hyperref}
\usepackage{cleveref}
\usepackage[algo2e,ruled]{algorithm2e}
\usepackage{epsf}
\usepackage{graphics}
\usepackage{psfrag}

\usepackage{color}

\usepackage{mathtools}
\usepackage{amsfonts}
\usepackage{amsthm}
\usepackage{amsmath}
\usepackage{amssymb}

\newtheorem{theorem}{Theorem}

\newtheorem{lemma}{Lemma}
\newtheorem{corollary}{Corollary}
\newtheorem{definition}{Definition}
\newtheorem{remark}{Remark}

\newtheorem{assumption}{Assumption}

\long\def\comment#1{}

\newcommand{\bm}[1]{\boldsymbol{#1}}

\usepackage{marvosym} 
\usepackage[english]{babel}
\usepackage[utf8]{inputenc}
\usepackage{appendix}
\usepackage{multirow}
\usepackage{pifont} 
\usepackage{chngcntr}
\usepackage{wrapfig}
\usepackage{enumitem}
\usepackage{natbib}
\usepackage{setspace}

\definecolor{aoenglish}{rgb}{0.0, 0.5, 0.0}
\definecolor{burgundy}{rgb}{0.5, 0.0, 0.13}

\usepackage{pifont}
\usepackage{mathtools}

\makeatletter
\long\def\@makecaption#1#2{
  \vskip 0.8ex
  \setbox\@tempboxa\hbox{\small {\bf #1:} #2}
  \parindent 1.5em  
  \dimen0=\hsize
  \advance\dimen0 by -3em
  \ifdim \wd\@tempboxa >\dimen0
  \hbox to \hsize{
    \parindent 0em
    \hfil 
    \parbox{\dimen0}{\def\baselinestretch{0.96}\small
      {\bf #1.} #2
    } 
    \hfil}
  \else \hbox to \hsize{\hfil \box\@tempboxa \hfil}
  \fi
}
\makeatother

\title{\bf{\LARGE{Adaptive Personalized Federated Learning}}}

\author{Yuyang Deng \qquad Mohammad Mahdi Kamani\qquad Mehrdad Mahdavi \vspace*{.2em} \\ 
The Pennsylvania State University \vspace*{.2em} \\ 
\texttt{ \{yzd82,mqk5591,mzm616\}@psu.edu}
}

\date{}
\begin{document}
\maketitle

\begin{abstract}
Investigation of the degree of personalization in federated learning algorithms has shown that only maximizing the performance of the global model will confine the capacity of the local models to personalize. In this paper, we advocate an adaptive personalized federated learning (APFL) algorithm, where each client will train their local models while contributing to the global model.  We derive the generalization bound of mixture of local and global models, and find the optimal mixing parameter. We also  propose a  communication-efficient  optimization method to collaboratively learn the personalized models and analyze its convergence in both smooth strongly convex and nonconvex settings.  The extensive experiments demonstrate the effectiveness of our personalization schema, as well as the correctness of established generalization theories.
\end{abstract}

\section{Introduction}\label{sec:intro}
With an enormously growing amount of decentralized data continually generated on a vast number of devices like smartphones, \textit{federated learning}  offers training a high-quality shared global model with a central server while reducing the systemic privacy risks and communication costs~\citep{mcmahan2017communication}. Despite the classical approaches, where large-scale datasets are located on massive and expensive data centers for training~\citep{dean2012large,li2014scaling}, in federated learning, the data and training both reside on the local nodes. This will ensure the privacy of the local data, while enabling us to learn from massive, not available otherwise, data on those devices. Not to mention the enormous reduction in communication sizes due to local training and data.

In federated learning, the ultimate goal is to learn a global model that achieves uniformly good performance over almost all participating clients. Motivated by this goal, most of the existing methods pursue the following procedure to learn a global model: (i) a subset of clients participating in the training is chosen at each round and receive the current copy of the global model; (ii) each chosen client updates the local version of the global model using its own local data, (iii) the server aggregates over the obtained local models to update the global model, and this process continues until convergence~\citep{mcmahan2017communication,mohri2019agnostic,karimireddy2019scaffold,pillutla2019robust}. Most notably, federated averaging (FedAvg) by~\cite{mcmahan2017communication} uses averaging as its aggregation method over the local learned models on clients. 

Due to inherent diversity among local data shards and highly non-IID distribution of the data across clients, FedAvg is hugely sensitive to its hyperparameters, and as a result, does not benefit from a favorable convergence guarantee~\citep{haddadpour2019convergence,li2020feddane}. In~\cite{karimireddy2019scaffold}, authors argue that if these hyperparameters are not carefully tuned, it will result in the divergence of FedAvg, as local models may drift significantly from each other. Therefore,  in the presence of statistical data heterogeneity, the global model might not generalize well on the local data of each client individually~\citep{jiang2019improving}. This is even more crucial in fairness-critical systems such as medical diagnosis~\citep{li2019fedmd}, where poor performance on local clients could result in damaging consequences. This problem is exacerbated even further as the diversity among local data of different clients is growing. This is depicted in Figure~\ref{fig:intro}, where the generalization and training losses of the global models of the FedAvg~\citep{mcmahan2017communication} and SCAFFOLD~\citep{karimireddy2019scaffold} on local data diverge when the diversity among different clients' data increases. This observation illustrates that solely optimizing for the global model's accuracy leads to a poor generalization of local clients. To embrace statistical heterogeneity and mitigate the effect of negative transfer, it is necessary to integrate the personalization into learning instead of finding a single consensus predictor. This \textit{pluralistic approach} for federated optimization (FO) has recently resulted in significant research in personalized learning schemes~\citep{eichner2019semi,smith2017federated,dinh2020personalized,mansour2020three,fallah2020personalized,li2020lotteryfl}.

\begin{figure}
    \centering
    \includegraphics[width=0.75\textwidth]{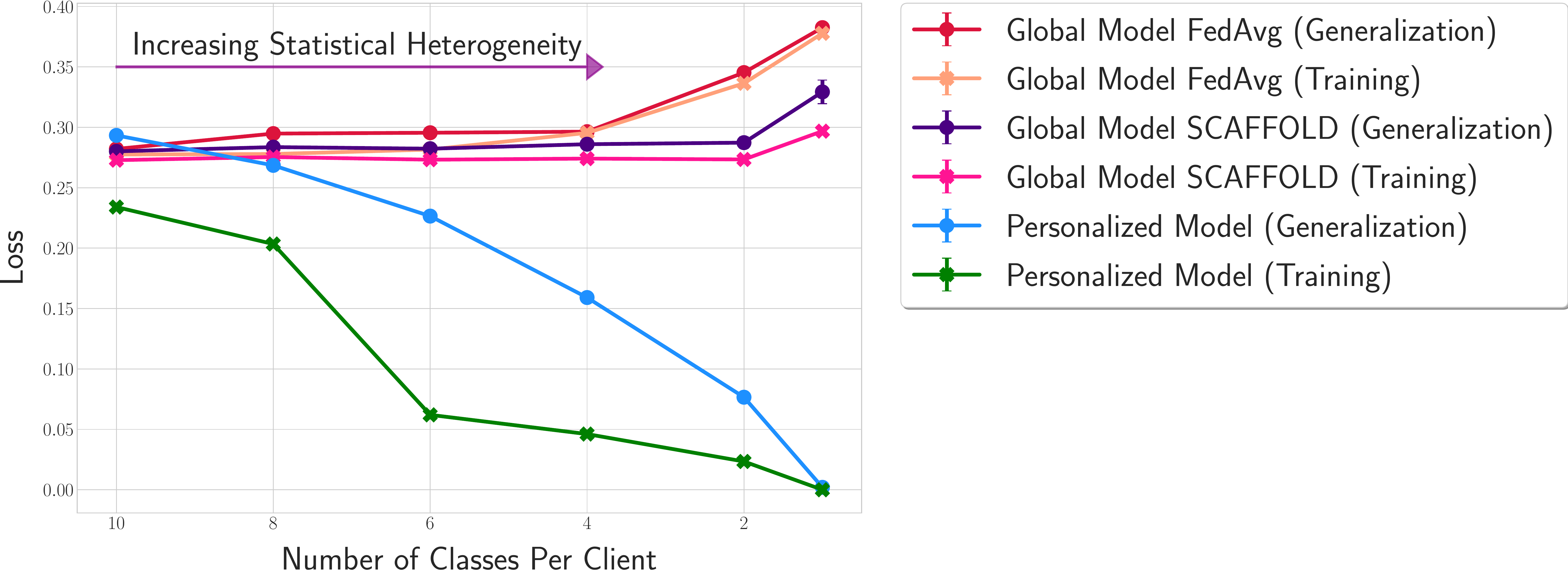}
    \caption{Comparing the generalization and training losses of our proposed personalized model with the global models of FedAvg~\citep{mcmahan2017communication} and SCAFFOLD~\citep{karimireddy2019scaffold}  by increasing the diversity among the data of clients on MNIST dataset with a logistic regression model. Increasing the diversity among local data can lead to a poor generalization performance of global models of FedAvg and SCAFFOLD on local data, while it is diminishing for the proposed personalized model. 
    }
    \label{fig:intro}
\end{figure}

In light of these observations, and to balance the trade-off between the benefit from collaboration with other users and the disadvantage from the statistical heterogeneity among different users' domains, in this paper, we propose an \textbf{adaptive personalized federated learning} (\texttt{APFL}) algorithm which aims to learn a personalized model for each user that is a mixture of optimal local and global models. We theoretically analyze the generalization ability of the personalized model on local distributions, with dependency on mixing parameter, the divergence between local and global distributions, as well as the number of local and global training data. To learn the personalized model, we propose a communication efficient optimization algorithm that adaptively learns the model by leveraging the relatedness between local and global models as learning proceeds. As it is shown in Figure~\ref{fig:intro}, by progressively increasing the diversity, the personalized model found by the proposed algorithm demonstrates a better generalization compared to the global models learned by FedAvg and SCAFFOLD.  We supplement our theoretical findings with extensive corroborating experimental results that demonstrate the superiority of the proposed personalization schema over the global and localized models of commonly used FO algorithms.

\paragraph{Organization.}The rest of the paper is organized as follows. In Section~\ref{sec:related_work}, we review and discuss related work. Section~\ref{sec:generalization} presents our personalized formulation and establish its  generalization guarantees. We formulate the communication-efficient optimization problem in Section~\ref{sec:optimization} and analyze its convergence rate in Section~\ref{sec:convergence}. In Section~\ref{sec:exp} we empirically verify proposed algorithm. Section~\ref{sec:discussion} discusses  some implications of our
results and poses some  questions for future study. We conclude in Section~\ref{sec:conclusion} and defer all the proofs to appendix. 
\section{Related Work}

\label{sec:related_work}
The number of research in federated learning is proliferating during the past few years. In federated learning, the main objective is to learn a global model that is good enough for yet to be seen data and has fast convergence to a local optimum. This indicates that there are several uncanny resemblances between federated learning and meta-learning approaches~\citep{finn2017model,nichol2018first}. However, despite this similarity, meta-learning approaches are mainly trying to learn multiple models, personalized for each new task, whereas in most federated learning approaches, the main focus is on the single global model. As discussed by~\citet{kairouz2019advances}, the gap between the performance of global and personalized models shows the crucial importance of personalization in federated learning. Several different approaches are trying to personalize the global model, primarily focusing on optimization error, while the main challenge with personalization is during the inference time. Some of these works on the personalization of models in a decentralized setting can be found in~\cite{vanhaesebrouck2017decentralized,almeida2018djam}, where in addition to the optimization error, they have network constraints or peer-to-peer communication limitation~\citep{bellet2017personalized,zantedeschi2019fully}. In general, as discussed by~\citet{kairouz2019advances}, there are three significant categories of personalization methods in federated learning, namely, local fine-tuning, multi-task learning, and contextualization.~\cite{yu2020salvaging} argue that the global model learned by federated learning, especially with having differential privacy and robust learning objectives, can hurt the performance of many clients. They indicate that those clients can obtain a better model by using only their own data. Hence, they empirically show that using these three approaches can boost the performance of those clients.
In addition to these three, there is also another category that fits the most to our proposed approach, which is mixing the global and local models.
\vspace{-0.25cm}\paragraph{Local fine-tuning:}
The dominant approach for personalization is local fine-tuning, where each client receives a global model and tune it using its own local data and several gradient descent steps. This approach is predominantly used in meta-learning methods such as MAML by~\cite{finn2017model} or domain adaptation and transfer learning~\citep{ben2010theory,mansour2009domain,pan2009survey}. \citet{jiang2019improving} discuss the similarity between federated learning and meta-learning approaches, notably the Reptile algorithm by~\cite{nichol2018first} and FedAvg, and combine them to personalize local models. They observed that federated learning with a single objective of performance of the global model could limit the capacity of the learned model for personalization. In~\cite{khodak2019adaptive}, authors using online convex optimization to introduce a meta-learning approach that can be used in federated learning for better personalization.~\citet{fallah2020personalized} borrow ideas from MAML to learn personalized models for each client with convergence guarantees. Similar to fine-tuning, they update the local models with several gradient steps, but they use second-order information to update the global model, like MAML. Another approach adopted for deep neural networks is introduced by~\cite{arivazhagan2019federated}, where they freeze the base layers and only change the last ``personalized'' layer for each client locally.
The main drawback of local fine-tuning is that it minimizes the optimization error, whereas the more important part is the generalization performance of the personalized model. In this setting, the personalized model is pruned to overfit.
\vspace{-0.25cm}\paragraph{Multi-task learning:}
Another view of the personalization problem is to see it as a multi-task learning problem similar to~\cite{smith2017federated}. In this setting, optimization on each client can be considered as a new task; hence, the approaches of multi-task learning can be applied. One other approach, discussed as an open problem in~\citet{kairouz2019advances}, is to cluster groups of clients based on some features such as region, as similar tasks, similar to one approach proposed by~\cite{mansour2020three}.
\vspace{-0.25cm}\paragraph{Contextualization:} An important application of personalization in federated learning is using the model under different contexts. For instance, in the next character recognition task in~\cite{hard2018federated}, based on the context of the use case, the results should be different. Hence, we need a personalized model on one client under different contexts. This requires access to more features about the context during the training. Evaluation of the personalized model in such a setting has been investigated by~\citet{wang2019federated}, which is in line with our approach in experimental results in Section~\ref{sec:exp}. \cite{liang2020think} propose to directly learn the feature representation locally, and train the discriminator globally, which reduces the effect of data heterogeneity and ensures the fair learning.

\vspace{-0.25cm}\paragraph{Personalization via model regularization:} Another significant trial for personalization is model regularization. There are several studies to introduce different personalization approaches for federated learning by regularize the difference between the global and local models.~\citet{hanzely2020federated} try to introduce a new formulation for federated learning where they add the regularization term on the distance of local and global models. In their effort, they use a mixing parameter, which controls the degree of optimization for both local models and the global model. The FedAvg~\citep{mcmahan2017communication} can be considered a special case of this approach. They show that the learned model is in the convex haul of both local and global models, and at each iteration, depend on the local models' optimization parameters, the global model is getting closer to the global model learned by FedAvg. Similarly, \cite{huang2020personalized,dinh2020personalized} also propose to use the regularization between local and global model, to realize the personalized learning. \cite{shen2020federated} propose a knowledge distillation way to achieve personalization, where they apply the regularization on the predictions between local model and global model. The extra benefit of their method is that they can solve the model heterogeneity issue in federated learning.

\vspace{-0.25cm}\paragraph{Personalization via model interpolation:} Parallel to our work, there are other studies to introduce different personalization approaches for federated learning by mixing the global and local models. 
The closest approach for personalization to our proposal is introduced by~\citet{mansour2020three}. In fact, they propose three different approaches for personalization with generalization guarantees, namely, client clustering, data interpolation, and model interpolation. Out of these three, the first two approaches need some meta-features from all clients that makes them not a feasible approach for federated learning, due to privacy concerns. The third schema, which is the most promising one in practice as well, has a close formulation to ours in the interpolation of the local and global models. However, in their theory, the generalization bound does not demonstrate the advantage of mixing models, but in our analysis, we will show how the model mixing can impact the generalization bound, by presenting its dependency on the mixture parameter, data diversity and optimal models on local and global distributions.

Beyond different techniques for personalization in federated learning, \citet{kairouz2019advances} ask an essential question of ``\textit{when is a global FL-trained model better?}'', or as we can ask, when is personalization better? The answer to these questions mostly depends on the distribution of data across clients. As we theoretically prove and empirically verify in this paper, when the data is distributed IID, we cannot benefit from personalization, and it is similar to the local SGD scenario~\citep{stich2018local,haddadpour2019local,haddadpour2019trading,woodworth2020local}. However, when the data is non-IID across clients, which is mostly the case in federated learning, personalization can help to balance between shared and local knowledge. Then, the question becomes, what degree of personalization is best for each client? While this was an open problem in~\cite{mohri2019agnostic} on how to appropriately mix the global and local model, we answer this question by adaptively tuning the degree of personalization for each client, as discussed in Section~\ref{sec:alpha_adap}, so it can perfectly become agnostic to the local data distributions. 
\section{Personalized Federated Learning}
\label{sec:generalization}
In this section, we propose a personalization approach for federated learning and analyze its statistical properties. Following the statistical learning theory, in a federated learning setting each client has access to its own data distribution $\mathcal{D}_i$ on domain $\Xi := \mathcal{X} \times \mathcal{Y}$, where $\mathcal{X} \in \mathbb{R}^d$ is the input domain and $\mathcal{Y}$ is the label domain. For any hypothesis $h \in \mathcal{H}$ the loss function is defined as $\ell:\mathcal{H}\times \Xi \rightarrow \mathbb{R}^+$. The true risk at local distribution  is denoted by $\mathcal{L}_{\mathcal{D}_i} (h) = \mathbb{E}_{(\bm{x}, y) \sim \mathcal{D}_i}\left[\ell\left(h(\bm{x}), y\right)\right]$. We use $\hat{\mathcal{L}}_{\mathcal{D}_i} (h)$ to denote the empirical risk of $h$ on distribution $\mathcal{D}_i$. We use $\bar{\mathcal{D}} = ({1}/{n})\sum_{i=1}^n \mathcal{D}_i$ to denote the average distribution over all clients. Intrinsically, as in federated learning, the global model is trained to minimize the empirical (i.e., ERM) loss with respect to distribution $\bar{\mathcal{D}}$, i.e.,  $\min_{h\in \mathcal{H}} \hat{\mathcal{L}}_{\bar{\mathcal{D}}} (h)$.

\subsection{Personalized model}\label{coupled learning}
In a standard federated learning scenario, where the goal is to learn a global model for all devices cooperatively, the learned global model obtained by minimizing the joint empirical distribution, either by proper weighting or in an agnostic manner, may not perfectly generalize on local users' data when the heterogeneity among local data shards is high (i.e., the global and local optimal models might drift significantly). However, by assuming that all users' data come from the (roughly) similar distribution, it is expected that the global model enjoys a better generalization accuracy on any user distribution over its domain than the user's own local model. Meanwhile, from the local user perspective, the key incentive to participate in ``federated''  learning is the desire to seek a reduction in the local generalization error with the help of other users' data. In this case, the ideal situation would be that the user can utilize the information from the global model to compensate for the small number its local training data while minimizing the harm induced by heterogeneity among each user's local data and the data shared by other devices. Obviously, when the local distribution is highly correlated with global distribution, the global model is preferable; otherwise, the global model might be ineffective to be employed as the local model.  This motivates us to mix the global model and local model with an adaptive weight as a joint prediction model, namely, the personalized model. 

In the adaptive personalized federated learning the goal is to find the optimal combination of the global model and the local model, in order to achieve a better client-specific model. In this setting, global server still tries to train the global model by minimizing the empirical risk on the aggregated domain $\bar{\mathcal{D}}$:
\begin{equation} 
\bar{h}^* = \arg \min_{h\in \mathcal{H}} \hat{\mathcal{L}}_{\bar{\mathcal{D}}} (h),\nonumber
\end{equation}
while each user trains a local model while incorporating part of the global model, with some mixing weight $\alpha_i$, i.e.,
\begin{equation} 
\hat{h}_{loc,i}^* = \arg \min_{h\in \mathcal{H}} \hat{\mathcal{L}}_{\mathcal{D}_i} ( \alpha_i h + (1-\alpha_i) \bar{h}^*),\nonumber
\end{equation}
Finally, the personalized model for $i$-th client is a convex combination of $\bar{h}^*$ and $\hat{h}_{loc,i}^*$:
\begin{equation}\label{eqn:personalized}
 h_{\alpha_i} = \alpha_i \hat{h}_{loc,i}^* + (1-\alpha_i) \bar{h}^* ,
\end{equation}
 It is worth mentioning that, $h_{\alpha_i}$ is not necessarily the minimizer of empirical risk $\hat{\mathcal{L}}_{\mathcal{D}_i}(\cdot)$, because we optimize $\hat{h}_{loc,i}^*$ with partially incorporating the global model. In fact, in most cases, as we will show in the convergence of the proposed algorithm, $h_{\alpha_i}$ will incur a residual risk if evaluated on the training set drawn from $\mathcal{D}_i$.

\subsection{Generalization guarantees}\label{sec:gen}
We now characterize the generalization of  the mixed model. We present the learning bounds for classification and regression tasks. For classification, we consider a binary classification task, with squared hinge loss $\ell(h(\bm{x}),y) = \left(\max\{0,1-yh(\bm{x})\} \right)^2$. In the regression task, we consider the MSE loss $\ell(h(\bm{x}),y) = (h(\bm{x})-y)^2$. Even though we present learning bounds under these two loss functions, our analysis can be generalized to any convex smooth loss. Before formally presenting the generalization bound, we  introduce the following quantity  to measure the empirical complexity of a hypothesis class $\mathcal{H}$ over a training set $\mathcal{S}$.
\begin{definition}
Let $S$ be a fixed set of samples and consider a hypothesis class $\mathcal{H}$. The worst case   disagreement between a pair of models measured by absolute loss is quantified by
\begin{align}
    \lambda_{\mathcal{H}} (S) = \sup_{h,h' \in \mathcal{H}} \frac{1}{|S|}\sum_{(\bm{x},y)\in S} |h(\bm{x})-h'(\bm{x})|.
\end{align}
\end{definition}
\begin{remark}
This quantity measures the complexity of a hypothesis class, by computing the maximum disagreement between two hypotheses on a sample training set.  A similar quantity is also employed in the related multiple source PAC learning or domain adaption studies~\cite{kifer2004detecting,mansour2009domain,ben2010theory,konstantinov2020sample,zhang2020localized}. We will show how this term impact our generalization bound.
\end{remark}
Equipped with this discrepancy notion, we now state the main result on the generalization of the proposed personalization schema. 
\begin{theorem}\label{thm:generalization}
Let $\mathcal{H}$ be a hypothesis class with finite VC dimension $d$. Assume loss function $\ell$ is Lipschitz continuous with constant $G$, and bounded in $[0, B]$. Then with probability at least $1-\delta$, there exists a constant $C$, such that the risk of the mixed model $h_{\alpha_i} = \alpha_i \hat{h}_{loc,i}^* + (1-\alpha_i) \bar{h}^*$ on the  $i$th local distribution $\mathcal{D}_i$ is bounded by:
{\begin{equation}
\label{eq: generalization bound-main}
\begin{aligned}
\mathcal{L}_{\mathcal{D}_i}(h_{\alpha_i})  &\leq  2(1-\alpha_i)^2\left(\hat{\mathcal{L}}_{ \bar{\mathcal{D}}}(\bar{h}^*)+B\|\bar{\mathcal{D}}-\mathcal{D}_i\|_1 +
C \sqrt{\frac{d+ \log (1/\delta)}{m}}\right) \nonumber \\ & \quad +2\alpha_i^2\left(\mathcal{L}_{\mathcal{D}_i}( h_{i}^* ) + 2C \sqrt{\frac{d+ \log (1/\delta)}{m_i}} + G\lambda_{\mathcal{H}}(\mathcal{S}_i) \right) , 
\end{aligned}
\end{equation}}
where $m_i, i=1, 2, \ldots, n$ is the number of training data at  $i$th user,  $m = m_1 + \ldots + m_n$ is the total number of all data, $\mathcal{S}_i$ to be the local training set drawn from $\mathcal{D}_i$, $\|\bar{\mathcal{D}}-\mathcal{D}_i\|_1 = \int_{\Xi} |\mathbb{P}_{(\bm{x}, y)\sim \bar{\mathcal{D}}}-\mathbb{P}_{(\bm{x}, y)\sim\mathcal{D}_i}|d\bm{x}dy$,  is the difference between distributions $\bar{\mathcal{D}}$ and $\mathcal{D}_i$, and $h_i^* = \arg \min_{h\in \mathcal{H}} \mathcal{L}_{\mathcal{D}_i} (h)$. 
\end{theorem}
\begin{proof}
The proof is provided in Appendix~\ref{app:proof_gen}.
\end{proof}

Theorem~\ref{thm:generalization} shows that  the generalization risk of $h_{\alpha_i}$ on $\mathcal{D}_i$ mainly depends on following key  quantities: i) $m$: the amount of global data drawn from $ \bar{\mathcal{D}}$, ii) divergence between distributions $\bar{\mathcal{D}}$ and $\mathcal{D}_i$ measured by absolute distance, and iii) $m_i$: the amount of local data drawn from $  \mathcal{D}_i$. Usually, the first quantity $m$, the amount of global data is fairly large compared to individual users, so global model usually has better generalization. The second quantity characterizes the data heterogeneity between the average distribution and $i$th local distribution. If this divergence is too high, then the global model may hurt the local generalization. For the third quantity, as amount of local data $m_i$ is often relatively small,  the generalization performance of local model can be very poor as well; hence, it should choose a small $\alpha_i$ to incorporate more proportion of the global model.  

\paragraph{Optimal mixing parameter.}We can also find  the optimal mixing parameter $\alpha_i^*$ that minimizes generalization bound in Theorem~\ref{thm:generalization}. Notice that the RHS of~(\ref{eq: generalization bound-main})  is quadratic in $\alpha_i$, so it admits a minimum value at 
{\begin{align}
    \alpha_i^* = \frac{ \left(\hat{\mathcal{L}}_{ \bar{\mathcal{D}}}(\bar{h}^*)+B\|\bar{\mathcal{D}}-\mathcal{D}_i\|_1 +
 C\sqrt{\frac{d+\log (1/\delta)}{m}}\right)}{  \left(\hat{\mathcal{L}}_{ \bar{\mathcal{D}}}(\bar{h}^*)+B\|\bar{\mathcal{D}}-\mathcal{D}_i\|_1 +
C\sqrt{\frac{d+\log (1/\delta)}{m}}\right)+  \left(\mathcal{L}_{\mathcal{D}_i}( h_{i}^* ) +  2C \sqrt{\frac{d+ \log (1/\delta)}{m_i}} + G\lambda_{\mathcal{H}}(\mathcal{S}_i) \right)}
\end{align}}
The optimal mixture parameter is strictly bounded in $[0,1]$, which matches our intuition.  If the divergence between the average distribution $\bar{\mathcal{D}}$ and $\mathcal{D}_i$ is large, then the value becomes close to 1, which implies if local distribution drifts too much from average distribution,  taking a majority of the global model is not an effective  choice, and it is preferable to take more local models. If $m_i$ is small, this value will be negligible, indicating that we need to mix more of the global model into the personalized model. Conversely, if $m_i$ is large, then this term will be again  roughly 1, which means taking the majority of local model will give the desired generalization performance.  

\begin{remark}
As $\mathcal{L}_{\mathcal{D}_i}\left( \hat{h}_{loc,i}^* \right)$ is the risk of the empirical risk minimizer on $\mathcal{D}_i$ after incorporating a model learned on a different domain (i.e., global distribution), one might argue that generalization techniques established in multi-domain learning theory~\citep{ben2010theory,mansour2009domain} can be utilized to serve our purpose. However, we note that the techniques developed in~\cite{ben2010theory,mansour2009domain} are only applicable to a settings where we aim at directly learning a model in some combination of source and target domain, while in our setting, we partially incorporate the model learned from source domain and then perform ERM on joint model over target domain.  Moreover, their results only apply to very simple loss functions, e.g., absolute loss or MSE loss, while we consider squared hinge loss in the classification case. Analogous to multiple domain theory, we derive the multi domain learning bound based on the divergence of source and target domains but measured in absolute distance, $\|\cdot\|_1$. As~\cite{mansour2009domain} points out,   divergence measured by absolute loss can be large, and as a result we leave the development of  a more general multiple domain learning theory that can deal with most popular loss functions like hinge loss, cross entropy loss and optimal transport, with tighter divergence measure on distributions as an open question.
\end{remark}

\begin{remark} 
We note that a very analogous  work to ours is~\cite{mansour2020three}, where a generalization bound is provided for mixing global and local models. However, their bound does not depend on $\alpha_i$, and hence we cannot see the advantage of personalizing schema.
\end{remark}

\section{Optimization Method}
\label{sec:optimization}
The proposed personalized model is rooted in adequately mixing the optimal global and slightly modified local empirical risk minimizers. Also, as it is revealed by generalization analysis, the per-device mixing parameter $\alpha_i$ is a key quantity for the generalization ability of the mixed model. In this section, we propose a communication efficient adaptive algorithm to learn the personalized local models and the global model. 

To do so, we let every hypothesis $h$ in the hypothesis space $\mathcal{H}$ to be parameterized by a vector $\bm{w} \in \mathbb{R}^d$ and denote the empirical risk at $i$th device by local objective function $f_i(\boldsymbol{{w}})$. Adaptive personalized federated learning can be formulated as a two-phase optimization problem: globally update the shared model, and locally update users' local models. Similar to FedAvg algorithm, the server will solve the following optimization problem: 
\begin{align}
\min_{\bm{w}\in \mathbb{R}^d} \left[F(\boldsymbol{{w}}) := \frac{1}{n} \sum_{i=1}^n \left\{f_i(\boldsymbol{{w}}) := \mathbb{E}_{{\xi}_i}\left[f_i\left( \bm{w},\xi_i\right)\right]\right\}\right],
\end{align}
where $f_i(.)$ is the local objective at $i$th client, $\xi_i$ is a minibatch of data in data shard at  $i$th client, and $n$ is the total number of clients.

Motivated by the trade-off between the global model and local model generalization errors in Theorem~\ref{thm:generalization}, we need to learn a personalized model as in~(\ref{eqn:personalized}) to optimize the local empirical risk. To this end, each client needs to solve the following optimization problem over its local data:
\begin{align}
\min_{\bm{v} \in \mathbb{R}^d}  f_i\left(\alpha_i\bm{v} +(1-\alpha_i)\bm{w}^*\right), 
\label{eq:local_obj}
\end{align}
where $\bm{w}^* = \arg \min_{\bm{w}} F(\bm{w})$ is the optimal global model. The balance between these two models is governed by a parameter $\alpha_i$, which is associated with the diversity of the local model and the global model. In general, when the local and global data distributions are well aligned, one would intuitively expect that the optimal choice for the mixing parameter would be small to gain more from the data of other devices. On the other side, when local and global distributions drift significantly, the mixing parameter needs to be closed to one to reduce the contribution from the data of other devices on the optimal local model. In what follows, we propose a local descent approach to optimize both objectives simultaneously.
\begin{algorithm2e}[t]
	\DontPrintSemicolon
    \caption{Local Descent \texttt{APFL}}
	\label{algorithm:LDAPFL}
	\textbf{input:} Mixture weights $\alpha_1,\cdots, \alpha_n$, Synchronization gap $\tau$, A set of randomly selected $K$ clients $U_0$, Local models $\bm{v}^{(0)}_{i}$ for $i\in [n]$ and initial local version of  global model $\bm{w}^{(0)}_{i}$ for $i\in [n]$.
	\\
	\For{$t=0,\cdots,T$}{
	\textbf{parallel} \For{$i \in U_t$}{
	    \eIf{t not divides $\tau$}{
	    ${\bm{w}}^{(t)}_i = {\bm{w}}^{(t-1)}_i - \eta_t \nabla f_i\left({\bm{w}}^{(t-1)}_i;\xi_i^t\right)$  \\
	    $\bm{v}^{(t)}_{i} = \bm{v}^{(t-1)}_i - \eta_t \nabla_{\bm{v}} f_i\left(\bar{\bm{v}}_i^{(t-1)}; \xi_i^t\right)$ \\
	    $ \bar{\bm{v}}_i^{(t)} = \alpha_i\bm{v}^{(t)}_{i} + (1-\alpha_i){\bm{w}}^{(t)}_i$\\
	    $U_t \xleftarrow{} U_{t-1}$
	    }{
	    { each selected client sends $\bm{w}_i^{(t)}$ to the server\\
	   
	    ${\bm{w}}^{(t)} = \frac{1}{|U_t|} \sum_{j\in U_t} {\bm{w}}^{(t)}_j $  \\ 
	    server uniformly samples a subset $U_t$ of $K$ clients.\\
	    server broadcast ${\bm{w}}^{(t)}$ to all chosen clients\\
	   }
	  }
	
	    }
	 \For{$i \notin U_t$}{
	    $\bm{v}^{(t)}_{i} = \bm{v}^{(t-1)}_{i}$
	}
  }
tin    \For{$i = 1,\ldots, n$}{ \textbf{output:} Personalized model: $\hat{\boldsymbol{v}}_i = \frac{1}{S_T} \sum_{t=1}^T p_t (\alpha_i\bm{v}_i^{(t)}+(1-\alpha_i)\frac{1}{K}\sum_{j\in U_t}\bm{w}_j^{(t)})$; \\ \qquad \qquad \ Global model: 
    $\hat{\bm{w}}  = \frac{1}{KS_T} \sum_{t=1}^T p_t   \sum_{j\in U_t} \boldsymbol{{w}}_j^{(t)}$.
    }
\end{algorithm2e} 
\subsection{Local Descent \texttt{APFL}}
In this subsection we propose our bilevel optimization algorithm, Local Descent \texttt{APFL}. At each communication round, server uniformly random selects $K$ clients as a set $U_t$. Each selected client will maintain three models at iteration $t$: local version of the global model $\boldsymbol{{w}}_i^{(t)}$, its own local model $\boldsymbol{v}_i^{(t)}$, and the mixed personalized model $\bar{\bm{v}}_i^{(t)} = \alpha_i \boldsymbol{v}_i^{(t)} + (1-\alpha_i) \boldsymbol{{w}}_i^{(t)}$. Then,  selected clients will perform the following updates locally on their own data for $\tau$ iterations:
\begin{align}
     {\bm{w}}^{(t)}_i &= \bm{w}^{(t-1)}_i - \eta_t \nabla f_i\left({\bm{w}}^{(t-1)}_i; \xi_i^t\right)  \\
	  \bm{v}^{(t)}_{i} &= \bm{v}^{(t-1)}_i - \eta_t \nabla_{\bm{v}} f_i\left(\bar{\bm{v}}_i^{(t-1)}; \xi_i^t\right),
\end{align}
where $\nabla f_i\left(.; \xi\right)$ denotes the stochastic gradient of $f(.)$ evaluated at mini-batch $\xi$. Then, using the updated version of the global model and the local model, we update the personalized model $\bar{\bm{v}}_i^{(t)}$ as well. The clients that are not selected in this round will keep their previous step local model $\bm{v}^{(t)}_{i} = \bm{v}^{(t-1)}_{i}$. After these $\tau$ local updates, selected clients will send their local version of the global model ${\bm{w}}^{(t)}_{i}$ to the server for aggregation by averaging:
\begin{align}
     {\bm{w}}^{(t)} = \frac{1}{|U_t|} \sum_{j\in U_t} {\bm{w}}^{(t)}_j.
\end{align}
Then the server will choose another set of $K$ clients for the next round of training and broadcast this new model to them. This process continues until convergence.

\subsection{Adaptive $\alpha$ update}\label{sec:alpha_adap}
Even though in Section~\ref{sec:gen}, we give the information theoretically optimal mixing parameter, in practice we usually do not know the distance between user's distribution and the average distribution. Thus, finding the optimal $\alpha$ is infeasible. However, we can infer it empirically during optimization. \begin{comment}{As it was discussed before and can be inferred from the empirical results in Section~\ref{sec:exp}, the optimum value for $\alpha$ depends on how the data are distributed in different clients, and how far is each local model from the global averaged model. When the data across different clients are highly non-IID, and the diversity among the models are high, we need to have a higher $\alpha$ value, so each client can mostly depend on its own data while benefiting from the knowledge shared by others. On the other hand, when the data are distributed uniformly random across clients, the local model of each client should be very close to the global model, and hence, the $\alpha$ value should be very close to zero. Motivating by this idea, we want to update $\alpha$ values during the training and for each client individually so they can find a better value based on the objective of the defined personalized model.}\end{comment} Based on the local objective defined in (\ref{eq:local_obj}), the empirical optimum value of $\alpha$ for each client can be found by solving $   \alpha_i^* = \arg\min_{\alpha_i\in [0,1]}  f_i\left(\alpha_i\bm{v}+(1-\alpha_i)\boldsymbol{{w}}\right)$, where we can use  an step of gradient descent to update it at every communication round:
\begin{equation}
  \alpha_i^* = \arg\min_{\alpha_i\in [0,1]}  f_i\left(\alpha_i\bm{v}+(1-\alpha_i)\boldsymbol{{w}}\right),
  \label{eq:alpha_obj}\vspace{-0.1cm}
\end{equation}
where we can use the gradient descent to optimize it at every communication round, using the following step:
\begin{align}\label{eq:alpha_gd}
    \alpha_i^{(t)} &= \alpha_i^{(t-1)} - \eta_t \nabla_{\alpha}f_i\left(\bar{\bm{v}}_i^{(t-1)};\xi_i^t\right)= \alpha_i^{(t-1)} - \eta_t \left\langle\bm{v}_i^{(t-1)} - \boldsymbol{{w}}_i^{(t-1)}, \nabla  f_i\left(\bar{\bm{v}}_i^{(t-1)};\xi_i^t\right)\right\rangle,
\end{align}
which interestingly shows that the mixing coefficient $\alpha$ is updated based on the correlation between the difference of the personalized and the local version global models, and the gradient at the in-device personalized model. It indicates that, when the global model is drifting from the personalized model, the value of $\alpha$  changes to adjust the balance between local data and shared knowledge among all devices captured by the global model. Obviously, when personalized and global models are very close to each other (IID data), $\alpha$ value does not change that much.

\section{Convergence Analysis}
\label{sec:convergence}

In this section we provide the convergence analysis of \texttt{APFL} with fixed $\alpha$ on strongly convex and nonconvex function respectively. In order to get a tight analysis, as well as putting  the optimization results in the context of generalization bounds discussed above,  we define the following parameterization-invariant quantities:
\begin{definition} [Gradient Diversity]  We define the following quantity to measure the diversity among local gradients with respect to the gradient of the $i$th client:
\begin{align}
    &\zeta_i = \sup_{\bm{w} \in \mathbb{R}^d}  \|\nabla F (\bm{w}) - \nabla f_i (\bm{w})\|_2^2. \nonumber
\end{align}
We also define the sum of gradient diversities of $n$ clients as: $\zeta = \sum_{i=1}^n\zeta_i. $
\label{gradient divergence}
\end{definition}
Definition~\ref{gradient divergence} is the classic notion characterizing the data heterogeneity across $n$ local functions, which is also employed in other local SGD analysis~\cite{woodworth2020minibatch}. As \cite{woodworth2020minibatch} points out, this quantity will be zero if and only if all local functions are identical.

In addition, for the analysis of strongly convex case, we further need the following quantity which also reflects the heterogeneity:
\begin{definition}[Local-Global Optimality Gap] We define the following quantity to measure the gap between optimal local model and optimal global model:
\begin{equation} \label{local-global gap}
    \Delta_i= \|\bm{v}_i^* - \bm{w}^*\|_2^2,
\end{equation}
where $\bm{v}_i^* = \arg \min_{\bm{v}} f_i(\bm{v})$ is the optimal local model at $i$th client, and $\bm{w}^* = \arg \min_{\bm{w}} F(\bm{w})$ is the global optimal. 
\end{definition}
We note that these two quantities only depend on the distributions of local data across clients and the geometry of loss functions.

We also need the following standard assumptions about the analytical properties of the loss functions to obtain convergence theory:
\begin{assumption}
[Smoothness]\label{assumption: smoothness}
There exists a $L > 0$ such that $\forall i \in [n]$,
\begin{equation}
\|\nabla f_i(\bm{x}) - \nabla f_i(\bm{y})\| \leq L \|\bm{x}-\bm{y}\|, \quad \forall \bm{x}, \bm{y} \in \mathbb{R}^d.
\end{equation}
\end{assumption}
\begin{assumption}[Bounded Variance]\label{assumption: bounded grad}
The variance of stochastic gradients computed at each local data shard is bounded, i.e.,  $\forall i \in [n]$:
\begin{equation}
    \mathbb{E}[\|\nabla f_i (\bm{x};\xi) - \nabla f_i(\bm{x})\|^2]  \leq \sigma^2. 
\end{equation}
\end{assumption}

\subsection{Strongly Convex Loss}
In this section, we will present the convergence analysis of local descent \texttt{APFL} on smooth strongly convex functions: 

\begin{assumption}[Strong Convexity]\label{assumption: strong convexity}
 There exists a $\mu > 0$ such that $\forall i \in [n]$,
\begin{equation}
f_i(\bm{x}) \geq f_i(\bm{y}) + \langle \nabla f_i(\bm{y}), \bm{y}-\bm{x} \rangle + \frac{\mu}{2} \|\bm{x}-\bm{y}\|^2, \quad \forall \bm{x}, \bm{y} \in \mathbb{R}^d.
\end{equation}
\end{assumption}

\paragraph{Technical challenges.}{The analysis of convergence rates in our setting is more involved compared to analysis of local SGD with periodic averaging by~\cite{stich2018local,woodworth2020minibatch}. The key difficulty arises from the fact that unlike local SGD where local  solutions are evolved by employing mini-batch SGD, in our setting we also partially incorporate the global model to compute  stochastic gradients over local data. In addition, our goal is to find the convergence rate of the mixed model, rather than merely the local model or global model. To better illustrate this, let us first clarify the notations of models that will be used in analysis. Let us consider the simple case for now where we set $K= n$ (all device participate averaging).  We define three virtual sequences: $\{\bm{w}^{(t)}\}_{t=1}^T$, $\{\bar{\bm{v}}^{(t)}\}_{t=1}^T$ and $\{\hat{\bm{v}}^{(t)}\}_{t=1}^T$ where $\bm{w}^{(t)} = \frac{1}{n} \sum_{j=1}^n \bm{w}_i^{(t)}$,$\bar{\bm{v}}_i^{(t)} = \alpha_i \bm{v}_i^{(t)} + (1-\alpha_i) \bm{w}^{(t)}_i$ $\hat{\bm{v}}_i^{(t)} = \alpha_i \bm{v}_i^{(t)} + (1-\alpha_i) \bm{w}^{(t)}$. Since the personalized model incorporates $1-\alpha_i$ percentage of global model, then the key challenge in the convergence analysis is to find out how much the global model benefits/hurts the local convergence. To this end, we analyze how much the dynamics of personalized model $\hat{\bm{v}}_i^{(t)}$ and global model $\bm{w}^{(t)}$ differ from each other at each iteration. To be more specific, we study the distance between gradients $\|\nabla f_i(\hat{\bm{v}}_i^{(t)}) - \nabla F(\bm{w}^{(t)})\|^2$. Surprisingly, we relate this distance to gradient diversity, personalized model convergence, global model convergence and local-global optimality gap:
 {\begin{equation}
    \mathbb{E}\left[\|\nabla f_i(\hat{\bm{v}}_i^{(t)}) - \nabla F(\bm{w}^{(t)})\|^2\right] \leq  6\zeta_i + 2L^2\mathbb{E}\left[\|\hat{\bm{v}}_i^{(t)} - \bm{v}^*\|^2\right] + 6L^2\mathbb{E}\left[\|\bm{w}^{(t)} - \bm{w}^*\|^2\right] +  6L^2\Delta_i.\nonumber
\end{equation}}
$\mathbb{E}\left[\|\hat{\bm{v}}_i^{(t)} - \bm{v}^*\|^2\right]$ and $  \mathbb{E}\left[\|\bm{w}^{(t)} - \bm{w}^*\|^2\right]$ will converge very fast under smooth strongly convex objective, and $\zeta_i$  and $\Delta_i$ will serve as residual error that indicates the heterogeneity among local functions.}

The following theorem establishes the convergence of global model. We note that we state the convergence rate in terms of key parameters and hide constants in $O(\cdot)$ notation for ease of discussion  and defer the detailed analysis to appendix.
\begin{theorem}[Global model convergence of Local Descent \texttt{APFL}]
\label{Thm: Global Convergence w sampling}
If each client's objective function satisfies Assumptions~\ref{assumption: smoothness}, \ref{assumption: bounded grad}, \ref{assumption: strong convexity}, using Algorithm \ref{algorithm:LDAPFL}, by choosing the learning rate $\eta_t = \frac{16}{\mu(t+a)}$, where $a = \max\{128\kappa,\tau\}$, $\kappa = \frac{L}{\mu}$, and using average scheme $\hat{\bm{w}}  = \frac{1}{KS_T} \sum_{t=1}^T p_t   \sum_{j\in U_t} \boldsymbol{{w}}_j^{(t)}$, where $p_t = (t+a)^2$, $S_T = \sum_{t=1}^T p_t$, and letting $F^*$ to denote  the minimum of the $F$, then the following convergence holds:
\begin{align}
      \mathbb{E}\left[F(\bm{\hat{w}})\right] 
    -   F^*  \leq  O\left(\frac{\mu\mathbb{E}\left[\|\bm{w}^{(1)} - \bm{w}^*\|^2\right]}{T^3}  \right) +O\left(\frac{\kappa^2 \tau\left(\sigma^2 + 2\tau \frac{\zeta}{K}  \right)}{\mu T^2}\right) +O\left(\frac{\kappa^2\tau\left(\sigma^2 + 2\tau \frac{\zeta}{K}  \right)\ln T}{\mu T^3}\right)  + O\left(\frac{\sigma^2}{KT}\right),
\end{align}
where $\tau$ is the number of local updates (i.e., synchronization gap) .
\end{theorem}
\begin{proof}
The proof is provided in Appendix~\ref{proof thm2}.
\end{proof}
\begin{remark}
It is noticeable that the obtained rate matches the convergence rate of the FedAvg, and if we choose $\tau = \sqrt{T/K}$, we recover the rate $O(1/KT)$, which is the convergence rate of well-known local SGD with periodic averaging~\citep{woodworth2020minibatch}.
\end{remark}
We now turn to stating the most important result, convergence of the personalized local model to the optimal local model.
\begin{theorem}[Personalized model convergence of Local Descent \texttt{APFL}]
\label{localpartial}
Assume  each client's objective function satisfies Assumptions~\ref{assumption: smoothness}, \ref{assumption: bounded grad}, \ref{assumption: strong convexity}, and let $\kappa = L/\mu$. Using Algorithm \ref{algorithm:LDAPFL}, by choosing the mixing weight $\alpha_i \geq \max\{ 1-\frac{1}{4\sqrt{6}\kappa}, 1-\frac{1}{4\sqrt{6}\kappa \sqrt{\mu}}\} $, learning rate: $\eta_t = \frac{16}{\mu(t+a)}$, where $a = \max\{128\kappa,\tau\}$, and using average scheme $\hat{\bm{v}}_i = \frac{1}{S_T} \sum_{t=1}^T p_t (\alpha_i\bm{v}_i^{(t)}+(1-\alpha_i)\frac{1}{K}\sum_{j\in U_t}\boldsymbol{{w}}_j^{(t)})$, where $p_t = (t+a)^2$, $S_T = \sum_{t=1}^T p_t$, and letting $f_i^*$ to denote  the local minimum of the $i$th client, then the following convergence rate holds for all clients $i \in [n]$:
\begin{align}
     \mathbb{E}[f_i(\boldsymbol{\hat{v}} _i)] - f_i^* &= O\left(\frac{\mu  }{bT^3}\right) +\alpha_i^2O\left(\frac{\sigma^2}{\mu b T}\right)+(1-\alpha_i)^2O\left(\frac{\zeta_i}{\mu b} + \frac{ \kappa L\Delta_i}{b}\right)\nonumber\\
      & + (1-\alpha_i)^2\left(O\left(\frac{\kappa L \ln T}{bT^3}\right) +O\left(\frac{\kappa^2\sigma^2}{\mu b KT}\right)+ O\left(\frac{\kappa^2\tau \left(\sigma^2 + \tau(\zeta_i + \frac{\zeta}{K})\right) }{\mu bT^2} + \frac{\kappa^4 \tau\left(\sigma^2 + 2\tau \frac{\zeta}{K}  \right) }{\mu bT^2}\right)  \right).\nonumber
\end{align}
where $b = \min\left\{\frac{K}{n},\frac{1}{2}\right\}$.
\end{theorem}
\begin{proof}
The proof is provided in Appendix~\ref{proof thm3}.
\end{proof}
An immediate implication of above theorem is the following.
\begin{corollary}
In Theorem~\ref{localpartial}, if we choose $\tau = \sqrt{T/K}$, then we recover the convergence rate:
\begin{align}
     \mathbb{E}[f_i(\boldsymbol{\hat{v}} _i)] - f_i^* & \leq   \alpha_i^2O\left(\frac{\sigma^2 }{\mu T}\right)   + (1-\alpha_i)^2\left(O\left(\frac{\kappa^2 \sigma^2}{\mu KT}\right)+ O\left(\frac{\kappa^2(\zeta_i + \frac{\zeta}{K})}{\mu KT }\right)\right)  + (1-\alpha_i)^2O\left(\frac{\zeta_i + \frac{\zeta}{K}}{\mu} +  \kappa L\Delta_i\right). \nonumber
\end{align}
\end{corollary}
A few remarks about the convergence of personalized local model are in place:
\begin{itemize} 
    \item If we set $\alpha_i = 1$, then we recover $O\left(\frac{1}{T}\right)$ convergence rate of single machine SGD. If we only focus on the terms with $(1-\alpha_i)^2$, which is contributed by the global model's convergence, and omit the residual error, we achieve the convergence rate of $O(1/KT)$ using only $\sqrt{KT}$ communication, which matches with the convergence rate of vanilla local SGD~\citep{stich2018local,woodworth2020minibatch} ($K$ is the number of all clients for local SGD).
    \item The residual error is related to the gradient diversity and local-global optimality gap, multiplied by a factor $1-\alpha_i$. It shows that taking any proportion of the global model will result in a sub-optimal ERM model. As we discussed in Section~\ref{coupled learning}, $h_{\alpha_i}$ will not be the empirical risk minimizer in most cases. 
\end{itemize}
We also remark that the analysis of convergence in Theorem~\ref{localpartial} relies on a constraint that $\alpha_i$ needs to be larger than some value in order to get a tight rate. In fact, this condition can be alleviated, but the residual error will not be as tight as stated in Theorem~\ref{localpartial}. To see this, we present the analysis of this relaxation in the following theorem.

\begin{theorem}[Personalized model convergence of Local Descent \texttt{APFL} without assumption on $\alpha_i$]
\label{without constraint}
If each client's objective function satisfies Assumptions~\ref{assumption: smoothness}, \ref{assumption: bounded grad}, \ref{assumption: strong convexity}, and its gradient is bounded by $G$, using Algorithm \ref{algorithm:LDAPFL}, learning rate: $\eta_t = \frac{8}{\mu(t+a)}$, where $a = \max\{64\kappa,\tau\}$, and using average scheme $\hat{\bm{v}}_i = \frac{1}{S_T} \sum_{t=1}^T p_t (\alpha_i\bm{v}_i^{(t)}+(1-\alpha_i)\frac{1}{K}\sum_{j\in U_t}\boldsymbol{{w}}_j^{(t)})$, where $p_t = (t+a)^2$, $S_T = \sum_{t=1}^T p_t$, and $f_i^*$ is the local minimum of the $i$th client, then the following convergence holds for all $i \in [n]$:
\begin{align}
\mathbb{E}[f_i(\boldsymbol{\hat{v}} _i)] - f_i^* &\leq  O\left(\frac{\mu  }{bT^3}\right) +\alpha_i^2O\left(\frac{\sigma^2}{\mu b T}\right)+(1-\alpha_i)^2O\left(\frac{G^2}{\mu b} \right)\nonumber\\
      & \quad + (1-\alpha_i)^2\left(O\left(\frac{\kappa L \ln T}{bT^3}\right) +O\left(\frac{\kappa^2\sigma^2}{\mu bKT}\right)+ O\left(\frac{\kappa^2 \tau^2(\zeta_i + \frac{\zeta}{K})+\kappa^2 \tau\sigma^2 }{\mu bT^2} + \frac{\kappa^4 \tau\left(\sigma^2 + 2\tau \frac{\zeta}{K}  \right) }{\mu bT^2}\right)  \right), \nonumber
\end{align}
where $b = \min\{\frac{K}{n}, \frac{1}{2}\}$.
\end{theorem}
\begin{proof}
The proof is provided in Appendix~\ref{app:cond_alpha}.
\end{proof}
\begin{remark}
Here we remove the assumption $\alpha_i \geq \max\{ 1-\frac{1}{4\sqrt{6}\kappa}, 1-\frac{1}{4\sqrt{6}\kappa \sqrt{\mu}}\}$. The key difference is that we can only show the residual error with dependency on $G$, instead of more accurate quantities $\zeta_i$ and $\Delta_i$. Apparently, when the diversity among data shards is small,  $\zeta_i$ and $\Delta_i$ terms become small which leads to a tighter convergence rate. Also notice that, to realize the bounded gradient assumption, we need to require the parameters come from a bounded domain $\mathcal{W}$. Thus, we need to do projection during parameter update, which is inexpensive.
\end{remark}

 \subsection{Nonconvex Loss}

In this section we provide the convergence results of \texttt{Local Descent APFL} on smooth nonconvex functions. Before that, let us first introduce a parameterized-invariant quantity:

\begin{definition}[Gradient Discrepancy] We define the following quantity as Gradient Discrepancy:
\begin{equation} \label{local-global gap}
  \Gamma = \sup_{\bm{x}_1,\bm{x}_2} \|\nabla F(\bm{x}_1) - \nabla F(\bm{x}_2)\|^2.
\end{equation}
\end{definition}
Notice that this quantity only depends on the geometry of loss function $F$, and the set where $\bm{x}_1$ and $\bm{x}_2$ are drawn from. If we assume the norm of gradient of F are bounded by $G$, the a trivial upper bound of this quantity will be $2G^2$.  As we will show in the later Theorem, this quantity together with gradient diversity will serve as residual error in the convergence rate. Recall that in Theorem~\ref{thm:generalization}, we have a discrepancy term $\lambda_{\mathcal{H}} (S) = \sup_{h,h' \in \mathcal{H}} \frac{1}{|S|}\sum_{(\bm{x},y)\in S} |h(\bm{x})-h'(\bm{x})|$. To some extent, gradient discrepancy $\Gamma$ is analogous to $\lambda_{\mathcal{H}} (S) $, since both of them reflect the intrinsic property of loss function and hypothesis class.

\begin{theorem}[Global model convergence of Local Descent \texttt{APFL}]
\label{Thm: nonconvex_global}
If each client's objective function satisfies Assumptions~\ref{assumption: smoothness}-\ref{assumption: bounded grad}, using Algorithm \ref{algorithm:LDAPFL}, by choosing $K=n$ and learning rate $\eta = \frac{1}{2\sqrt{5}L\sqrt{T}}$, then the following convergence holds:
\begin{align}
  \frac{1}{T}\sum_{t=1}^T  \mathbb{E} \left[\left \| \nabla F(\bm{w}^{(t)})\right\|^2\right]    & \leq O\left(\frac{L}{ \sqrt{T}} \right)+   O \left( \frac{\sigma^2 + \frac{\sigma^2}{n} +\frac{\zeta}{n}}{T} \right)+ O \left( \frac{\sigma^2  }{\sqrt{T}} \right). \nonumber
\end{align}
\end{theorem}
\begin{proof}
The proof is provided in Appendix~\ref{app: nonconvex}.
\end{proof}

The following Theorem establish the convergence rate of personalized model on nonconvex loss function:
\begin{theorem}[Personalized model convergence of Local Descent \texttt{APFL}]
\label{Thm: nonconvex}
If each client's objective function satisfies Assumptions~\ref{assumption: smoothness}-\ref{assumption: bounded grad}, using Algorithm \ref{algorithm:LDAPFL}, by choosing $K=n$ and learning rate $\eta = \frac{1}{2\sqrt{5}L\sqrt{T}}$, then the following convergence holds:
\begin{align}
  &\frac{1}{T} \sum_{t=1}^T \mathbb{E} \left[\left \| \nabla f_i (\hat{\bm{v}}^{(t)}_i) \right\|^2\right] \nonumber \\
      & \leq  O \left( \frac{ L }{\sqrt{T}}  \right)  +  \alpha_i^4 O\left(\frac{\sigma^2}{\sqrt{T}}\right)  + (1-\alpha_i)^2O \left( \frac{\sigma^2}{n\sqrt{T}}\right)+  (1-\alpha_i^2)^2 \left (  \zeta_i + \Gamma \right) \nonumber\\
    & \quad +   \alpha_i^4 (1-\alpha_i)^2  O \left( \frac{  \tau^4\left( \sigma^2 + \frac{\sigma^2}{n}+ \frac{\zeta}{n}  \right) }{T^2}  + \frac{  \tau^2\left( \sigma^2 + \frac{\sigma^2}{n} + \zeta_i \right) }{T} \right) +  (1-\alpha_i)^2     O\left(\frac{\tau^2 (\sigma^2 + \frac{\sigma^2}{n} +\frac{\zeta}{n})}{T}  \right)   \nonumber.
\end{align}
where $\tau$ is the number of local updates (i.e., synchronization gap).
\end{theorem}
\begin{proof}
The proof is provided in Appendix~\ref{app: nonconvex}.
\end{proof}
An immediate implication of above theorem is the following.
\begin{corollary}
In Theorem~\ref{Thm: nonconvex}, if we choose $\tau = n^{-1/4} T^{1/4}$, then we recover the convergence rate:
\begin{align}
  &\frac{1}{T} \sum_{t=1}^T \mathbb{E} \left[\left \| \nabla f_i (\hat{\bm{v}}^{(t)}_i) \right\|^2\right] \nonumber \\
      & \leq  O \left( \frac{ L }{\sqrt{T}}  \right)  +  \alpha_i^4 O\left(\frac{\sigma^2}{\sqrt{T}}\right) \nonumber\\
      & \quad +    (1-\alpha_i)^2 \left ( \frac{L}{ \sqrt{T}}  + \frac{ \sigma^2}{ n\sqrt{T}} \right)  +   (1-\alpha_i)^2  O \left( \frac{ (\sigma^2 + \frac{\sigma^2}{n} +\frac{\zeta}{n})}{\sqrt{nT}}  +  \frac{ \alpha_i^4 \left( \sigma^2 + \frac{\sigma^2}{n} + \zeta_i \right) }{\sqrt{nT}} \right)  +  (1-\alpha_i^2)^2 \left (  \zeta_i + \Gamma \right) \nonumber.
\end{align}
\end{corollary}

Here we recover $O \left( \frac{ 1 }{\sqrt{T}}  \right) + (1-\alpha_i)^2 O \left( \frac{ 1 }{\sqrt{T}} +  \frac{ 1 }{\sqrt{nT}} \right)+  (1-\alpha_i^2)^2 \left (  \zeta_i + \Gamma \right)$, with $n^{3/4}T^{3/4}$ communication rounds. The rate with factor $(1-\alpha_i)^2$ is contributed from the global model convergence, and here we still have some additive residual error reflected by $\zeta_i$ and $\Gamma$. Compared to most related work of~\cite{haddadpour2019convergence} regarding the convergence of local SGD on nonconvex functions, they obtain $O(1/\sqrt{nT})$, the sublinear speedup w.r.t. $n$, while we only have speedup on partial terms in convergence rate. This could be solved by using different learning rate for local and global update, which we will leave as an open problem. Additionally, we assume $K=n$ to derive the convergence in nonconvex setting, so how to alleviate this assumption is still open. 

\section{Experiments}\label{sec:exp}
In this section, we empirically show the effectiveness of the proposed algorithm in personalized federated learning. To that end, we aim at showing the convergence of both optimization and generalization errors of our proposed algorithms on training data and local validation data, respectively. More importantly, we show that using our algorithm, we can utilize the model for each client, in order to get the best generalization error on their own validation data. Throughout these experiments we aim at finding answers to the following questions:
\begin{itemize}
    \item How does our proposed \texttt{APFL} algorithm perform on training a model with a strongly convex loss function, compared to other approaches?
    \item What is the effect of sampling of clients on the proposed \texttt{APFL}?
    \item How does the adaptive tuning of $\alpha$ affect the training of the personalized model?
    \item How does \texttt{APFL} performs on training a model with a non-convex loss function?
    \item How does the proposed \texttt{APFL} algorithm perform compared to other personalization approaches for federated learning?
\end{itemize}
To answer these questions we use various models (such as logistic regression, CNN, and MLP), on different datasets. First, we describe the experimental setup we used for our experiments, as close as possible to a real federated learning setup, and then present the results.

\subsection{Experimental setup}
To mimic the real setting of the federated learning, we run our code on Microsoft Azure systems, using Azure Machine Learning API. We developed our code on PyTorch~\citep{paszke2019pytorch} using its ``distributed'' API. Then, we deploy this code on Standard F64s family of VMs in Microsoft Azure, where each node has $64$ vCPUs that enable us to run multiple threads of the training simultaneously. We use the Message Passing Interface (MPI) to connect each node to the server. To use PyTorch in compliance with MPI, we need to build it against the MPI. Thus, we build our user-managed docker container with the aforementioned settings.

\paragraph{Datasets.} We use four datasets for our experiments, MNIST\footnote{\url{http://yann.lecun.com/exdb/mnist/}}, CIFAR10~\citep{krizhevsky2009learning}, EMNIST~\citep{cohen2017emnist}, and a synthetic dataset.

\paragraph{MNIST and CIFAR10} For the MNIST and CIFAR10 datasets to be similar to the setting in federated learning, we need to manually distribute them in a non-IID way, hence the data distribution is pathologically heterogeneous. To this end, we follow the steps used by~\cite{mcmahan2017communication}, where they partitioned the dataset based on labels and for each client draw samples from some limited number of classes. We use the same way to create $3$ datasets for the MNIST, that are, MNIST non-IID with $2$ classes per client, MNIST non-IID with $4$ classes per client, and MNIST IID, where the data is distributed uniformly random across different clients. Also, we create an non-IID CIFAR10 dataset, where each client has access to only $2$ classes of data. 

\paragraph{EMNIST} In addition to pathological heterogeneous data distribution, we applied our algorithm on a real-world heterogeneous dataset, which is an extension to the MNIST dataset. The EMNIST dataset includes images of characters divided by authors, where each author has a different style, make their distributions different~\cite{caldas2018leaf}. We use only digit characters and $1000$ authors' data to train our models on.

\paragraph{Synthetic} For generating the synthetic dataset, we follow the procedure used by~\cite{li2018federated}, where they use two parameters, say \texttt{synthetic$\left(\gamma,\beta\right)$}, that control how much the local model and the local dataset of each client differ from that of other clients, respectively. Using these parameters, we want to control the diversity between data and model of different clients. The procedure is that for each client we generate a weight matrix $\bm{W}_i \in \mathbb{R}^{m \times c}$ and a bias $\bm{b}\in \mathbb{R}^c$, where the output for the $i$th client is $y_i = \arg\max\left(\sigma\left(\bm{W}_i^\top \bm{x}_i + b\right)\right)$, where $\sigma(.)$ is the softmax. In this setting, the input data $\bm{x}_i \in \mathbb{R}^{m}$ has $m$ features and the output $y$ can have $c$ different values indicating number of classes. The model is generated based on a Gaussian distribution $\bm{W}_i \sim \mathcal{N}\left(\bm{\mu}_i,1\right)$ and $\bm{b}_i \sim \mathcal{N}\left(\bm{\mu}_i,1\right)$, where $\bm{\mu}_i \sim \mathcal{N}\left(0,\gamma\right)$. The input is drown from a Gaussian distribution $\bm{x}_i \sim \mathcal{N}\left(\bm{\nu}_i,\bm{\Sigma}\right)$, where $\bm{\nu}_i \sim \mathcal{N}\left(V_i,1\right)$ and $V_i \sim \mathcal{N}\left(0,\beta\right)$. Also the variance $\bm{\Sigma}$ is a diagonal matrix with value of $\bm{\Sigma}_{k,k} = k^{-1.2}$. Using this procedure, we generate three different datasets, namely \texttt{synthetic$\left(0.0,0.0\right)$}, \texttt{synthetic$\left(0.5,0.5\right)$}, and  \texttt{synthetic$\left(1.0,1.0\right)$}, where we move from an IID dataset to a highly non-IID data.

\subsection{Experimental results}
In this part, we discuss the results of applying \texttt{APFL} in a federated setting. For all the experiments, we have $100$ users (except for the EMNIST experiment, where we use the data of $1000$ clients), each of which has access to its own data only. For the learning rate, we use the linear decay structure with respect to the stochastic steps, suggested by~\cite{bottou2012stochastic}. At each iteration the learning rate is decreased by $1\%$, unless otherwise stated.
Throughout these experiments we are reporting the results for these three models:
\begin{itemize}
    \item \textbf{Global Model}: The global model is the aggregated model after each round of communication. In this experiment it could be the global model of either FedAvg or SCAFFOLD~\citep{karimireddy2019scaffold}.
    \item \textbf{Localized Global Model}: This is the fine-tuned version of the global model at each round of communication after $\tau$ steps of local SGD. Here, we have either the local FedAvg or the local SCAFFOLD, representing the local version of their respective global model. The reported results are for the average of the performance over all the local models on each online client. In all the experiments $\tau=10$, unless otherwise stated.
    \item \textbf{Personalized Model}: This model is the personalized model produced by our proposed algorithm \texttt{APFL} (Algorithm~\ref{algorithm:LDAPFL}). The reported results are the average of the respective performance of personalized models over all online clients at each round of communication. In the comparison with other personalization approaches, this model refers to their respective personalized model.
\end{itemize}
To corroborate the theoretical findings, we report the performance over training data for optimization part and local validation data (from the same distribution as training data for each client) for the generalization part.

\paragraph{Strongly convex loss.}  First, we run this set of experiments  on the MNIST dataset, with different levels of non-IIDness using the three datasets from MNIST discussed before. We use logistic regression with parameter regularization as our loss function, to have a strongly convex loss function: $f(\bm{w})= \frac{1}{m} \sum_{i=1}^m \log(1+\exp(-y_i \bm{w}^{T}\bm{x}_i))+ \frac{\lambda}{2}\|\bm{w}\|^2 $. For this part, we do not have client sampling for federated learning. We compare the personalized model of \texttt{APFL} with localized models of FedAvg~\citep{mcmahan2017communication} and SCAFFOLD~\citep{karimireddy2019scaffold}, as well as their global models. The initial learning rate is set to $0.1$ and it is decaying as mentioned before. The results of running this experiment on $100$ clients are depicted in Figure~\ref{fig:mnist_loss_acc}, where we move from highly non-IID data distribution (left) to IID data distribution (right). The first row shows the training loss of different local and global models as well as personalized models with different rates of personalization as $\alpha$. The second row shows the generalization performance of different models on each client's local validation data. As it can be seen, global models learned from FedAvg and SCAFFOLD have high local training losses. On the other hand, taking more proportion of the local model in the personalized model (namely, increasing $\alpha$) will result in the lower training losses. For generalization ability, we can see that the best performance is given by personalized model with $\alpha = 0.25$ in both (a) and (b) cases, which outperforms the global (FedAvg and SCAFFOLD) and their localized versions. However, as we move toward IID data distribution, the effect of personalization vanishes; that is, for IID data, models learned by FedAvg and SCAFFOLD have better generalization capabilities. Hence, as expected by the theoretical findings, we can benefit from personalization the most when there is a statistical heterogeneity between the data of different clients. When the data are distributed IID, local models of FedAvg or SCAFFOLD is preferable.
\begin{figure}[t!]
	\centering
	\subfigure{
		\centering
		\includegraphics[width=0.305\textwidth]{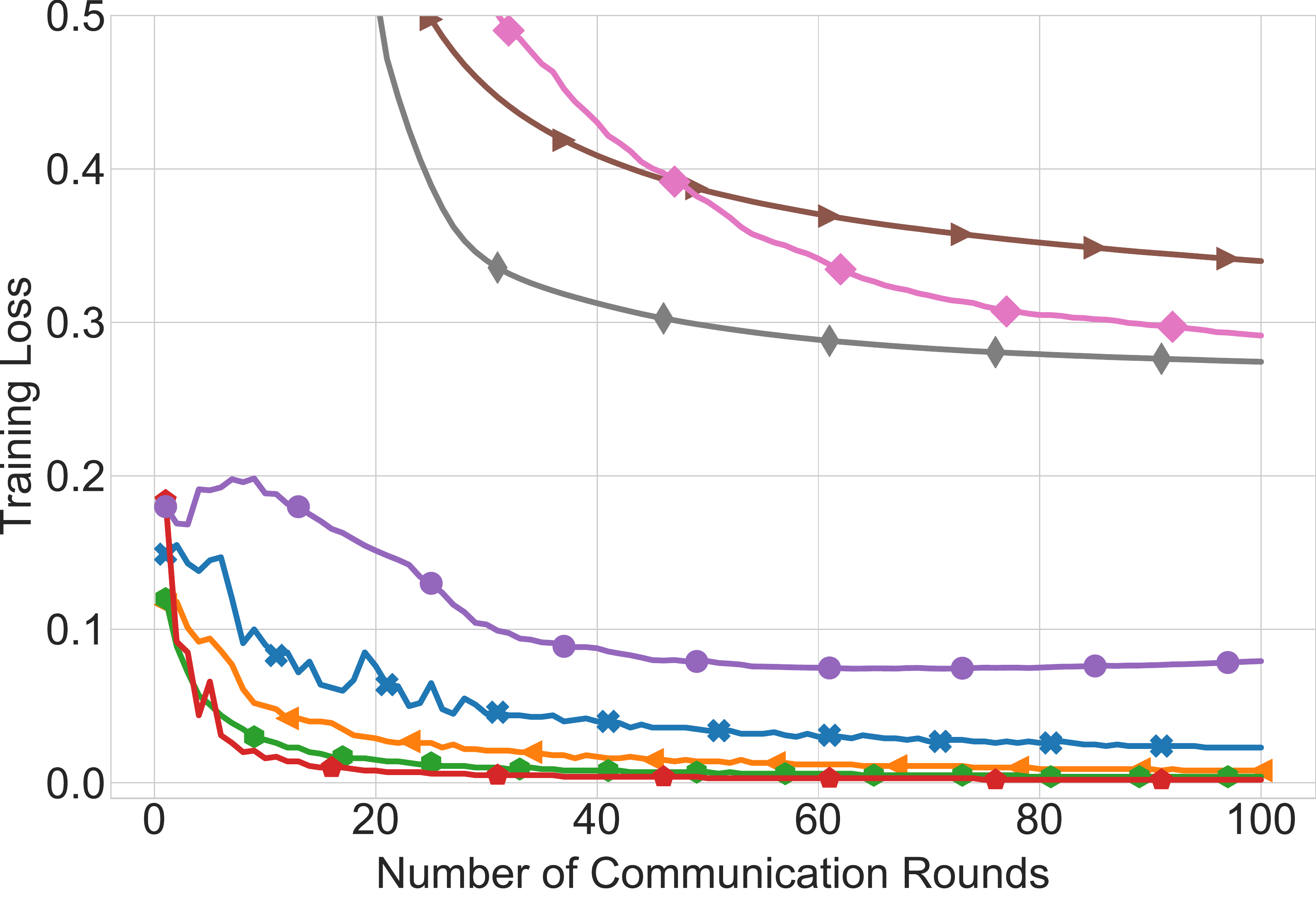}
		\label{fig:mnist_loss_comm_class2}
		}
		\hfill
		\subfigure{
			\centering 
			\includegraphics[width=0.305\textwidth]{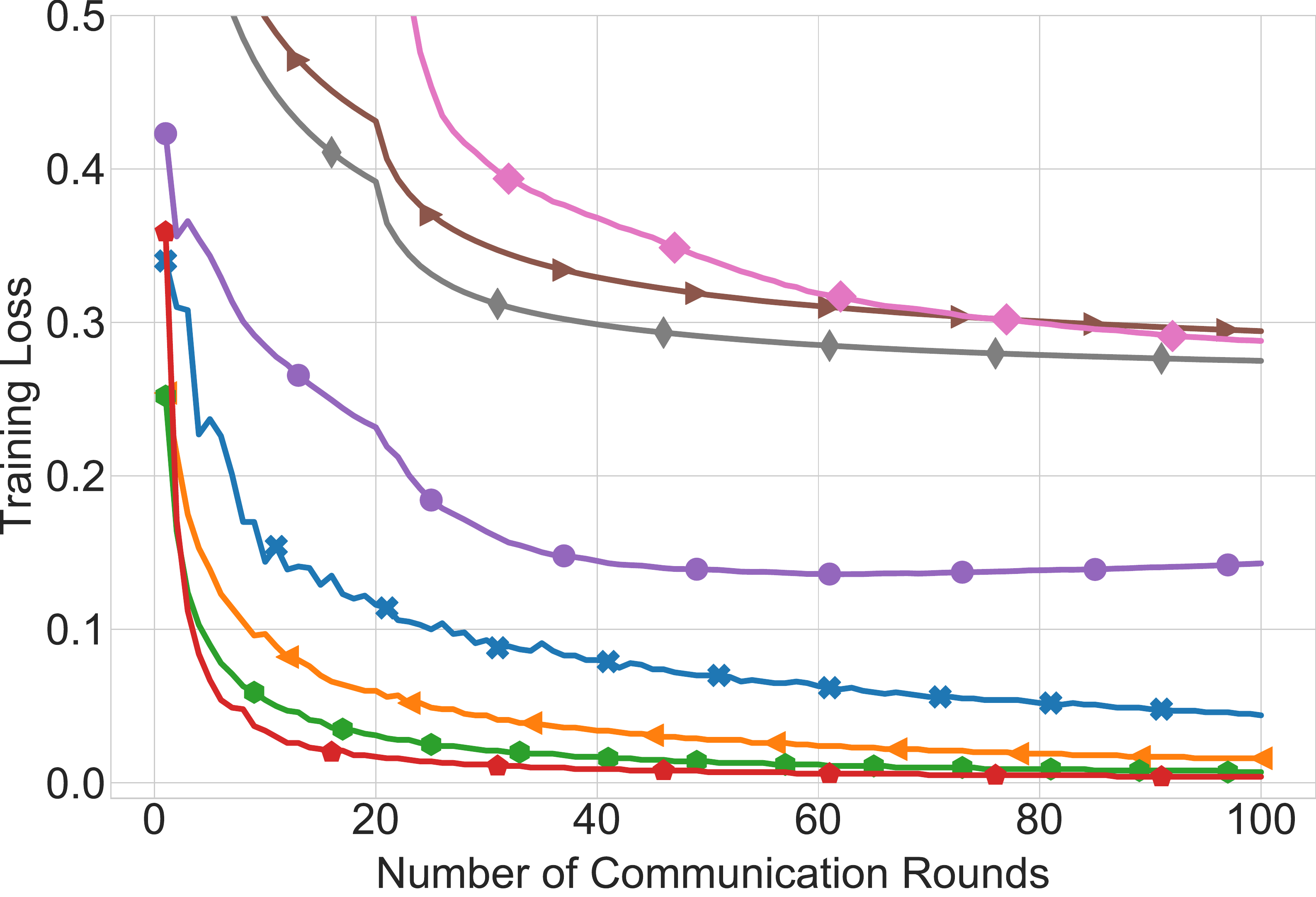}
			\label{fig:mnist_loss_comm_class4}
			}
		\hfill
	\subfigure{  
		\centering 
		\includegraphics[width=0.305\textwidth]{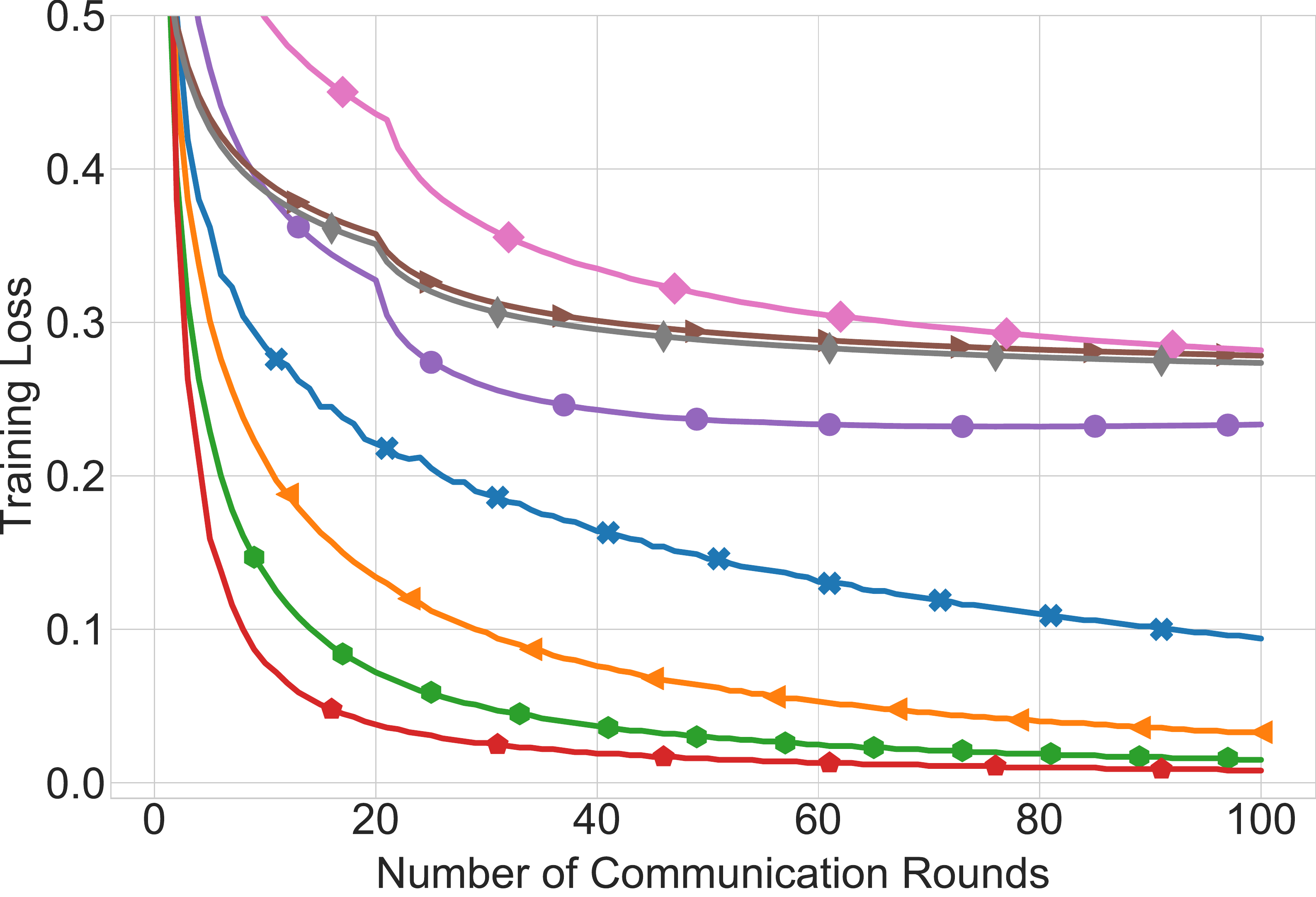}
		\label{fig:mnist_loss_comm_iid}
		}
		
	\setcounter{subfigure}{0}
	\subfigure[Non-IID with 2 classes per client]{
		\centering
		\includegraphics[width=0.305\textwidth]{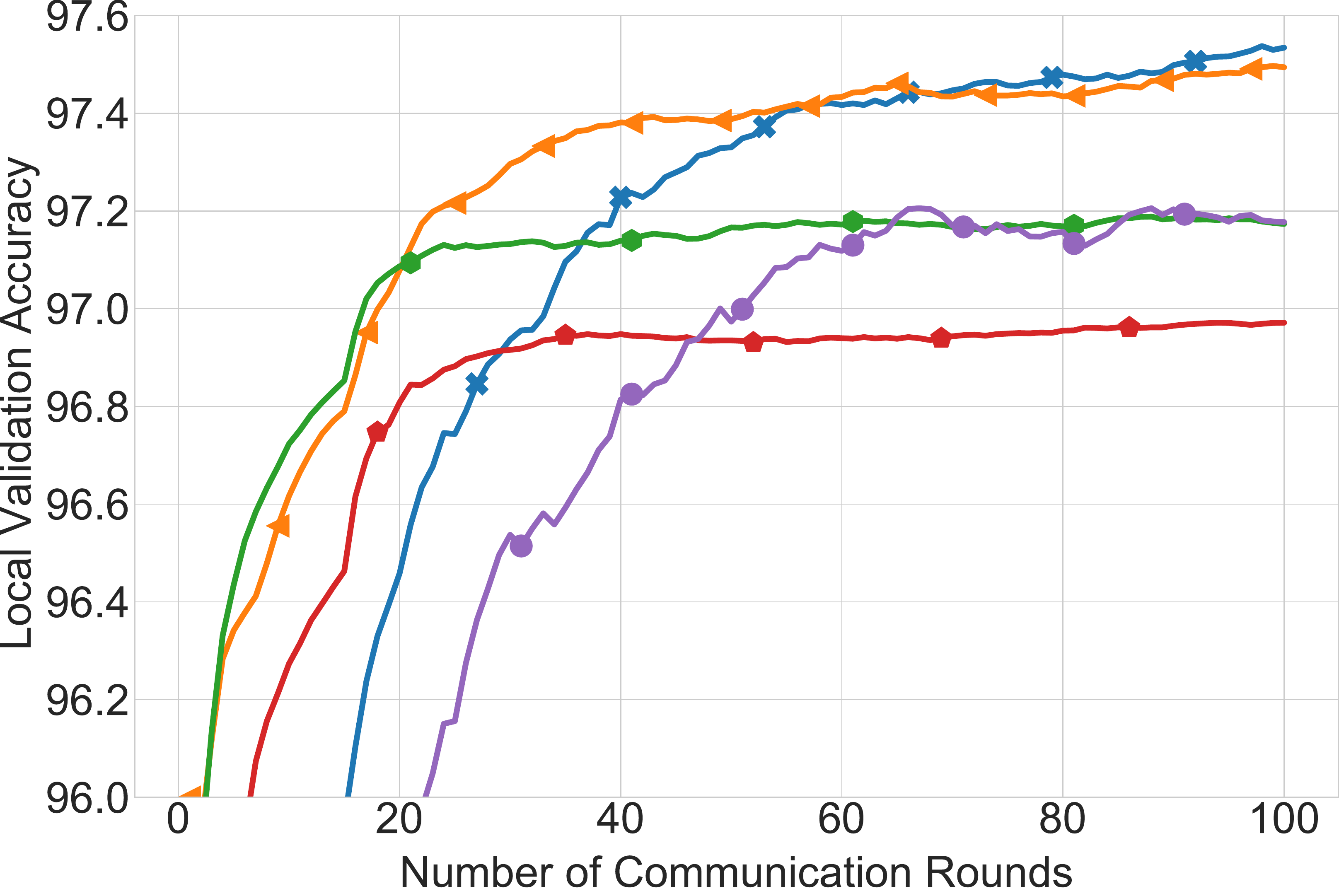}
		\label{fig:mnist_acc_comm_class2}
	}
	\hfill
	\subfigure[Non-IID with 4 classes per client]{  
			\centering 
			\includegraphics[width=0.305\textwidth]{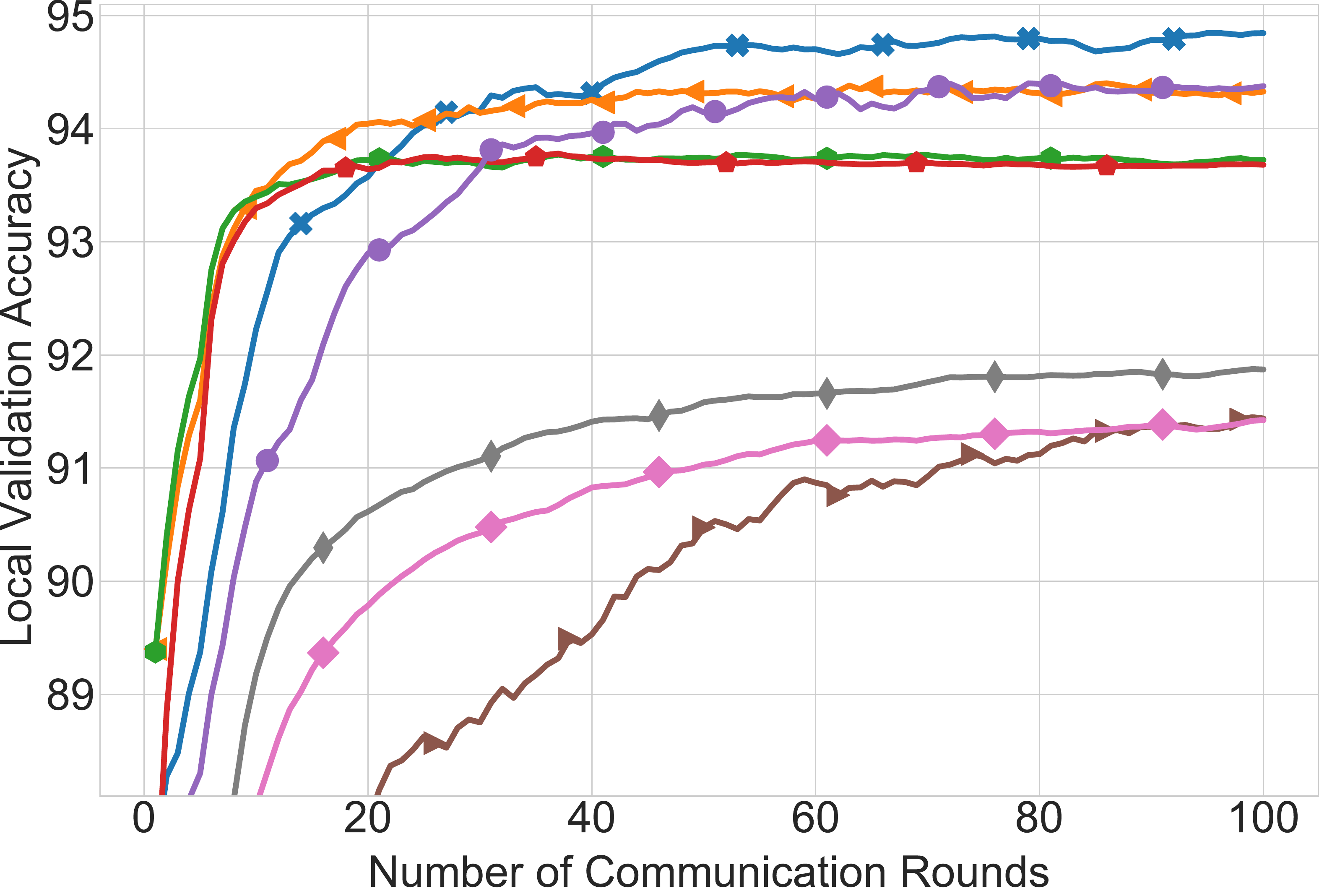}
			\label{fig:mnist_acc_comm_class4}
	}
	\hfill
	\subfigure[IID data distribution among clients]{   
		\centering 
		\includegraphics[width=0.305\textwidth]{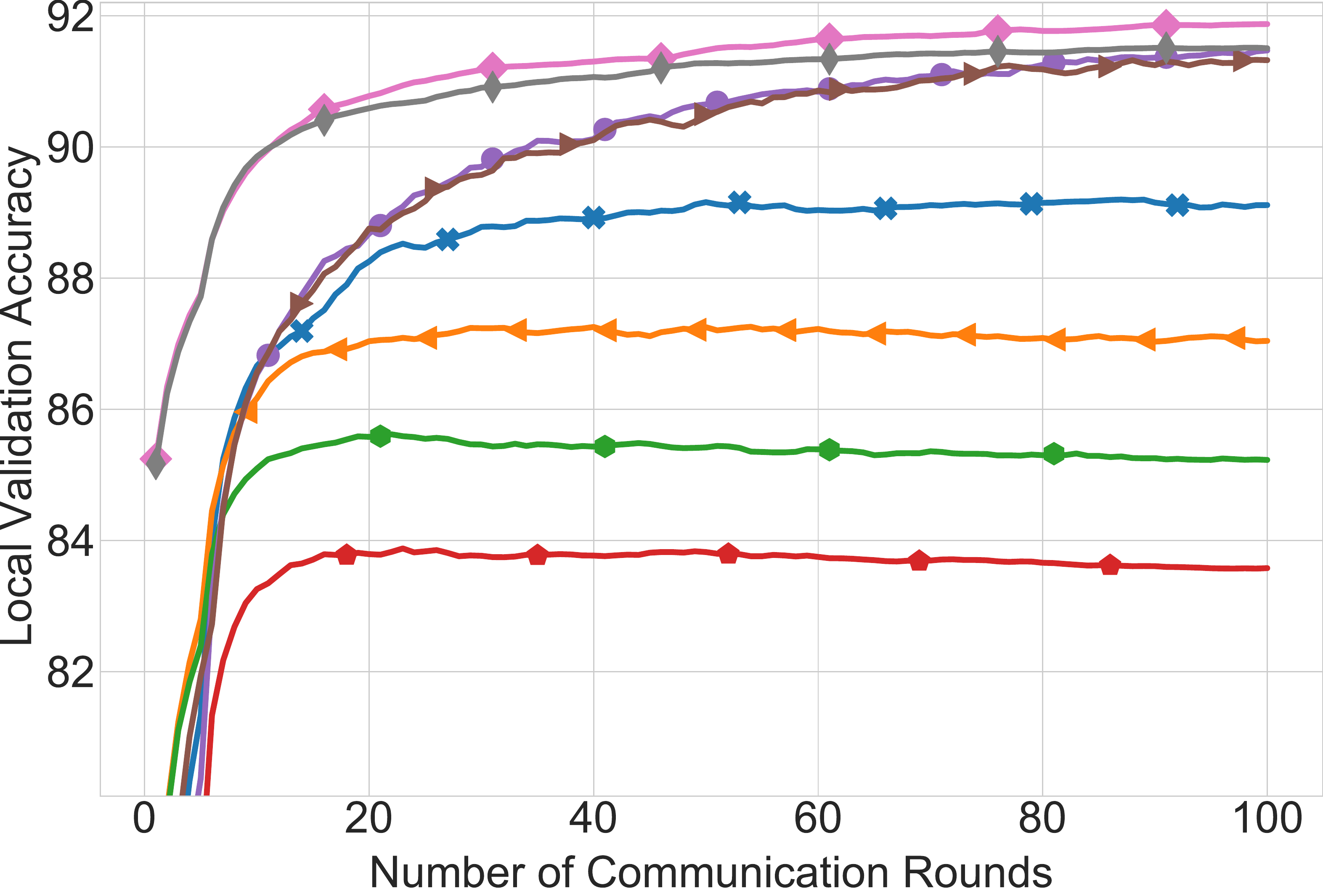}
		\label{fig:mnist_acc_comm_iid}
		}
		
	\centering
	\subfigure{\centering 
			\includegraphics[width=0.8\textwidth]{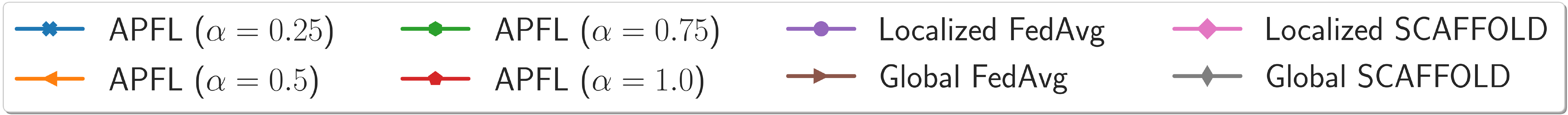}
	}
	\caption[]{Comparing the performance of the proposed \texttt{APFL} algorithm with FedAvg~\citep{mcmahan2017communication} (\texttt{APFL} with $\alpha=0$) and SCAFFOLD~\citep{karimireddy2019scaffold} on the MNIST dataset with different levels of non-IID data distribution among different clients using a logistic regression model. The first row shows the training loss for global models, as well as local and personalized models, averaged over all clients. The second row shows the generalization of the same models on their validation data. In (a), the second row, SCAFFOLD lines and global FedAvg line are removed since they represent low values, which degrade the readability of the plot.}
	\label{fig:mnist_loss_acc}
\end{figure}

\paragraph{Effect of sampling.} To understand how the sampling of different clients will affect the performance of the \texttt{APFL} algorithm, we run the same experiment with different sampling rates for the MNIST dataset. The results of this experiment are depicted in Figure~\ref{fig:mnist_loss_acc_sampling}, where we run the experiment for different sampling rates of $K \in \{0.3,0.5,0.7\}$. Also, we run it with different values of $\alpha \in \{0.25,0.5,0.75\}$. The results are reported for the personalized model of \texttt{APFL} and localized FedAvg. As it can be inferred, decreasing the sampling ratio has a negative impact on both the training and generalization performance of FedAvg. However, we can see that despite the sampling ratio, \texttt{APFL} is outperforming local model of the FedAvg in both training and generalization. Also, from the results of Figure~\ref{fig:mnist_loss_acc}, we know that for this dataset that is highly non-IID, larger $\alpha$ values are preferred. Increasing $\alpha$ can diminish the negative impacts of sampling on personalized models both in training and generalization.

\begin{figure*}[t!]
	\centering
	\subfigure{
		\centering
		\includegraphics[width=0.30\textwidth]{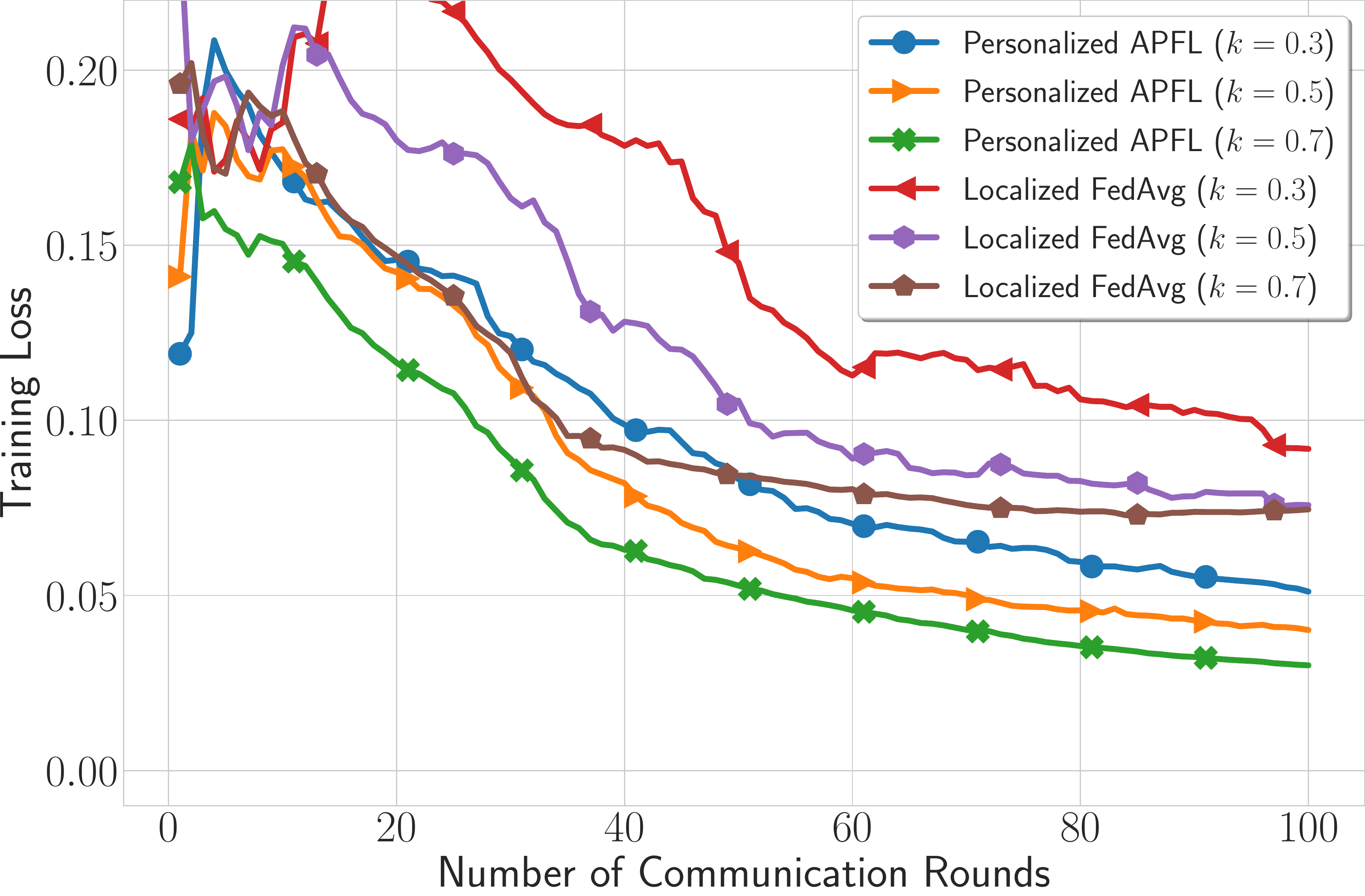}
		\label{fig:mnist_loss_comm_alpha025}
		}
		\hfill
		\subfigure{
			\centering 
			\includegraphics[width=0.30\textwidth]{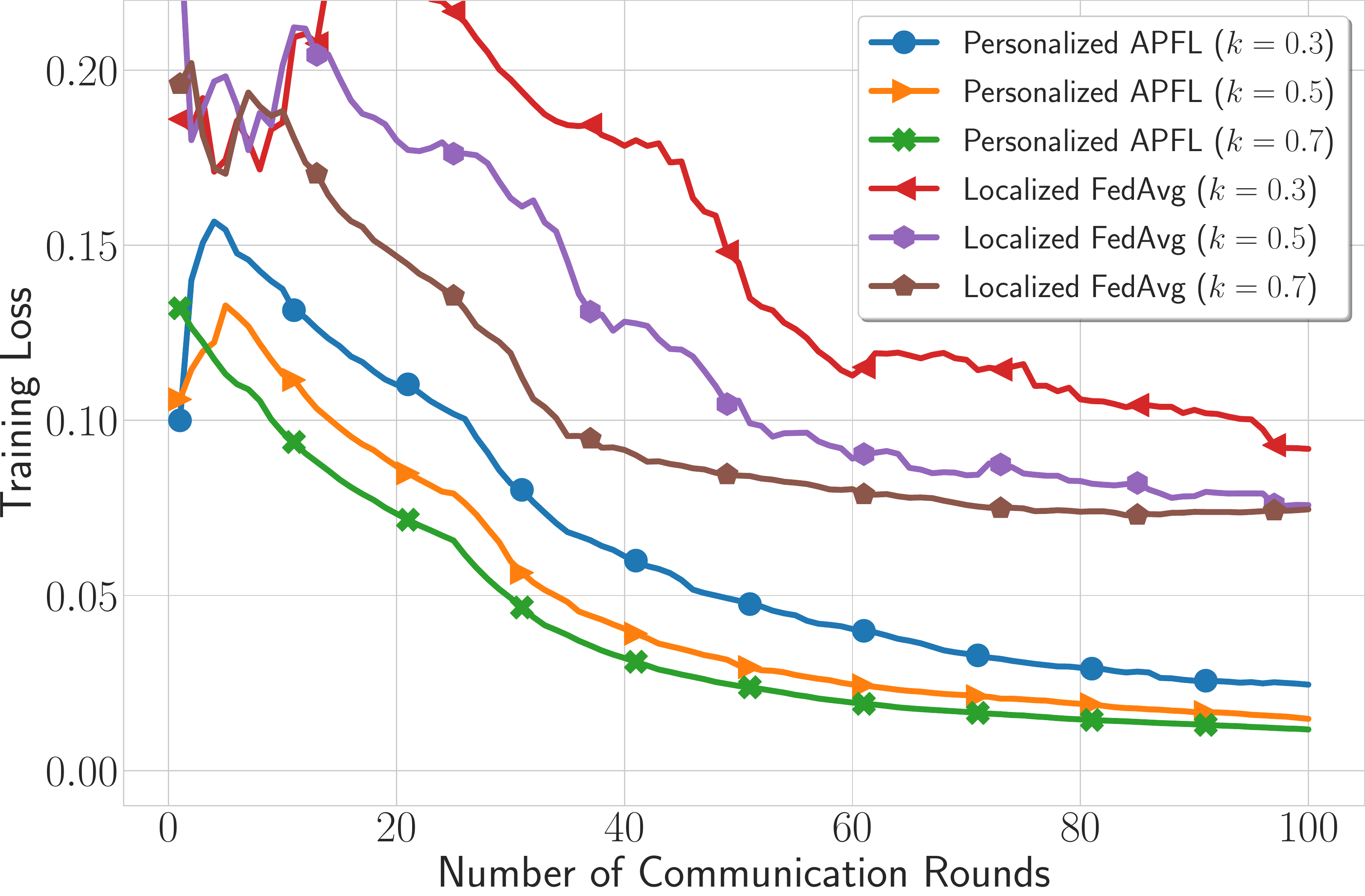}
			\label{fig:mnist_loss_comm_alpha05}
			}
		\hfill
	\subfigure{  
		\centering 
		\includegraphics[width=0.30\textwidth]{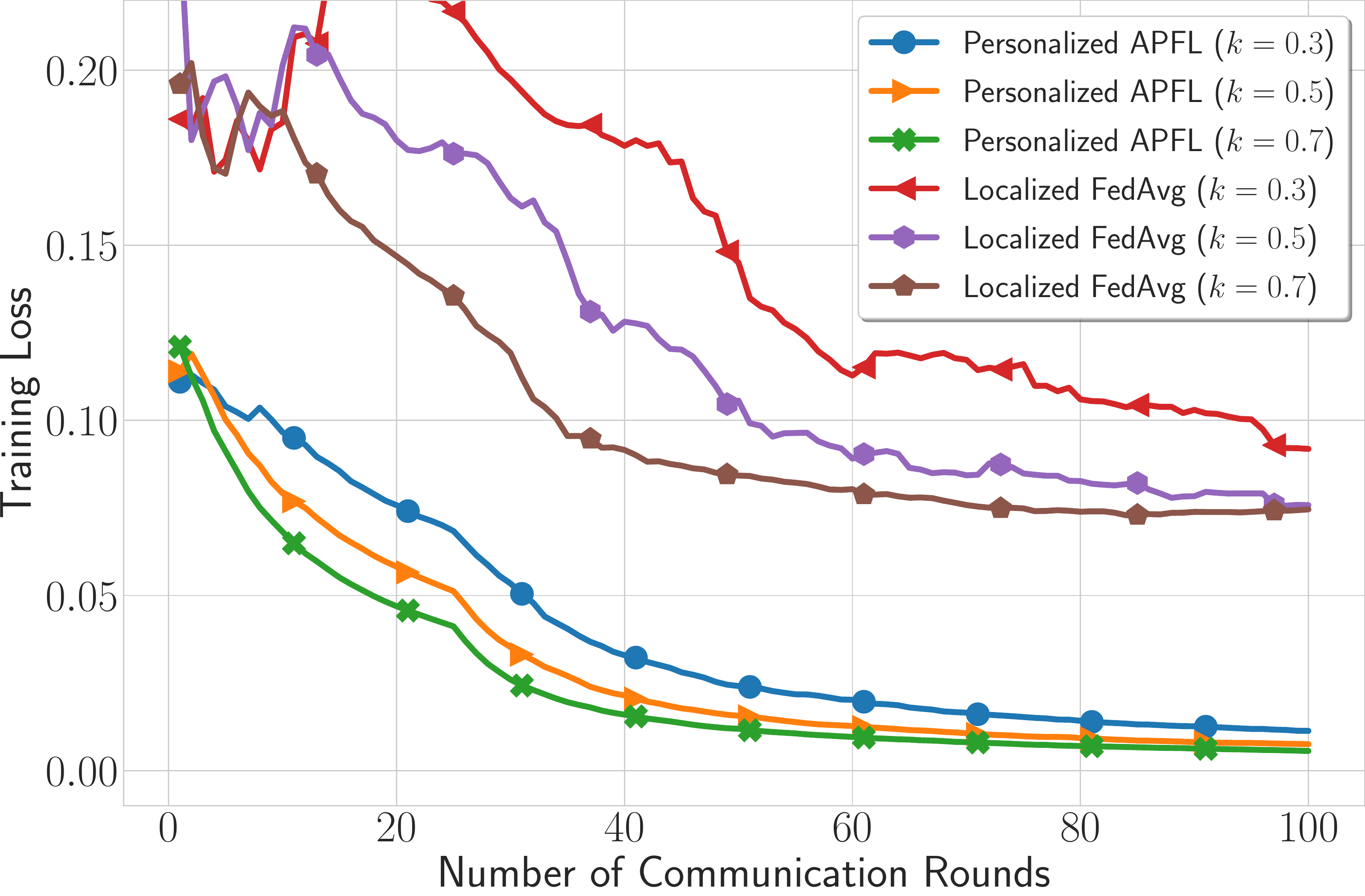}
		\label{fig:mnist_loss_comm_alpha075}
		}
		
	\setcounter{subfigure}{0}
	\subfigure[$\alpha=0.25$]{
		\centering
		\includegraphics[width=0.305\textwidth]{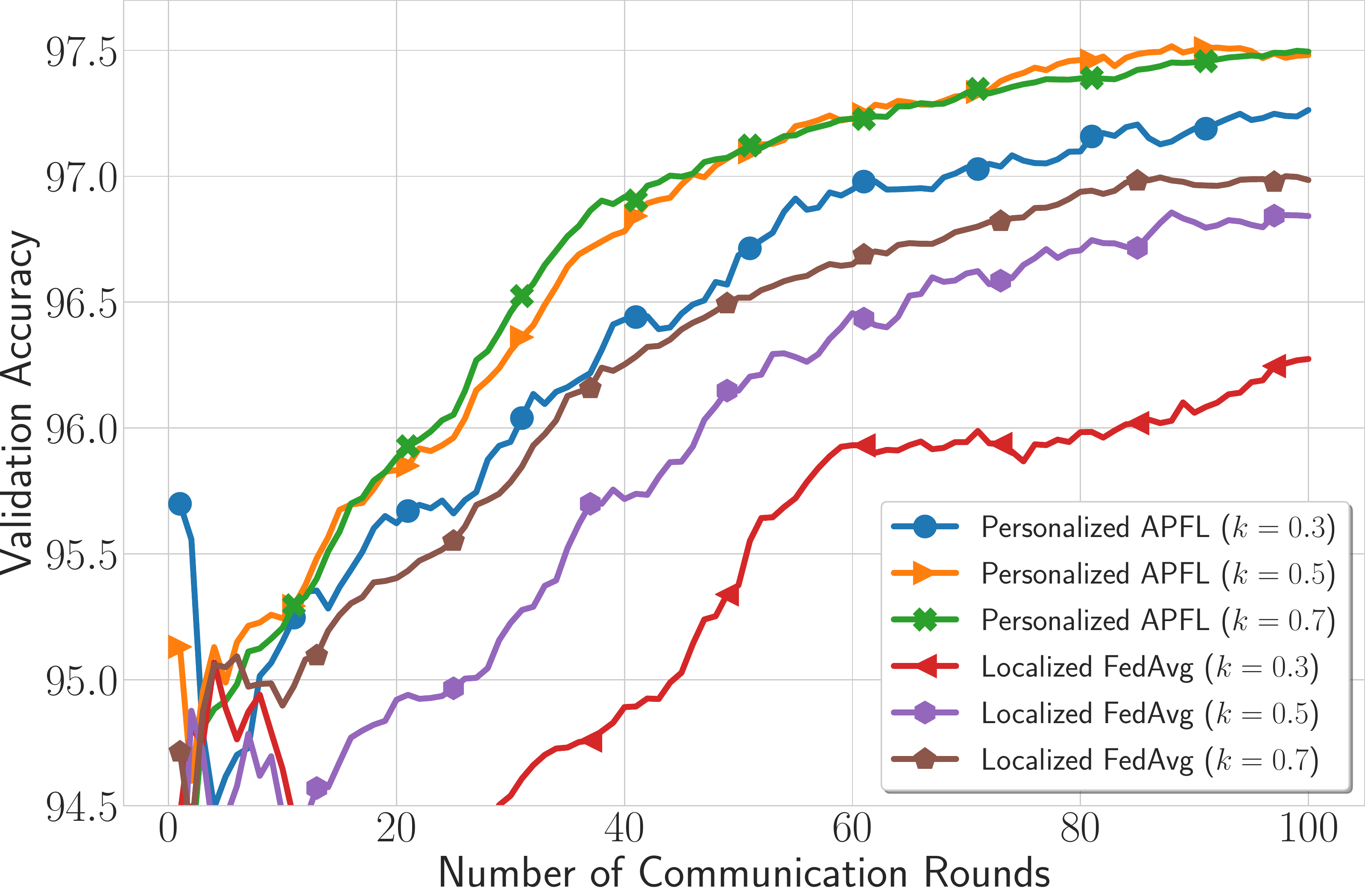}
		\label{fig:mnist_acc_comm_alpha025}
	}
	\hfill
	\subfigure[$\alpha=0.5$]{  
			\centering 
			\includegraphics[width=0.31\textwidth]{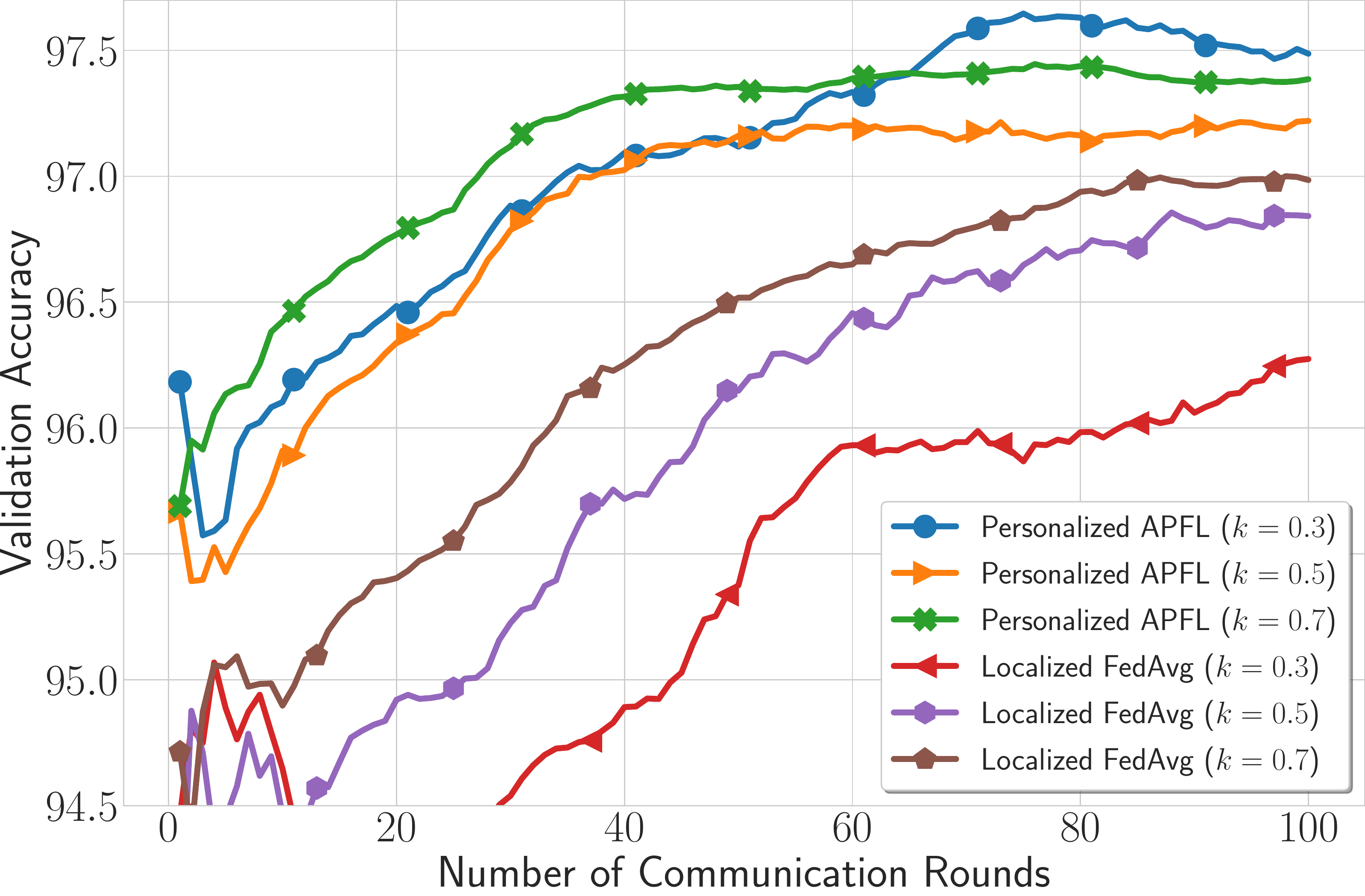}
			\label{fig:mnist_acc_comm_alpha05}
	}
	\hfill
	\subfigure[$\alpha=0.75$]{   
		\centering 
		\includegraphics[width=0.31\textwidth]{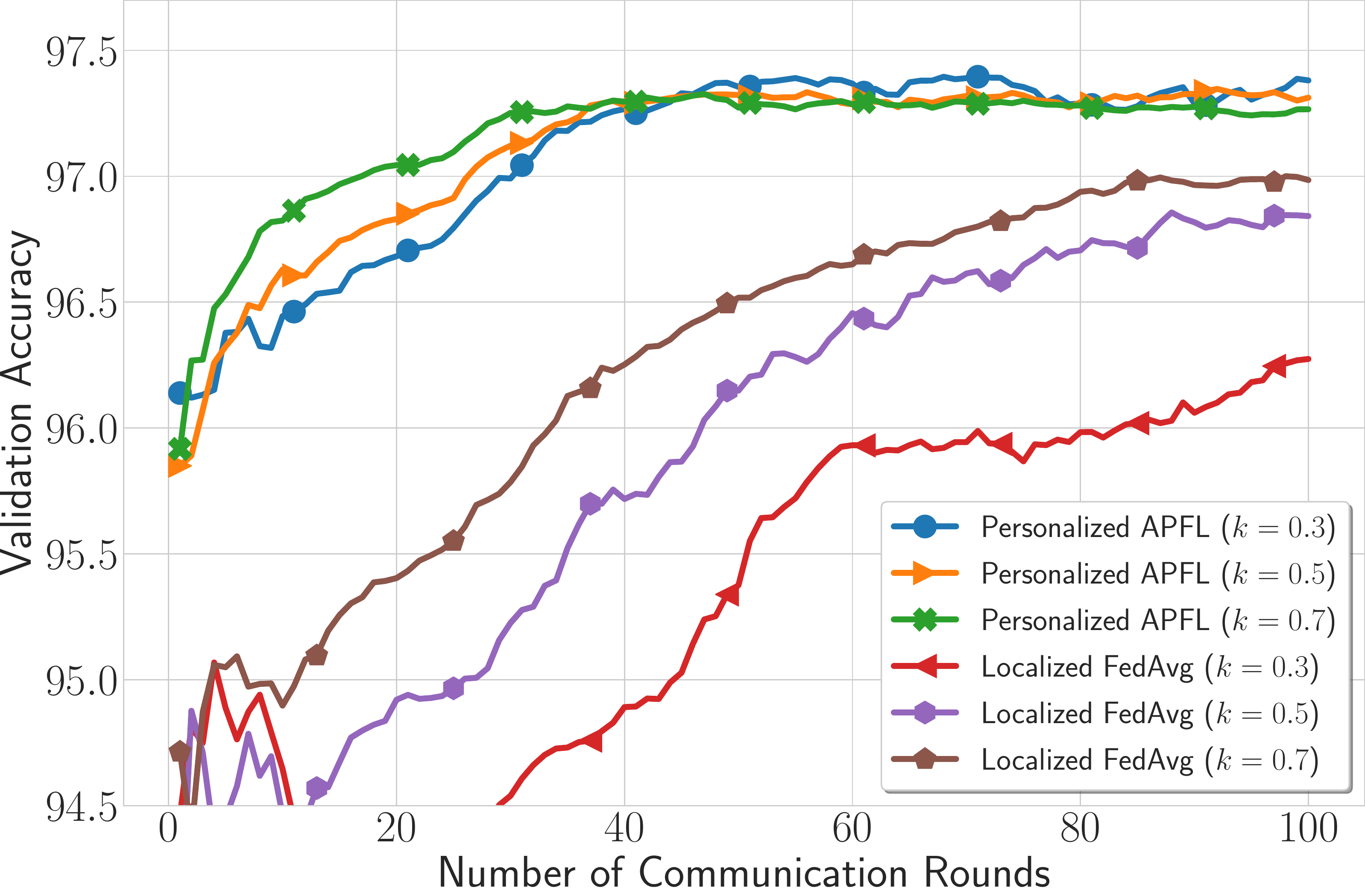}
		\label{fig:mnist_acc_comm_alpha75}
		}
	\caption[]{Evaluating the effect of sampling on \texttt{APFL} and FedAvg algorithm using the MNIST dataset that is non-IID with only $2$ classes per client with logistic regression as the loss. The first row is training performance on the local model of FedAvg and personalized model of \texttt{APFL} with different sampling rates from $\{0.3,0.5,0.7\}$. The second row is the generalization performance of models on local validation data, aggregated over all clients. It can be inferred that despite the sampling ratio, \texttt{APFL} can superbly outperform FedAvg. 
    }
	\label{fig:mnist_loss_acc_sampling}
\end{figure*}

\paragraph{Adaptive $\alpha$ update.} Now, we want to show how adaptively learning the value of $\alpha$ across different clients, based on (\ref{eq:alpha_gd}), will affect the training and generalization performance of \texttt{APFL}'s personalized models. For this experiment, we will use the three mentioned synthetic datasets we generated with logistic regression as the loss function. We set the intial value of $\alpha_i^{(0)}=0.01$ for every $i\in [n]$. The results of this experiment are depicted in Figure~\ref{fig:synthetic_loss_acc_adaptive}, where the first figure shows the training performance of different datasets. The second figure is comparing the generalization of local and personalized models. As it can be inferred, in training, \texttt{APFL} outperforms FedAvg in the same datasets. More interestingly, in generalization of learned \texttt{APFL} personalized models, all datasets achieve almost the same performance as a result of adaptively updating $\alpha$ values, while the FedAvg algorithm has a huge gap with them. This shows that, when we do not know the degree of diversity among data of different clients, we should adaptively update $\alpha$ values to guarantee the best generalization performance.
\begin{figure}[ht!]
	\centering
	\subfigure{
		\centering
		\includegraphics[width=0.4\textwidth]{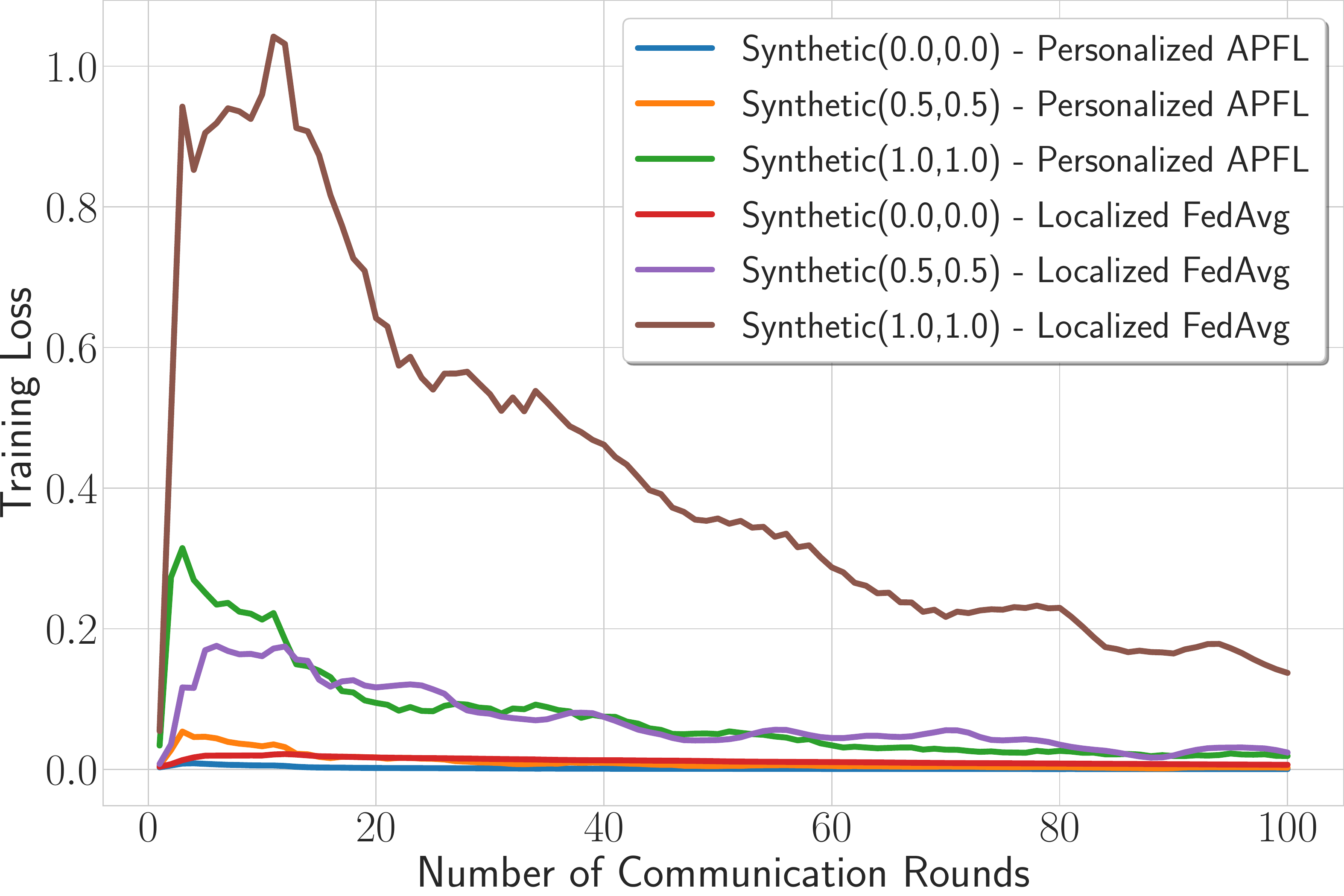}
		\label{fig:synthetic_loss_comm_adaptive}
		}
		\hspace{1cm}
		\subfigure{
			\centering 
			\includegraphics[width=0.4\textwidth]{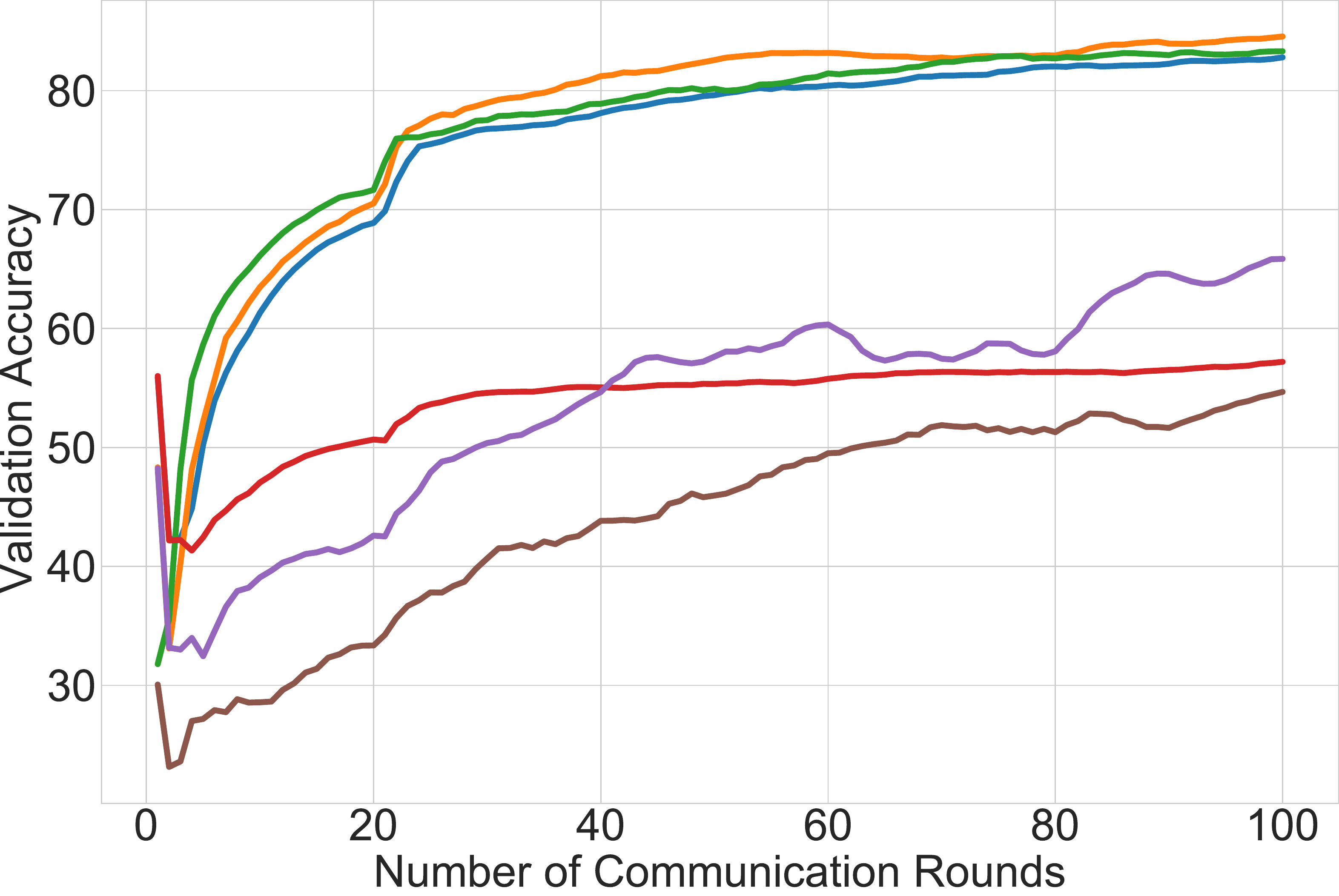}
			\label{fig:synthetic_acc_comm_adaptive}
			}
	\caption[]{Comparing the personalized model of \texttt{APFL} with adaptive $\alpha$ and the local model in FedAvg. The first figure is the training performance, where \texttt{APFL} outperforms FedAvg when comparing the same dataset. The second figure shows the generalization of these methods on local validation data. \texttt{APFL} superbly outperforms FedAvg in generalization performance and adaptively updating $\alpha$ results in the same performance for datasets with different levels of diversity.}
	\label{fig:synthetic_loss_acc_adaptive}
\end{figure}

\paragraph{Nonconvex loss.} To showcase the results for a nonconvex loss, we will use CIFAR10 dataset that is distributed in a non-IID way with $2$ classes per client. We apply it to a CNN model with $2$ convolution layers, followed by $2$ fully connected layers, using cross entropy as the loss function. As it can be inferred from the results in Table~\ref{tab:cifar_cnn}, even in the nonconvex loss function scenario, the personalized model learned from our \texttt{APFL} algorithm outperforms the localized  models of FedAvg and SCAFFOLD, as well as the global models, in both optimization and generalization. In this case adaptively tuning the $\alpha$ seems to achieve the best training loss, while $\alpha=0.25$ case achieves the best generalization performance. The initial learning rates of \texttt{APFL} and FedAvg algorithms are set to $0.1$ with the mentioned decay structure, while for SCAFFOLD this value is $0.05$ with $5\%$ decay per iteration to avoid divergence.

\begin{table}[]
\setstretch{1.25}
\centering
\resizebox{1\textwidth}{!}{
\begin{tabular}{l|c|c|c|c|c|c|c|c|}
\cline{2-9}
& \multicolumn{4}{c|}{\texttt{APFL}}  & \multicolumn{2}{c|}{FedAvg} & \multicolumn{2}{c|}{SCAFFOLD}                                             \\ \cline{2-9} 
 & $\alpha = 0.25$ & $\alpha = 0.5$ & $\alpha = 0.75$ & Adaptive $\alpha$ & Global Model  & Localized Model & Global Model   & Localized Model \\ \hline
\multicolumn{1}{|l|}{Training Loss} &
  \begin{tabular}[c]{@{}c@{}}$0.154 \pm$\\ $0.003$\end{tabular} &
  \begin{tabular}[c]{@{}c@{}}$0.113 \pm$\\ $0.008$\end{tabular} &
  \begin{tabular}[c]{@{}c@{}}$0.103 \pm$\\ $0.007$\end{tabular} &
  \begin{tabular}[c]{@{}c@{}}$\mathbf{0.101 \pm}$\\ $\mathbf{0.013}$\end{tabular} &
  \begin{tabular}[c]{@{}c@{}}$1.789 \pm$\\ $0.004$\end{tabular} &
  \begin{tabular}[c]{@{}c@{}}$0.369 \pm$\\ $0.005$\end{tabular} &
  \begin{tabular}[c]{@{}c@{}}$1.70 \pm$\\ $0.001$\end{tabular} &
  \begin{tabular}[c]{@{}c@{}}$0.593 \pm$\\ $0.012$\end{tabular} \\ \hline
\multicolumn{1}{|l|}{Validation Accuracy} &
  \begin{tabular}[c]{@{}c@{}}$\mathbf{89.33\%  \pm}$\\ $\mathbf{0.26\%}$\end{tabular} &
  \begin{tabular}[c]{@{}c@{}}$88.74\% \pm$\\ $0.14\%$\end{tabular} &
  \begin{tabular}[c]{@{}c@{}}$89.04\% \pm$\\ $0.22\%$\end{tabular} &
  \begin{tabular}[c]{@{}c@{}}$88.87\% \pm$\\ $0.51\%$\end{tabular} &
  \begin{tabular}[c]{@{}c@{}}$32.51\% \pm$\\ $0.47\%$\end{tabular} &
  \begin{tabular}[c]{@{}c@{}}$83.16\%  \pm$\\ $0.37\%$\end{tabular} &
  \begin{tabular}[c]{@{}c@{}}$37.16\% \pm$\\ $0.3\%$\end{tabular} &
  \begin{tabular}[c]{@{}c@{}}$85.25\% \pm$\\ $0.2\%$\end{tabular} \\ \hline
\end{tabular}
}
\vspace{0.1cm}
\caption{The results of training CIFAR10 dataset on a CNN model using different algorithms after 100 rounds of communication. \texttt{APFL} models outperform both localized and global models of both FedAvg and SCAFFOLD. Bold fonts show the best results in each row.}\label{tab:cifar_cnn}
\end{table}

\paragraph{Natural heterogeneous data.} In addition to the CIFAR10 dataset with a pathological heterogeneous data distribution, we apply our algorithm on a natural heterogeneous dataset, EMNIST~\citep{caldas2018leaf}. We use the data from $1000$ clients, and for each round of communication we randomly select $10\%$ of clients to participate in the training. We use an MLP model with $2$ hidden layers, each with $200$ neurons and ReLU as the activation function, using cross entropy as the loss function. For \texttt{APFL}, we use the adaptive $\alpha$ scheme with initial value of $0.5$ for each client. We run both algorithms for $250$ rounds of communication. In each round, each online client performs the local updates for $1$ epoch on its data. Figure~\ref{fig:emnist_loss_acc} shows the results of this experiment for personalized model of \texttt{APFL} and the localized model of the FedAvg. \texttt{APFL} with adaptive $\alpha$ can reach to the same training loss of the local FedAvg, while greatly outperforms the local FedAvg model in generalization on local validation data.
\begin{figure}[ht!]
	\centering
	\subfigure{
		\centering
		\includegraphics[width=0.4\textwidth]{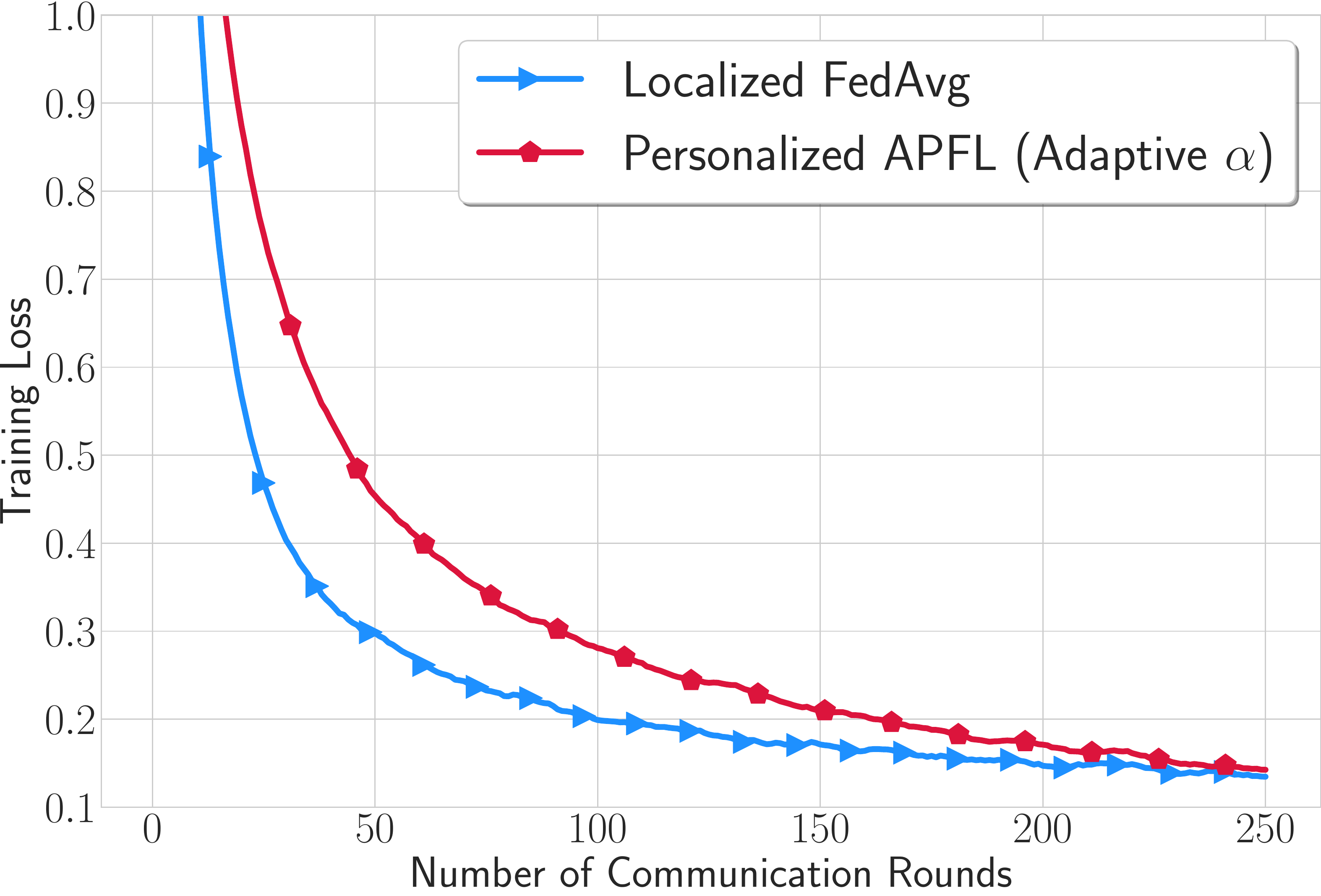}
		\label{fig:emnist_loss_comm}
		}
		\hspace{1cm}
		\subfigure{
			\centering 
			\includegraphics[width=0.4\textwidth]{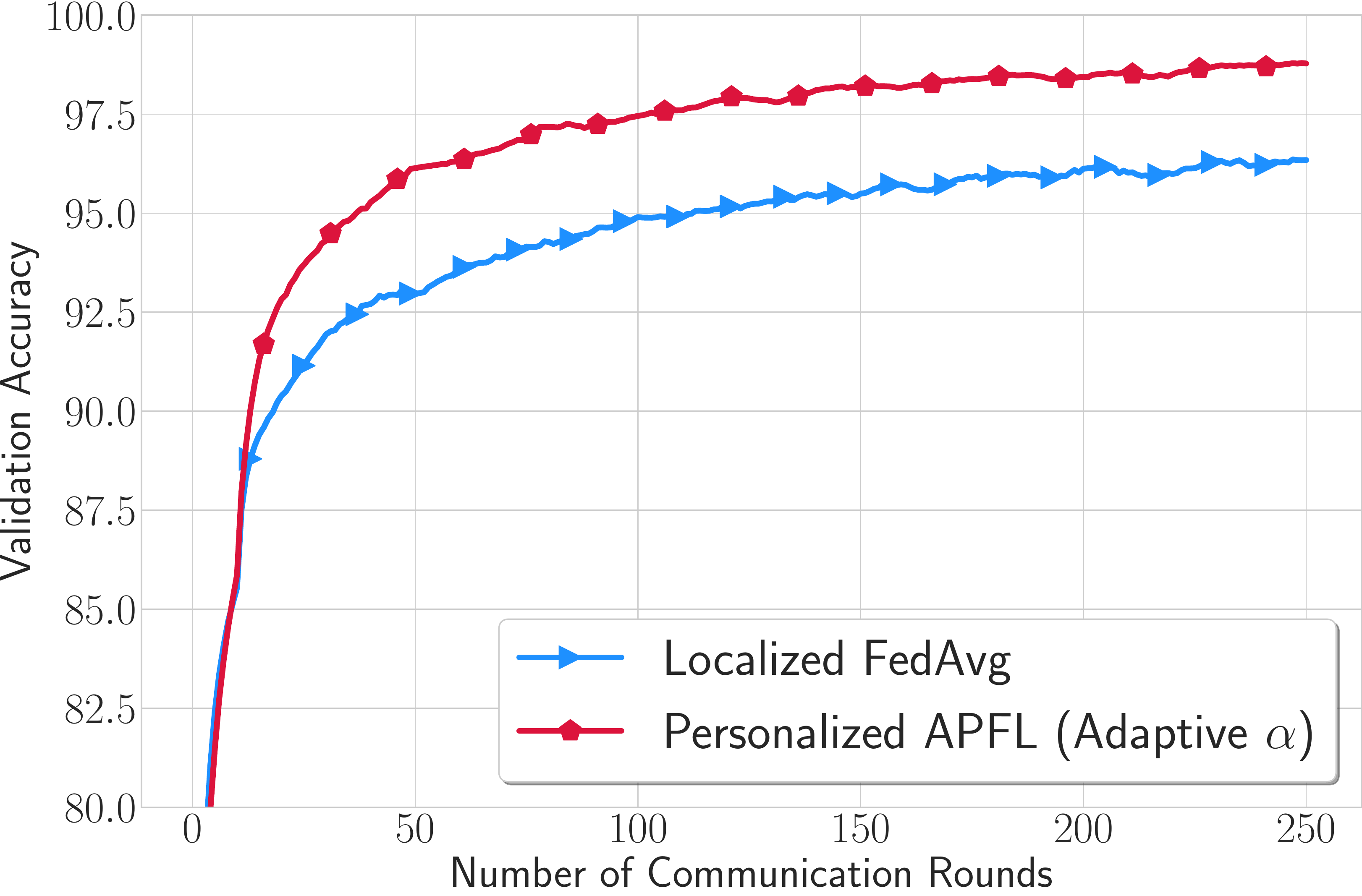}
			\label{fig:emnist_acc_comm}
			}
	\caption[]{The results of applying FedAvg and \texttt{APFL} (with adaptive $\alpha$) on an MLP model using EMNIST dataset, which is naturally heterogeneous. \texttt{APFL} achieves the same training loss of localized FedAVG, while outperforms it in validation accuracy.}
	\label{fig:emnist_loss_acc}
\end{figure}

\paragraph{Comparison with other personalization methods.} In this part, we compare our proposed \texttt{APFL} with two recent approaches of personalization in federated learning. In addition to FedAvg, we compare with perFedAvg introduced by~\cite{fallah2020personalized} using a meta-learning approach, and pFedMe introduce by~\cite{dinh2020personalized} using a regularization with Moreau envelope function. We run these algorithms to train an MLP with $2$ hidden layers, each with $200$ neurons similar to EMNIST case, on a non-IID MNIST dataset with $2$ classes per client. For perFedAvg, similar to their setting, we use learning rates of $\alpha = 0.01$ (different from the $\alpha$ in our \texttt{APFL}) and $\beta = 0.001$. To have a fair comparison, we use the same validation for perFedAvg in our algorithm and we use $10\%$ of training data as the test dataset that updates the meta-model. For pFedMe, following their setting, we use $\lambda = 15$, $\eta=0.01$, and $K=5$. For all experiments we use $\tau=20$ with total number of communications to $100$ and the batch size is $20$. The results of these experiments are presented in Table~\ref{tab:MNIST_MLP}, where \texttt{APFL} clearly outperforms all other models in both training and generalization. Among different \texttt{APFL} models the one with $\alpha=0.75$ has the lowest training loss, while the one with adaptive $\alpha$ has the best validation accuracy. perFedAvg is slightly better than the localized model of FedAvg, however, it is worse than the \texttt{APFL} models. While pFedMe performs better than the global model of FedAvg, it cannot surpass neither the localized model of the FedAvg nor \texttt{APFL} models. 

\begin{table}[]
\setstretch{1.25}
\resizebox{1\textwidth}{!}{
\begin{tabular}{l|c|c|c|c|c|c|c|c|}
\cline{2-9}
\multicolumn{1}{c|}{} & \multicolumn{4}{c|}{\texttt{APFL}} & \multicolumn{2}{c|}{FedAvg}  & perFedAvg  & pFedMe            \\ \cline{2-9} 
\multicolumn{1}{c|}{} & $\alpha = 0.25$ & $\alpha = 0.5$ & $\alpha = 0.75$ & Adaptive $\alpha$ & Global Model & Localized Model & Personalized Model & Personalized Model \\ \hline
\multicolumn{1}{|l|}{Training Loss} &
  \begin{tabular}[c]{@{}c@{}}$0.011 \pm$\\ $0.0007$\end{tabular} &
  \begin{tabular}[c]{@{}c@{}}$0.004 \pm$\\ $0.0004$\end{tabular} &
  \begin{tabular}[c]{@{}c@{}}$\mathbf{0.002 \pm}$\\ $\mathbf{0.0001}$\end{tabular} &
  \begin{tabular}[c]{@{}c@{}}$0.004 \pm$\\ $0.0008$\end{tabular} &
  \begin{tabular}[c]{@{}c@{}}$0.240 \pm$\\ $0.006$\end{tabular} &
  \begin{tabular}[c]{@{}c@{}}$0.041 \pm$\\ $0.002$\end{tabular} &
  \begin{tabular}[c]{@{}c@{}}$0.039 \pm$\\ $0.002$\end{tabular} &
  \begin{tabular}[c]{@{}c@{}}$0.182 \pm$\\ $0.004$\end{tabular} \\ \hline
\multicolumn{1}{|l|}{Validation Accuracy} &
  \begin{tabular}[c]{@{}c@{}}$98.07\% \pm$\\ $0.10\%$\end{tabular} &
  \begin{tabular}[c]{@{}c@{}}$98.04\% \pm$\\ $0.08\%$\end{tabular} &
  \begin{tabular}[c]{@{}c@{}}$97.86\% \pm$\\ $0.09\%$\end{tabular} &
  \begin{tabular}[c]{@{}c@{}}$\mathbf{98.10\% \pm}$\\ $\mathbf{0.10\%}$\end{tabular} &
  \begin{tabular}[c]{@{}c@{}}$93.81\% \pm$\\ $0.29\%$\end{tabular} &
  \begin{tabular}[c]{@{}c@{}}$97.75\% \pm$\\ $0.15\%$\end{tabular} &
  \begin{tabular}[c]{@{}c@{}}$97.83\% \pm$\\ $0.12\%$\end{tabular} &
  \begin{tabular}[c]{@{}c@{}}$95.92\% \pm$\\ $0.1\%$\end{tabular} \\ \hline
\end{tabular}
}
\vspace{0.1cm}
\caption{The results of applying different personalization models on MNIST dataset with an MLP model after $100$ rounds of communication. Models learned by \texttt{APFL} can outperform other models in both training loss and generalization accuracy learned by FedAvg, perFedAvg~\citep{fallah2020personalized}, and pFedMe~\citep{dinh2020personalized}. Bold fonts show the best results among different models.}\label{tab:MNIST_MLP}
\end{table} 
\section{Discussion and Extensions}\label{sec:discussion}
\paragraph{Connection between learning guarantee and convergence.}
As Theorem~\ref{thm:generalization} suggests, the generalization bound depends on the divergence of the local and global distributions. In the language of optimization, the counter-part of divergence of distribution is the gradient diversity; hence, the gradient diversity appears in our empirical loss convergence rate (Theorem~\ref{localpartial}). The other interesting discovery is in the generalization bound, we have the term $\lambda_{\mathcal{H}}$ and $\mathcal{L}_{\mathcal{D}_i}(h_i^*)$, which are intrinsic to the distributions and hypothesis class. Meanwhile, in the convergence result, we have the term $\|\bm{v}^*_i - \bm{w}^*\|^2$, which also only depends on the data distribution and hypothesis class we choose. In addition, $\|\bm{v}^*_i - \bm{w}^*\|^2$ also reveals the divergence between local  and global optimal solutions.

\paragraph{Why \texttt{APFL} is ``Adaptive''.}
Both information-theoretically (Theorem~\ref{thm:generalization}) and computationally (Theorem~\ref{localpartial}), we prove that when the local distribution drifts far away from the average distribution, the global model does not contribute too much to improve the local generalization and we have to tune the mixing parameter $\alpha$ to a larger value. Thus it is necessary to make $\alpha$ updated adaptively during empirical risk minimization. In Section \ref{sec:alpha_adap},~(\ref{eq:alpha_gd}) shows that the update of $\alpha$ depends on the correlation of local gradient and deviation between local and global models. Experimental results show that our method can adaptively tune $\alpha$, and can outperform the training scheme using fixed $\alpha$. 

\paragraph{Comparison with local ERM model.} A crucial question about personalization is \emph{when it is preferable to employ a mixed model?}, and  \emph{how bad a local ERM model will be?} In the following corollary, we answer this by showing that the risk of local ERM model  can be strictly worse than that of our personalized model. 

\begin{corollary}
\label{coro: generalization gap}
 Continuing with Theorem~\ref{thm:generalization},  there exist a distribution $\mathcal{D}_i$, constant $C_1$ and $C_2$, such that with probability at least $1-\delta$, the following upper bound  for the difference between risks of personalized model $h_{\alpha_i}$ and local ERM model $\hat{h}_{i}^*$ on $\mathcal{D}_i$, holds :
{\begin{equation}
\label{eq: generalization gap}
\begin{aligned}
\mathcal{L}_{\mathcal{D}_i}(h_{\alpha_i}) - \mathcal{L}_{\mathcal{D}_i}(\hat{h}_{i}^*)  &\leq (2\alpha_i^2-1)   \mathcal{L}_{\mathcal{D}_i}( h_{i}^* ) + (2\alpha_i^2 C_1-C_2) \sqrt{\frac{d+ \log (1/\delta)}{m_i}}  + 2\alpha_i^2 G \lambda_{\mathcal{H}}(\mathcal{S}_i) \nonumber\\
     & \quad + 2(1-\alpha_i)^2\left(\hat{\mathcal{L}}_{ \bar{\mathcal{D}}}(\bar{h}^*)+B\|\bar{\mathcal{D}}-\mathcal{D}_i\|_1 +
 C_1\sqrt{\frac{d+\log (1/\delta)}{m}}\right). 
\end{aligned}
\end{equation}} 
\end{corollary} 
 By examining the above bound, the personalized model is preferable to local model if this value is less than 0. In this case,  we require $(2\alpha^2-1)$ and $(2\alpha_i^2 C_1-C_2)$ to be negative, which is  satisfied by choosing $\alpha_i \leq \min\{\frac{\sqrt{2}}{2},  \sqrt{\frac{C_2}{2C_1}} \}$. Then, the term $  \sqrt{\frac{d+ \log (1/\delta)}{m_i}}$, should be sufficiently large, and the divergence term, as well as the global model generalization error has to be small. In this case, from the local model perspective, it can benefit from incorporate some global model. Using the similar technique, we can prove the supremacy of mixed model over global model as well. 

\begin{proof} [Proof of Corollary~\ref{coro: generalization gap}]
 
Since in Theorem~\ref{thm:generalization}, we already obtained upper bound for $\mathcal{L}_{\mathcal{D}_i}(h_{\alpha_i})$ as following,
{\begin{equation} 
\begin{aligned}
\mathcal{L}_{\mathcal{D}_i}(h_{\alpha_i}) 
&\leq 2\alpha_i^2\left(\mathcal{L}_{\mathcal{D}_i}( h_{i}^* ) + 2C_1 \sqrt{\frac{d+ \log (1/\delta)}{m_i}} + G\lambda_{\mathcal{H}}(\mathcal{S}_i) \right)\\
&\quad +2(1-\alpha_i)^2\left(\hat{\mathcal{L}}_{ \bar{\mathcal{D}}}(\bar{h}^*)+B\|\bar{\mathcal{D}}-\mathcal{D}_i\|_1 +
C_1 \sqrt{\frac{d+ \log (1/\delta)}{m}}\right),\nonumber
\end{aligned}
\end{equation}}
to find the upper bound of $\mathcal{L}_{\mathcal{D}_i}(h_{\alpha_i}) - \mathcal{L}_{\mathcal{D}_i}(\hat{h}_{i}^*)$, we just need the lower bound of $\mathcal{L}_{\mathcal{D}_i}(\hat{h}_{i}^*)$. The fundamental theorem of statistical learning~\citep{shalev2014understanding,mohri2018foundations} states a lower risk bound for agnostic PAC learning: for a hypothesis class with finite VC dimension $d$, then there exists a distribution $\mathcal{D}$, such that for any learning algorithm, which learns a hypothesis $h \in \mathcal{H}$ on $m$ i.i.d. samples from $\mathcal{D}$, there exists a constant $C$, with the probability at least $1-\delta$, we have:
\begin{align}
    \mathcal{L}_{\mathcal{D}}(h) - \min_{h'\in\mathcal{H}}\mathcal{L}_{\mathcal{D}}(h') \geq C \sqrt{\frac{d+\log(1/\delta)}{m}}\nonumber.
\end{align}
Since $\hat{h}_{i}^*$ is learnt by ERM algorithm, the agnostic PAC learning lower risk bound also holds for it, so in worst case it might hold that under distribution $\mathcal{D}_i$, if  $\hat{h}_{i}^*$ is learnt by ERM algorithm using $m_i$ samples, then there is a $C_2$, such that with  probability at least $1-\delta$, we have:
\begin{align}
    \mathcal{L}_{\mathcal{D}_i}(\hat{h}_{i}^*)    \geq \mathcal{L}_{\mathcal{D}_i}(h_i^*) + C_2 \sqrt{\frac{d+\log(1/\delta)}{m_i}}\nonumber.
\end{align}
 Thus we can bound $\mathcal{L}_{\mathcal{D}_i}(h_{\alpha_i}) - \mathcal{L}_{\mathcal{D}_i}(\hat{h}_{i}^*)$ as Corollary~\ref{coro: generalization gap} claims.
\end{proof}

\paragraph{Personalization for new participant nodes.}
Suppose we already have a trained global model $\bm{\hat{w}}$, and now a new device $k$ joins in the network, which is desired to personalize the global model to adapt its own domain. This can be done by performing a few local stochastic gradient descent updates from the given global model as an initial local model:
\begin{align}
  \bm{v}^{(t+1)}_k =  \bm{v}^{(t)}_k - \eta_t \nabla_{\bm{v}} f_k (\alpha_k \bm{v}_k^{(t)} + (1-\alpha_k)\bm{\hat{w}};\xi_k^{(t)})
\end{align}
to quickly learn a personalized model for the newly joined device. One thing worthy of investigation is the difference between \texttt{APFL} and meta-learning approaches, such as model-agnostic meta-learning~\citep{finn2017model}. Our goal is to share the knowledge among the different users, in order to reduce the generalization error;  while meta-learning cares more about how to build a meta-learner, to help training models faster and with fewer samples. In this scenario, similar to FedAvg, when a new node joins the network, it gets the global model and takes a few stochastic steps based on its own data to update the global model. In Figure~\ref{fig:synthetic_loss_acc_meta}, we show the results of applying FedAvg and \texttt{APFL} on synthetic data with two different rates of diversity, \texttt{synthetic$\left(0.0,0.0\right)$} and \texttt{synthetic$\left(0.5,0.5\right)$}. In this experiment, we keep $3$ nodes with their data off in the entire training for $100$ rounds of communication between $97$ nodes. In each round, each client updates its local and personalized models for one epoch. After the training is done, those $3$ clients will join the network and get the latest global model and start training local and personalized models of their own. Figure~\ref{fig:synthetic_loss_acc_meta} shows the training loss and validation accuracy of these $3$ nodes during the $5$ epochs of updates. The local model represents the model that will be trained in FedAvg, while the personalized model is the one resulting from \texttt{APFL}. Although the goal of \texttt{APFL} is to adaptively learn the personalized model during the training, it can be inferred that \texttt{APFL} can learn a better personalized model in a meta-learning scenario as well.

\begin{figure}[ht!]
	\centering
	\subfigure{
		\centering
		\includegraphics[width=0.4\textwidth]{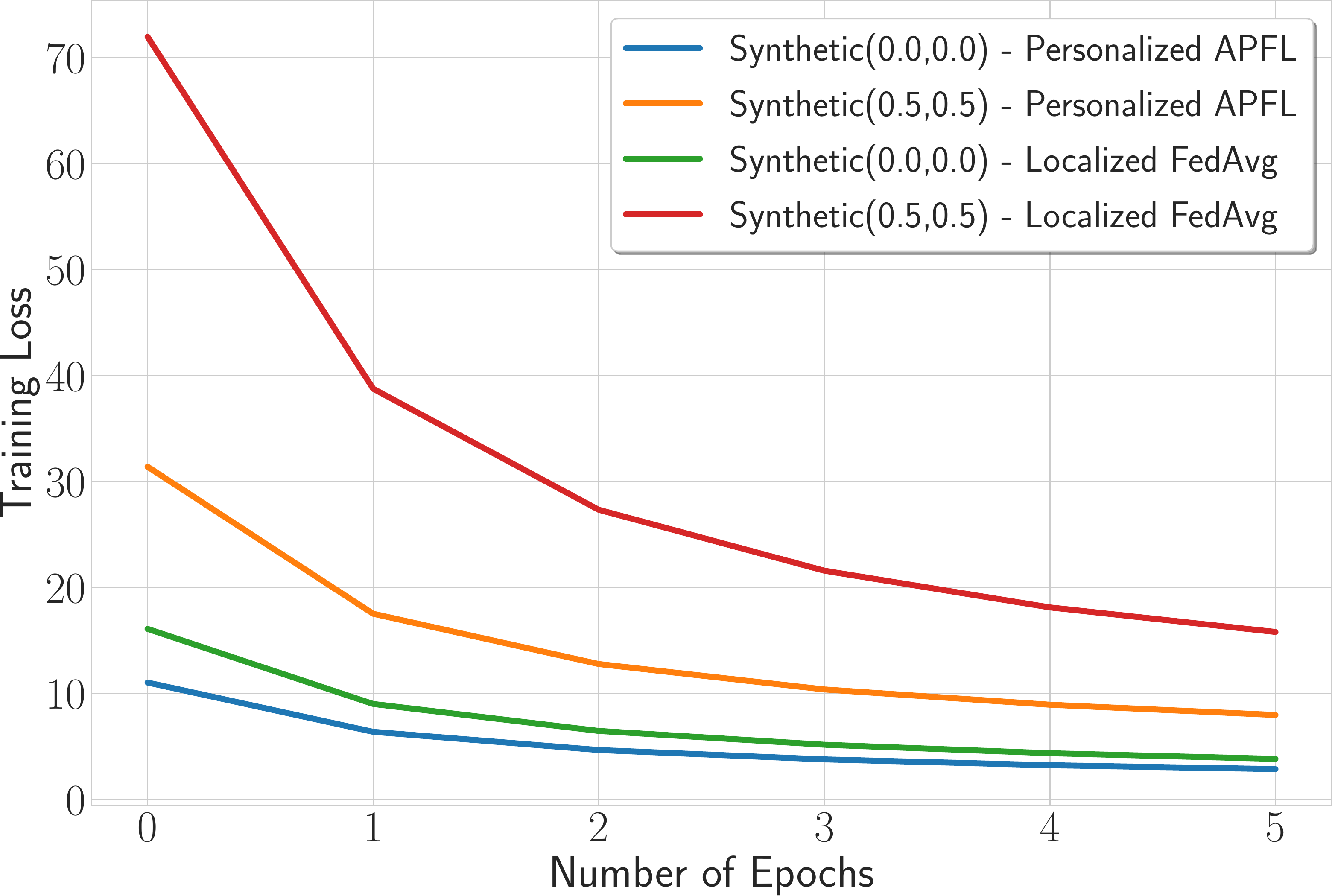}
		\label{fig:synthetic_loss_comm_meta}
		}
		\hspace{1cm}
		\subfigure{
			\centering 
			\includegraphics[width=0.4\textwidth]{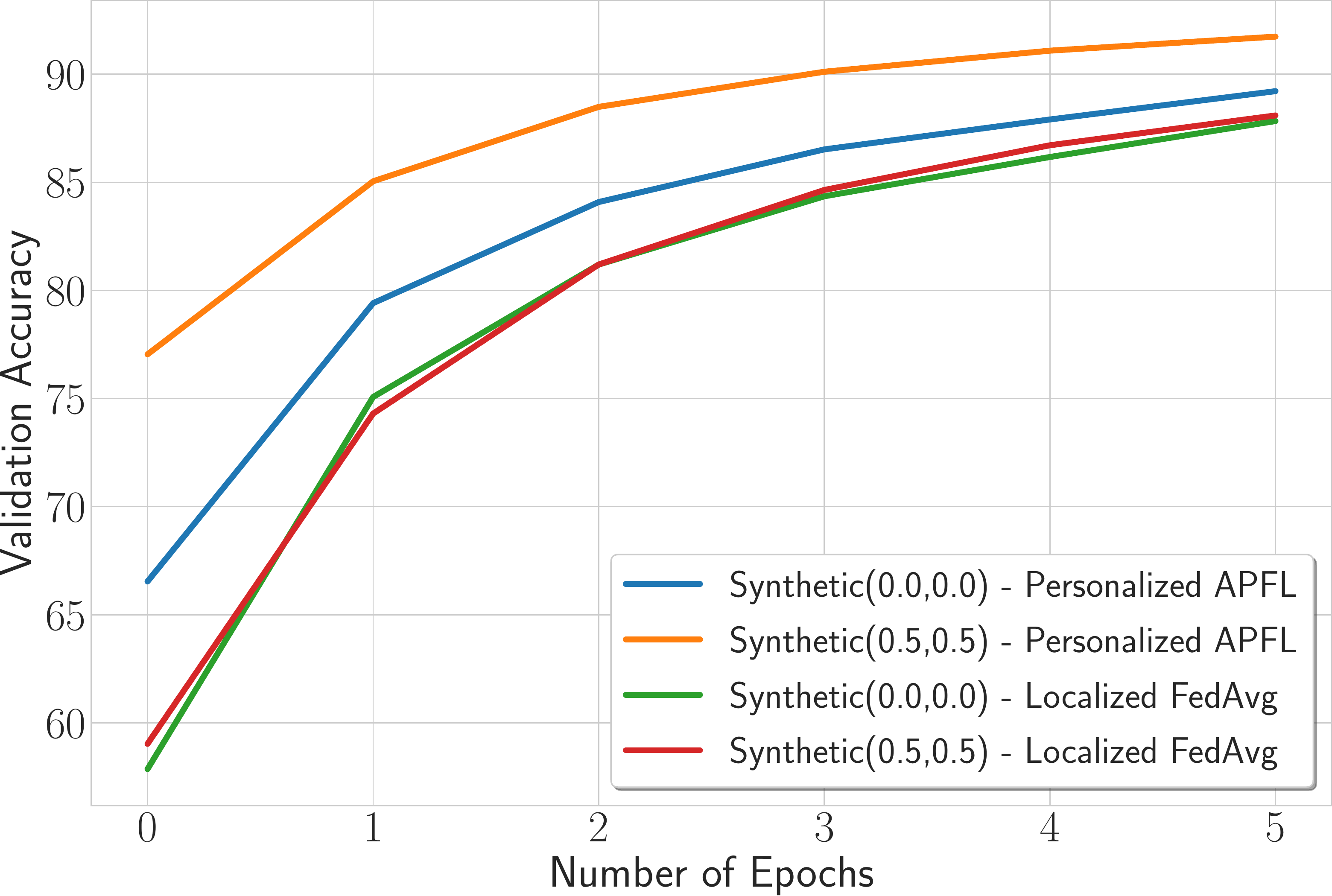}
			\label{fig:synthetic_acc_comm_meta}
			}
	\caption[]{Comparing the effect of fine-tuning with the local model of FedAvg and with the personalized model of \texttt{APFL} on the synthetic datasets. The model is trained for $100$ rounds of communication with $97$ clients, and then $3$ clients will join in fine-tuning the global model based on their own data. It can be seen that the model from \texttt{APFL} can better personalize the global model with respect to the FedAvg method both in training loss and validation accuracy. Increasing diversity makes it harder to personalize, however, \texttt{APFL} surpasses FedAvg again.}
	\label{fig:synthetic_loss_acc_meta}
	\vspace{-0.5cm}
\end{figure}

\paragraph{Agnostic global model.}
As pointed out by~\citet{mohri2019agnostic}, the global model can be distributionally robust if we optimize the agnostic loss:
\begin{align}
\min_{\bm{w}\in \mathbb{R}^d} \max_{\bm{q} \in \Delta_n} F(\boldsymbol{{w}}) :=  \sum_i^n q_i f_i(\boldsymbol{{w}}),
\end{align}
where $\Delta_n = \{\boldsymbol{q} \in \mathbb{R}_{+}^{n}\ \;|\; \sum_{}^{}{q_i} = 1 \}$ is the $n$-dimensional simplex. We call this scenario ``Adaptive Personalized Agnostic Federated Learning''.
In this case, the analysis will be more challenging since the global empirical risk minimization is performed at a totally different domain, so the risk upper bound for $h_{\alpha_i}$ we derived does not hold anymore. Also, from a computational standpoint, since the resulted problem is a minimax optimization problem,  the convergence analysis of agnostic \texttt{APFL} will be more involved, which we will leave as an interesting future work. 
\section{Conclusion and Future Work}\label{sec:conclusion}
In this paper, we proposed an adaptive federated learning algorithm that learns a mixture of local and global models as the personalized model. Motivated by learning theory in domain adaptation, we provided generalization guarantees for our algorithm that demonstrated the dependence on the diversity between each clients' data distribution and the representative sample of the overall distribution of data, and the number of per-device samples as key factors in personalization. Moreover, we proposed a communication-reduced optimization algorithm to learn the personalized models and analyzed its convergence rate  for both smooth strongly convex and nonconvex functions. Finally, we empirically backed up our theoretical results by conducting  experiments in a federated setting.

\section*{Acknowledgment}
We would like to thank Farzin Haddadpour for insightful discussions  in the early stage of this project.
\bibliography{ref.bib}
\bibliographystyle{plainnat}

\newpage
\onecolumn
\appendix
\clearpage
\section{Proof of Generalization Bound}
\label{app:proof_gen}
In this section we  present the proof of generalization bound for APFL algorithm.  Recall that we define the following hypotheses on $i$th local true and empirical distributions:
\begin{equation*}
\begin{aligned}
    \hat{h}_{i}^* &= \arg \min_{h\in\mathcal{H}} \hat{\mathcal{L}}_{\mathcal{D}_i}(h) && (\textsc{Local Empirical  Risk Minimizer} )\\
    h_{i}^* &= \arg \min_{h\in\mathcal{H}} \mathcal{L}_{\mathcal{D}_i}(h) && (\textsc{Local True Risk Minimizer})\\
    \bar{h}^* &= \arg \min_{h\in\mathcal{H}} \hat{\mathcal{L}}_{\bar{\mathcal{D}}}(h) && (\textsc{Global Empirical Risk Minimizer})\\
    \hat{h}_{loc,i}^* &= \arg \min_{h\in \mathcal{H}} \hat{\mathcal{L}}_{\mathcal{D}_i} ( \alpha_i h + (1-\alpha_i) \bar{h}^*) && (\textsc{Mixed  Empirical Risk Minimizer})\\
    h_{loc,i}^* &= \arg \min_{h\in \mathcal{H}} \mathcal{L}_{\mathcal{D}_i} ( \alpha_i h + (1-\alpha_i) \bar{h}^*) && (\textsc{Mixed True  Risk Minimizer})
\end{aligned}
\end{equation*}
where $\hat{\mathcal{L}}_{\mathcal{D}_i}(h)$ and ${\mathcal{L}}_{\mathcal{D}_i}(h)$ denote the empirical and true risks on $\mathcal{D}_i$, respectively.

From a high-level technical view, since we wish to bound the risk of the mixed model on local distribution $\mathcal{D}_i$, first we need to utilize the convex property of the risk function, and decompose it into two parts: $ \mathcal{L}_{\mathcal{D}_i}\left( \hat{h}_{loc,i}^* \right)   $ and $  \mathcal{L}_{\mathcal{D}_i}\left( \bar{h}^* \right)$. To bound $ \mathcal{L}_{\mathcal{D}_i}\left( \hat{h}_{loc,i}^* \right)$, a natural idea is to characterize it by the risk of optimal model $\mathcal{L}_{\mathcal{D}_i}\left(  {h}_{ i}^* \right)$, plus some excess risk. However,  due to fact that $\hat{h}_{loc,i}^* $ is not the sole  local empirical risk minimizer, rather it partially incorporates the global model, we need to characterize to what extent  it drifts from the local empirical risk minimizer $\hat{h}_{i}^*$. This drift can be depicted by the hypothesis capacity, so that is our motivation to define $\lambda_{\mathcal{H}}(\mathcal{S})$ to quantify the empirical loss discrepancy over $\mathcal{S}$ among pair of hypotheses in $\mathcal{H}$. We have to admit that there should be a tighter theory to bound this drift, depending how global model is  incorporated, which we leave it as a future work.

The following simple result will be useful in the proof of generalization.
\begin{lemma} \label{lm1}
 Let $\mathcal{H}$ be a hypothesis class and  $\mathcal{D}$ and $\mathcal{D}'$ denote two probability measures  over space $\Xi$. Let  $\mathcal{L}_{\mathcal{D}} (h) = \mathbb{E}_{(\bm{x}, y) \sim \mathcal{D}}\left[\ell\left(h(\bm{x}), y\right)\right]$ denote the risk of $h$ over $\mathcal{D}$ . If the loss function $\ell(\cdot)$ is bounded by $B$, then for every $h \in \mathcal{H}$:
\begin{equation}
    \mathcal{L}_{\mathcal{D} }(h) \leq  \mathcal{L}_{\mathcal{D}'}(h)+ B\|\mathcal{D} - \mathcal{D}'\|_1,
\end{equation} 
where $\| \mathcal{D}-\mathcal{D}'\|_1 = \int_{\Xi} |\mathbb{P}_{(\bm{x}, y)\sim \mathcal{D}}-\mathbb{P}_{(\bm{x}, y)\sim\mathcal{D}'}|d\bm{x}dy$.
\begin{proof}
\begin{align*}
    & \mathcal{L}_{\mathcal{D}}(h) \leq \mathcal{L}_{\mathcal{D}'}(h) + |\mathcal{L}_{\mathcal{D}}(h)-\mathcal{L}_{\mathcal{D}'}(h) |\\
     & \qquad \quad \leq \mathcal{L}_{\mathcal{D}}(h) + \int_{\Xi}| \ell(y,h(\bm{x}))|   |\mathbb{P}_{(\bm{x}, y)\sim  \mathcal{D}}-\mathbb{P}_{(\bm{x}, y)\sim\mathcal{D}'}|d\bm{x}dy\\
     & \qquad \quad = \mathcal{L}_{\mathcal{D}}(h) + B\|\mathcal{D} - \mathcal{D}'\|_1.
\end{align*} 
\end{proof}
\end{lemma}

\paragraph{Proof of Theorem~\ref{thm:generalization}}
We now turn to  proving the generalization bound for the proposed APFL algorithm. Recall that for the classification task we consider squared hinge loss, and for the regression case we consider MSE loss. We will first prove that in both cases we can decompose the risk as follows:
  \begin{align}
\mathcal{L}_{\mathcal{D}_i}(h_{\alpha_i}^*)  \leq 2\alpha_i^2 \mathcal{L}_{\mathcal{D}_i}\left( \hat{h}_{loc,i}^* \right)   + 2(1-\alpha_i)^2 \mathcal{L}_{\mathcal{D}_i}\left( \bar{h}^*(\bm{x}) \right) .  
\end{align} 
We start with the classification case first. Note that, hinge loss: $\max\{0, 1-z\} $ is convex in $z$, so $\max\{0, 1-y(\alpha_i h + (1-\alpha_i)h')\} \leq \alpha_i \max\{0, 1-yh \} + (1-\alpha_i)\max\{0, 1-yh'\}$, according to Jensen's inequality. Hence, we have:
 \begin{align}
\mathcal{L}_{\mathcal{D}_i}(h_{\alpha_i}^*) &=  \mathcal{L}_{\mathcal{D}_i}(\alpha_i \hat{h}_{loc,i}^* + (1-\alpha_i)\bar{h}^*)\nonumber\\ 
&= \mathbb{E}_{(\bm{x},y)\sim \mathcal{D}_i}\left(\max\{0, 1-y(\alpha_i \hat{h}_{loc,i}^*(\bm{x}) + (1-\alpha_i)\bar{h}^*(\bm{x}))\} \right)^2\nonumber\\
&= \mathbb{E}_{(\bm{x},y)\sim \mathcal{D}_i}\left(\alpha_i \max\{0, 1-y\hat{h}_{loc,i}^*(\bm{x})  \} + (1-\alpha_i)\max\{0, 1-y\bar{h}^*(\bm{x})\}\right)^2\nonumber\\
&\leq 2\alpha_i^2\mathbb{E}_{(\bm{x},y)\sim \mathcal{D}_i}\left(\max\{0, 1-y\hat{h}_{loc,i}^*(\bm{x})  \}\right)^2\nonumber \\& \quad + 2(1-\alpha_i)^2\mathbb{E}_{(\bm{x},y)\sim \mathcal{D}_i}\left(\max\{0, 1-y\bar{h}^*(\bm{x})\}\right)^2\nonumber\\
&\leq 2\alpha_i^2 \mathcal{L}_{\mathcal{D}_i}\left( \hat{h}_{loc,i}^* \right)   + 2(1-\alpha_i)^2 \mathcal{L}_{\mathcal{D}_i}\left( \bar{h}^* \right) .  \nonumber
\end{align}

For regression case:
 \begin{equation*}
\begin{aligned}
\mathcal{L}_{\mathcal{D}_i}(h_{\alpha_i}^*) &=  \mathcal{L}_{\mathcal{D}_i}(\alpha_i \hat{h}_{loc,i}^* + (1-\alpha_i)\bar{h}^*)\nonumber\\ 
&= \mathbb{E}_{(\bm{x},y)\sim \mathcal{D}_i}\left\| y-(\alpha_i \hat{h}_{loc,i}^*(\bm{x}) + (1-\alpha_i)\bar{h}^*(\bm{x}))\right\|^2\\ 
&= \mathbb{E}_{(\bm{x},y)\sim \mathcal{D}_i}\left\|\alpha_i y- \alpha_i \hat{h}_{loc,i}^*(\bm{x}) +(1-\alpha_i)y- (1-\alpha_i)\bar{h}^*(\bm{x})\right\|^2\\ 
&\leq 2\alpha_i^2\mathbb{E}_{(\bm{x},y)\sim \mathcal{D}_i}\left\|y-  \hat{h}_{loc,i}^*(\bm{x})\right\|^2 + 2(1-\alpha_i)^2\mathbb{E}_{(\bm{x},y)\sim \mathcal{D}_i}\left\|y- \bar{h}^*(\bm{x})\right\|^2\\
&\leq 2\alpha_i^2 \mathcal{L}_{\mathcal{D}_i}\left( \hat{h}_{loc,i}^* \right)  + 2(1-\alpha_i)^2 \mathcal{L}_{\mathcal{D}_i}\left( \bar{h}^* \right)  
\end{aligned}
\end{equation*}
 Thus we can conclude:
 \begin{align}
\mathcal{L}_{\mathcal{D}_i}(h_{\alpha_i}^*) \leq 2\alpha_i^2\underbrace{\mathcal{L}_{\mathcal{D}_i}\left( \hat{h}_{loc,i}^* \right)}_{T_1}  + 2(1-\alpha_i)^2\underbrace{\mathcal{L}_{\mathcal{D}_i}\left( \bar{h}^* \right)}_{T_2}. \label{eq: T2}
\end{align}
We proceed to bound the terms $T_1$ and $T_2$ in RHS of above inequality. We first bound $T_1$ as follows. The first step is to utilize uniform VC dimension error bound over $\mathcal{H}$ ~\cite{mohri2018foundations}: 
\begin{align}
    \forall h \in \mathcal{H}, |\mathcal{L}_{\mathcal{D}_i}(h)-\hat{\mathcal{L}}_{\mathcal{D}_i}(h)|\leq C\sqrt{\frac{d+\log (1/\delta)}{ m_i}},\nonumber
\end{align}
where $C$ is constant factor.
So we can bound $T_1$ as:
\begin{equation*}
\begin{aligned}
   T_1 &= \mathcal{L}_{\mathcal{D}_i} ( \hat{h}_{loc,i}^*  ) =\mathcal{L}_{\mathcal{D}_i} ( h_{ i}^*  )+ \mathcal{L}_{\mathcal{D}_i}( \hat{h}_{loc,i}^*)-\mathcal{L}_{\mathcal{D}_i}(h_{ i}^*) \\
   & = \mathcal{L}_{\mathcal{D}_i} ( h_{ i}^*)\\
   & \quad + \underbrace{\mathcal{L}_{\mathcal{D}_i}( \hat{h}_{loc,i}^*)-\hat{\mathcal{L}}_{\mathcal{D}_i}( \hat{h}_{loc,i}^*)}_{\leq C \sqrt{\frac{d+ \log (1/\delta)}{m_i}}} +  {\hat{\mathcal{L}}_{\mathcal{D}_i}( \hat{h}_{loc,i}^*) - \hat{\mathcal{L}}_{\mathcal{D}_i}( h_{ i}^*)}
   +\underbrace{\hat{\mathcal{L}}_{\mathcal{D}_i}(h_{ i}^*)-\mathcal{L}_{\mathcal{D}_i}(h_{ i}^*)}_{\leq  C \sqrt{\frac{d+ \log (1/\delta)}{m_i}}} \\
   & \leq \mathcal{L}_{\mathcal{D}_i} ( h_{ i}^*  )+ 2C \sqrt{\frac{d+ \log (1/\delta)}{m_i}} +\hat{\mathcal{L}}_{\mathcal{D}_i}( \hat{h}_{loc,i}^*) - \hat{\mathcal{L}}_{\mathcal{D}_i}( \hat{h}_{i}^*).
\end{aligned}
\end{equation*}

 Note that 
 \begin{align}
    \hat{\mathcal{L}}_{\mathcal{D}_i}(\hat{h}_{loc,i}^*) - \hat{\mathcal{L}}_{\mathcal{D}_i}(\hat{h}_i^*)  \leq G\frac{1}{|\mathcal{S}_i|}\sum_{(\bm{x},y)\in \mathcal{S}_i}  |\hat{h}_{loc,i}^*(\bm{x})-\hat{h}_i^*(\bm{x})| \leq G \lambda_{\mathcal{H}} (\mathcal{S}_i),\nonumber
\end{align}
As a result  we can bound $T_1$ by:
\begin{align}
   T_1   \leq \mathcal{L}_{\mathcal{D}_i} ( h_{ i}^*  )+ 2C \sqrt{\frac{d+ \log (1/\delta)}{m_i}} +G \lambda_{\mathcal{H}} (\mathcal{S}_i).\nonumber
\end{align}
 We now turn to bounding  $T_2$. Plugging Lemma~\ref{lm1} in~(\ref{eq: T2}) and using uniform generalization risk bound will immediately give:
 \begin{align}
     T_2 \leq \hat{\mathcal{L}}_{\bar{\mathcal{D}}}(\bar{h}^*)+B^2\|\mathcal{D} - \bar{\mathcal{D}}\|_1 + C\sqrt{\frac{d+\log (1/\delta)}{m}}.\nonumber
 \end{align}

 Plugging $T_1$ and $T_2$ back into (\ref{eq: T2})  concludes the proof.
 \qed 

\begin{remark}
One thing worth mentioning is that, we assume the customary boundedness of loss functions. Actually it can be satisfied if the data and the parameters of hypothesis are bounded. For example, considering the scenario where we are learning a linear model $\bm{w}$ with the constraint $\|\bm{w}\| \leq 1$, and also the data tuples $(\bm{x},y)$ are drawn from some bounded domain, then the loss  is obviously bounded by some finite real value.
\end{remark}

\section{Proof of Convergence}\label{app:convergence}
In this section, we present the proof of convergence raters. For ease of mathematical derivations, we first consider the case \textit{without sampling clients} at each communication step and then generalize the proof to the setting where \textit{$K$ devices are sampled uniformly at random} by the server as employed in the proposed algorithm.

\subsection{Proof without Sampling}
Before giving the convergence analysis of the Algorithm~\ref{algorithm:LDAPFL} in the main paper, we first discuss a warm-up case: local descent APFL without client sampling. As Algorithm~\ref{algorithm:LDAPFLo} shows, all clients will participate in the averaging stage every $\tau$ iterations. The convergence of global and local models in  Algorithm~\ref{algorithm:LDAPFLo} are given in the following theorems. We start by stating the convergence of  local model.
\begin{theorem}[Local model convergence of \texttt{Local Descent APFL without Sampling}]
\label{localAPFLo}
If each client's objective function satisfies Assumption \ref{assumption: smoothness}-\ref{assumption: strong convexity}, using Algorithm \ref{algorithm:LDAPFLo}, choosing the mixing weight $\alpha_i \geq \max\{ 1-\frac{1}{4\sqrt{6}\kappa}, 1-\frac{1}{4\sqrt{6}\kappa \sqrt{\mu}}\} $, learning rate $\eta_t = \frac{16}{\mu(t+a)}$, where $a = \max\{128\kappa,\tau\}$, and using average scheme $\hat{\bm{v}}_i = \frac{1}{S_T} \sum_{t=1}^T p_t (\alpha_i\bm{v}_i^{(t)}+(1-\alpha_i)\frac{1}{n}\sum_{j=1}^n\boldsymbol{{w}}_j^{(t)})$, where $p_t = (t+a)^2$, $S_T = \sum_{t=1}^T p_t$, and $f_i^*$ is the local minimum of the $i$th client, then the following convergence holds for all $i \in [n]$:
\begin{align}
     \mathbb{E}[f_i(\boldsymbol{\hat{v}} _i)] - f_i^* &= O\left(\frac{\mu  }{T^3}\right) +\alpha_i^2O\left(\frac{\sigma^2}{\mu T}\right)+(1-\alpha_i)^2O\left(\frac{\zeta_i}{\mu} +  \kappa L\Delta_i\right)\nonumber\\
      & \quad+ (1-\alpha_i)^2\left(O\left(\frac{\kappa L \ln T}{T^3}\right) +O\left(\frac{\kappa^2\sigma^2}{\mu nT}\right)+ O\left(\frac{\kappa^2\tau \left(\sigma^2 + \tau(\zeta_i + \frac{\zeta}{n})\right) }{\mu T^2} \right)+ O\left(\frac{\kappa^4 \tau\left(\sigma^2 + 2\tau \frac{\zeta}{n}  \right) }{\mu T^2}\right)  \right).\nonumber
\end{align}
\end{theorem}
\begin{algorithm2e}[t]
	\DontPrintSemicolon
    \caption{\texttt{Local Descent APFL (without sampling)}}
	\label{algorithm:LDAPFLo}
	\textbf{input:} Mixture weights $\alpha_1,\cdots, \alpha_n$, Synchronization gap $\tau$, Local models $\bm{v}^{(0)}_{i}$ for $i\in [n]$ and local version of global model $\bm{w}^{(0)}_{i}$ for $i\in [n]$.
	\\
	\For{$t=0,\cdots,T$}{

	    \eIf{t not divides $\tau$}{
	    ${\bm{w}}^{(t)}_i = {\bm{w}}^{(t-1)}_i - \eta_t \nabla f_i\left({\bm{w}}^{(t-1)}_i;\xi_i^t\right)$  \\
	    $\bm{v}^{(t)}_{i} = \bm{v}^{(t-1)}_i - \eta_t \nabla_{\bm{v}} f_i\left(\bar{\bm{v}}_i^{(t-1)}; \xi_i^t\right)$ \\
	    $ \bar{\bm{v}}_i^{(t)} = \alpha_i\bm{v}^{(t)}_{i} + (1-\alpha_i){\bm{w}}^{(t)}_i$\\
	    }{
	    {each client sends $\bm{w}_j^{(t)}$ to the server\\
	   
	    ${\bm{w}}^{(t)} = \frac{1}{n} \sum_{j=1}^n {\bm{w}}^{(t)}_j $  \\ 
	    server broadcast ${\bm{w}}^{(t)}$ to all clients\\
	   }

	    }
	
  }
    \For{$i = 1,\cdots, n$}{ \textbf{output:} Personalized model: $\hat{\bm{v}}_i = \frac{1}{S_T} \sum_{t=1}^T p_t (\alpha_i\bm{v}_i^{(t)}+(1-\alpha_i)\frac{1}{n}\sum_{j=1}^n\bm{w}_j^{(t)})$;\\ \qquad \qquad \ Global model: 
    $\hat{\bm{w}}  = \frac{1}{nS_T} \sum_{t=1}^T p_t   \sum_{j=1}^n \boldsymbol{{w}}_j^{(t)}$.
    }
\end{algorithm2e}
The following theorem obtains the convergence of global model in Algorithm~\ref{algorithm:LDAPFLo}.
\begin{theorem}[Global model convergence of \texttt{Local Descent APFL without Sampling}]
\label{Thm: Global Convergence o sampling}
If each client's objective function satisfies Assumptions~\ref{assumption: smoothness}-\ref{assumption: strong convexity}, using Algorithm \ref{algorithm:LDAPFLo}, choosing the mixing weight $\alpha_i \geq \max\{ 1-\frac{1}{4\sqrt{6}\kappa}, 1-\frac{1}{4\sqrt{6}\kappa \sqrt{\mu}}\} $, learning rate $\eta_t = \frac{16}{\mu(t+a)}$, where $a = \max\{128\kappa,\tau\}$, and using average scheme $\hat{\bm{w}}  = \frac{1}{nS_T} \sum_{t=1}^T p_t   \sum_{j=1}^n\boldsymbol{{w}}_j^{(t)}$, where $p_t = (t+a)^2$, $S_T = \sum_{t=1}^T p_t$,  then the following convergence holds:
{\small\begin{align}
      \mathbb{E}\left[F( \bm{\hat{w}})\right] 
    -   F( \bm{w}^{*})  \leq O\left(\frac{\mu\mathbb{E}\left[\|\bm{w}^{(1)} - \bm{w}^*\|^2\right]}{T^3}  \right) +O\left(\frac{\kappa^2 \tau\left(\sigma^2 + 2\tau \frac{\zeta}{n}  \right)}{\mu T^2}\right) +O\left(\frac{\kappa^2\tau\left(\sigma^2 + 2\tau \frac{\zeta}{n}  \right)\ln T}{\mu T^3}\right)  + O\left(\frac{\sigma^2}{nT}\right),\nonumber
\end{align}}
where $\bm{w}_* = \arg\min_{\bm{w}} F(\bm{w})$ is the optimal global solution.
\end{theorem}

\subsubsection{Proof of Useful Lemmas}
Before giving the proof of Theorem~\ref{Thm: Global Convergence o sampling} and \ref{localAPFLo}, we first prove few useful lemmas.Recall that we define virtual sequences $\{\bm{w}^{(t)}\}_{t=1}^T$,$\{\bar{\bm{v}}_i^{(t)}\}_{t=1}^T$,$\{\hat{\bm{v}}_i^{(t)}\}_{t=1}^T$ where $\bm{w}^{(t)} = \frac{1}{n} \sum_{i=1}^n \bm{w}_i^{(t)}$,$\bar{\bm{v}}_i^{(t)} = \alpha_i \bm{v}_i^{(t)} + (1-\alpha_i) \bm{w}_i^{(t)}$,$\hat{\bm{v}}_i^{(t)} = \alpha_i \bm{v}_i^{(t)} + (1-\alpha_i) \bm{w}^{(t)}$. 

We start  with the following lemma that bounds the difference between the gradients of local objective and global objective at local and global models.
\begin{lemma} \label{v-w}
For Algorithm~\ref{algorithm:LDAPFLo}, at each iteration, the gap between local gradient and global gradient is bounded by
{\small\begin{equation}
    \mathbb{E}\left[\|\nabla f_i(\hat{\bm{v}}_i^{(t)}) - \nabla F(\bm{w}^{(t)})\|^2\right] \leq 2L^2\mathbb{E}\left[\|\hat{\bm{v}}_i^{(t)} - \bm{v}^*\|^2\right] + 6\zeta_i + 6L^2\mathbb{E}\left[\|\bm{w}^{(t)} - \bm{w}^*\|^2\right] +  6L^2\Delta_i.\nonumber
\end{equation}}
\end{lemma}
\begin{proof}
From the smoothness assumption and by applying the Jensen's inequality we  have:
{\small\begin{align}
\mathbb{E}\left[\|\nabla f_i(\hat{\bm{v}}_i^{(t)}) - \nabla F(\bm{w}^{(t)})\|^2\right] &\leq 2\mathbb{E}\left[\|\nabla f_i(\hat{\bm{v}}_i^{(t)}) - \nabla f_i(\bm{v}_i^{*})\|^2\right] + 2\mathbb{E}\left[\|\nabla f_i(\bm{v}_i^{*}) - \nabla F(\bm{w}^{(t)})\|^2\right]\nonumber\\
   &\leq 2L^2\mathbb{E}\left[\|\hat{\bm{v}}_i^{(t)} - \bm{v}^*\|^2\right] + 6\mathbb{E}\left[\|\nabla f_i(\bm{v}_i^{*}) - \nabla f_i(\bm{w}^{*})\|^2\right] \nonumber \\
   &\quad + 6\mathbb{E}\left[\|\nabla f_i(\bm{w}^{*}) - \nabla F(\bm{w}^{*})\|^2\right] + 6\mathbb{E}\left[\|\nabla F(\bm{w}^{*}) - \nabla F(\bm{w}^{(t)})\|^2\right]  \nonumber\\
   & \leq 2L^2\mathbb{E}\left[\|\hat{\bm{v}}_i^{(t)} - \bm{v}^*\|^2\right] +  6L^2\mathbb{E}\left[\|\bm{v}^{*}_i - \bm{w}^*\|^2\right]   + 6\zeta_i  + 6L^2\mathbb{E}\left[\|\bm{w}^{(t)} - \bm{w}^*\|^2\right] \nonumber\\
  &\leq 2L^2\mathbb{E}\left[\|\hat{\bm{v}}_i^{(t)} - \bm{v}^*\|^2\right] +  6L^2\Delta_i + 6\zeta_i + 6L^2\mathbb{E}\left[\|\bm{w}^{(t)} - \bm{w}^*\|^2\right].\nonumber
\end{align}}
\end{proof}

\begin{lemma}[Local model deviation without sampling]\label{deviation}
For Algorithm\ref{algorithm:LDAPFLo}, at each iteration, the deviation between each local version of the global model $\bm{w}^{(t)}_i$ and the global model $\bm{w}^{(t)}$ is bounded by:
\begin{align}
     &\mathbb{E}\left[\|\bm{w}^{(t)} - \bm{w}_i^{(t)}\|^2\right] \leq 3\tau \sigma^2 \eta_{t-1}^2 + 3(\zeta_i + \frac{\zeta}{n}) \tau^2   \eta_{t-1}^2,\nonumber\\
    &\frac{1}{n}\sum_{i=1}^n \mathbb{E}\left[\|\bm{w}^{(t)} - \bm{w}_i^{(t)}\|^2\right] \leq 3\tau \sigma^2 \eta_{t-1}^2 + 6 \tau^2 \frac{\zeta}{n} \eta_{t-1}^2,\nonumber
\end{align}
where $\frac{\zeta}{n} = \frac{1}{n} \sum_{i=1}^n \zeta_i$.
\end{lemma}
\begin{proof}
According to Lemma 8 in~\cite{woodworth2020minibatch}:
\begin{align}
  \mathbb{E}\left[\|\bm{w}^{(t)} - \bm{w}_i^{(t)}\|^2\right] &\leq \frac{1}{n} \sum_{j=1}^n \mathbb{E}\left[\|\bm{w}_j^{(t)} - \bm{w}_i^{(t)}\|^2\right] \nonumber\\
    &\leq 3\left(\sigma^2 + \zeta_i\tau+ \frac{\zeta}{n}\tau \right)\sum_{p=t_c}^{t-1}\eta_p^2 \prod_{q=p+1}^{t-1}\left(1 - \mu \eta_q \right) \nonumber\\ 
    \frac{1}{n}\sum_{i=1}^n \mathbb{E}\left[\|\bm{w}^{(t)} - \bm{w}_i^{(t)}\|^2\right] &\leq \frac{1}{n^2}\sum_{i=1}^n \sum_{j=1}^n \mathbb{E}\left[\|\bm{w}_j^{(t)} - \bm{w}_i^{(t)}\|^2\right] \nonumber\\
    &\leq 3\left(\sigma^2 + 2\tau \frac{\zeta}{n}  \right)\sum_{p=t_c}^{t-1}\eta_p^2 \prod_{q=p+1}^{t-1}\left(1 - \mu \eta_q \right) \nonumber. 
\end{align}
Plugging in $\eta_q = \frac{16}{\mu(a+q)}$ yields:
\begin{align}
      \mathbb{E}\left[\|\bm{w}^{(t)} - \bm{w}_i^{(t)}\|^2\right] & \leq 3\left(\sigma^2 + \zeta_i\tau+ \frac{\zeta}{n}\tau \right)\sum_{p=t_c}^{t-1}\eta_p^2 \prod_{q=p+1}^{t-1}\frac{a+q-16}{a+q} \nonumber\\
    & \leq 3\left(\sigma^2 + \zeta_i\tau+ \frac{\zeta}{n}\tau \right)\sum_{p=t_c}^{t-1}\eta_p^2 \prod_{q=p+1}^{t-1}\frac{a+q-16}{a+q} \nonumber\\
    & \leq 3\left(\sigma^2 + \zeta_i\tau+ \frac{\zeta}{n}\tau \right)\sum_{p=t_c}^{t-1}\eta_p^2 \prod_{q=p+1}^{t-1}\frac{a+q-2}{a+q} \nonumber\\
    & \leq 3\left(\sigma^2 + \zeta_i\tau+ \frac{\zeta}{n}\tau \right)\sum_{p=t_c}^{t-1}\eta_p^2  \frac{(a+p-1)(a+p)}{(a+t-2)(a+t-1)} \nonumber\\
     & \leq 3\left(\sigma^2 + \zeta_i\tau+ \frac{\zeta}{n}\tau \right)\sum_{p=t_c}^{t-1}\eta_p^2  \frac{\eta_{t-1}^2}{\eta_p^2} \nonumber\\
     &\leq 3\tau\left(\sigma^2 + \zeta_i\tau+ \frac{\zeta}{n}\tau \right)\eta_{t-1}^2.\nonumber
\end{align}
Similarly,
\begin{align}
      \frac{1}{n}\sum_{i=1}^n \mathbb{E}\left[\|\bm{w}^{(t)} - \bm{w}_i^{(t)}\|^2\right] \leq 3\tau \sigma^2 \eta_{t-1}^2 + 6 \tau^2 \frac{\zeta}{n} \eta_{t-1}^2. \nonumber
\end{align}
\end{proof}

\begin{lemma}(Convergence of global model) \label{lm5}
Let $\bm{w}^{(t)} = \frac{1}{n} \sum_{i=1}^n \bm{w}_i^{(t)}$. Under the setting of Theorem \ref{localAPFLo}, we have:
\begin{align}
   &\mathbb{E}\left[\|\bm{w}^{(T+1)} - \bm{w}^*\|^2\right] \leq \frac{a^3}{ (T+a)^3}  \mathbb{E}\left[\|\bm{w}^{(1)} - \bm{w}^*\|^2\right] \nonumber\\
   &  +\left(T+16\left(\frac{1}{a+1}+\ln (T+a)\right)\right) \frac{1536a^2\tau\left(\sigma^2 + 2\tau \frac{\zeta}{n}  \right)L^2 }{(a-1)^2\mu^4(T+a)^3}+\frac{128\sigma^2 T (T+2a)}{n\mu^2(T+a)^3}.\nonumber
\end{align}
\end{lemma}
\begin{proof}
By the updating rule we have:
\begin{equation}
   \bm{w}^{(t+1)} - \bm{w}^* = \bm{w}^{(t)} - \bm{w}^* - \eta_t \frac{1}{n} \sum_{j=1}^n \nabla f_j( \bm{w}_j^{(t)}; \xi_j^t).\nonumber
\end{equation}
Then, taking square of norm and expectation on both sides, as well as applying strong convexity and smoothness assumptions yields:
{\small\begin{align}
   \mathbb{E}\left[\| \bm{w}^{(t+1)} - \bm{w}^*\|^2\right]  &\leq \mathbb{E}\left[\| \bm{w}^{(t)} - \bm{w}^*\|^2\right] - 2\eta_t \mathbb{E}\left[\left \langle \frac{1}{n} \sum_{j=1}^n \nabla f_j( \bm{w}_j^{(t)}), \bm{w}^{(t)} - \bm{w}^*  \right\rangle \right]+ \eta_t^2\frac{\sigma^2}{n} + \eta_t^2\mathbb{E}\left[\left\|\frac{1}{n} \sum_{j=1}^n \nabla f_j( \bm{w}_j^{(t)})\right\|^2\right] \nonumber\\
   &\leq \mathbb{E}\left[\| \bm{w}^{(t)} - \bm{w}^*\|^2\right] - 2\eta_t\mathbb{E}\left[ \left \langle \nabla F( \bm{w}^{(t)}), \bm{w}^{(t)} - \bm{w}^*  \right\rangle \right]+ \eta_t^2\frac{\sigma^2}{n} + \eta_t^2\underbrace{\mathbb{E}\left[\left\|\frac{1}{n} \sum_{j=1}^n \nabla f_j( \bm{w}_j^{(t)})\right\|^2\right]}_{T_1}\nonumber \\
   & \quad \underbrace{- 2\eta_t \mathbb{E}\left[\left \langle \frac{1}{n} \sum_{j=1}^n \nabla f_j( \bm{w}_j^{(t)})-\nabla F( \bm{w}^{(t)}), \bm{w}^{(t)} - \bm{w}^*  \right\rangle\right]}_{T_2}\nonumber \\
   &\leq (1-\mu\eta_t)\mathbb{E}\left[\| \bm{w}^{(t)} - \bm{w}^*\|^2\right]  - 2\eta_t ( \mathbb{E}[F( \bm{w}^{(t)})] - F( \bm{w}^{*})) + \eta_t^2\frac{\sigma^2}{n} + T_1 + T_2, \label{lm5 0}
\end{align}}
where at the last step we used the strongly convex property. 

Now we are going to bound $T_1$. By the Jensen's inequality and smoothness, we have:
\begin{align}
    T_1 &\leq 2\eta_t^2\mathbb{E}\left[\left\|\frac{1}{n} \sum_{j=1}^n \nabla f_j( \bm{w}_j^{(t)}) - \nabla F( \bm{w}^{(t)})\right\|^2\right] + 2\eta_t^2\mathbb{E}\left[\left\|\nabla F( \bm{w}^{(t)})\right\|^2\right]\nonumber \\
    &\leq 2\eta_t^2 L^2 \frac{1}{n} \sum_{j=1}^n\mathbb{E}\left[\|\bm{w}_j^{(t)} -\bm{w}^{(t)} \|^2\right] + 4\eta_t^2L\left(\mathbb{E}\left[F( \bm{w}^{(t)})\right] - F( \bm{w}^{*})\right) \label{eq:T1}
\end{align}
Then, we bound $T_2$ as:
\begin{align}
    T_2 &\leq   \eta_t\left(\frac{2}{\mu}\mathbb{E}\left[\left \| \frac{1}{n} \sum_{j=1}^n \nabla f_j( \bm{w}_j^{(t)})-\nabla F( \bm{w}^{(t)})\right\|^2\right] + \frac{\mu}{2}\mathbb{E}\left[\| \bm{w}^{(t)} - \bm{w}^* \|^2\right]\right) \nonumber \\
    &\leq \frac{2\eta_tL^2}{\mu} \frac{1}{n} \sum_{j=1}^n\mathbb{E}\left[\left \|\bm{w}_j^{(t)}- \bm{w}^{(t)}\right\|^2\right] + \frac{\mu\eta_t}{2}\mathbb{E}\left[\| \bm{w}^{(t)} - \bm{w}^* \|^2\right]. \label{eq:T2}
\end{align}
Now, by plugging back $T_1$ and $T_2$ from~(\ref{eq:T1}) and~(\ref{eq:T2}) in~(\ref{lm5 0}), we have:
\begin{align}
  &\mathbb{E}\left[ \|\bm{w}^{(t+1)} - \bm{w}^*\|^2 \right] \nonumber\\
  &\leq \left(1-\frac{\mu\eta_t}{2}\right)\mathbb{E}\left[\| \bm{w}^{(t)} - \bm{w}^*\|^2\right]  \underbrace{- (2\eta_t - 4\eta_t^2L)}_{\leq -\eta_t} \left( \mathbb{E}\left[F( \bm{w}^{(t)})\right] - F( \bm{w}^{*})\right) + \eta_t^2\frac{\sigma^2}{n} \nonumber \\
   & \quad +  \left(\frac{2\eta_t L^2}{\mu}  + 2\eta_t^2 L^2 \right)\frac{1}{n} \sum_{j=1}^n\mathbb{E}\left[  \left \|\bm{w}_j^{(t)}- \bm{w}^{(t)}\right\|^2\right] \label{eq: L5_1}\\
    &\leq \left(1-\frac{\mu\eta_t}{2}\right)\mathbb{E}\left[\| \bm{w}^{(t)} - \bm{w}^*\|^2\right]  + \eta_t^2\frac{\sigma^2}{n} +  \left(\frac{2\eta_t L^2}{\mu}  + 2\eta_t^2 L^2 \right)\frac{1}{n} \sum_{j=1}^n \mathbb{E}\left[ \left \|\bm{w}_j^{(t)}- \bm{w}^{(t)}\right\|^2\right].  \nonumber
\end{align}
Now, by using Lemma~\ref{deviation} we have:
\begin{equation}
  \mathbb{E}\left[  \|\bm{w}^{(t+1)} - \bm{w}^*\|^2\right] \leq \left(1-\frac{\mu\eta_t}{2}\right) \mathbb{E}\left[ \|\bm{w}^{(t)} - \bm{w}^*\|^2\right] + \left(\frac{2\eta_tL^2}{\mu}  + 2\eta_t^2 L^2 \right) 3\tau\left(\sigma^2 + 2\tau \frac{\zeta}{n}  \right)\eta_{t-1}^2 + \eta_t^2  \frac{\sigma^2}{n}.\nonumber
\end{equation}
Note that $(1-\frac{\mu\eta_t}{2})\frac{p_t}{\eta_t} = \frac{\mu (t+a)^2(t-8+a)}{16}\leq \frac{\mu (t-1+a)^3}{16}=\frac{p_{t-1}}{\eta_{t-1}}$, so we multiply $\frac{p_{t}}{\eta_{t}}$ on both sides and do the telescoping sum:
\begin{align}
  \frac{p_T}{\eta_T} \mathbb{E}\left[ \|\bm{w}^{(T+1)} - \bm{w}^*\|^2 \right]&\leq \frac{p_0}{\eta_0}  \mathbb{E}\left[\|\bm{w}^{(1)} - \bm{w}^*\|^2\right]+\sum_{t=1}^{T} \left(\frac{2 L^2}{\mu}  + 2\eta_t  L^2 \right) 3\tau\left(\sigma^2 + 2\tau \frac{\zeta}{n}  \right)p_t \eta_{t-1}^2  + \sum_{t=1}^{T} p_t\eta_t \frac{\sigma^2}{n} \\
  & \leq \frac{p_0}{\eta_0}  \mathbb{E}\left[\|\bm{w}^{(1)} - \bm{w}^*\|^2\right]+\sum_{t=1}^{T} \left(\frac{2 L^2}{\mu}  + 2\eta_t  L^2 \right) 3\tau\left(\sigma^2 + 2\tau \frac{\zeta}{n}  \right)\frac{256 a^2}{\mu^2(a-1)^2}  + \sum_{t=1}^{T} p_t\eta_t \frac{\sigma^2}{n}.\label{eq: L5_2}
\end{align}
Then, by re-arranging the terms will conclude the proof:
\begin{align}
   \mathbb{E}\left[\|\bm{w}^{(T+1)} - \bm{w}^*\|^2\right] &\leq \frac{a^3}{ (T+a)^3}  \mathbb{E}\left[\|\bm{w}^{(1)} - \bm{w}^*\|^2\right] \nonumber\\
   &\quad +\left(T+16\left(\frac{1}{a+1}+\ln (T+a)\right)\right) \frac{1536a^2\tau\left(\sigma^2 + 2\tau \frac{\zeta}{n}  \right)L^2 }{(a-1)^2\mu^4(T+a)^3}+\frac{128\sigma^2 T (T+2a)}{n\mu^2(T+a)^3}, \nonumber
\end{align}
where we use the inequality $\sum_{t=1}^T \frac{1}{t+a} \leq \frac{1}{a+1}+\int_{1}^T \frac{1}{t+a} < \frac{1}{a+1}+\ln (T+a)$.
\end{proof}

\subsubsection{Proof of Theorem~\ref{Thm: Global Convergence o sampling}} \label{proof thm4}
\begin{proof}
According to (\ref{eq: L5_1}) and (\ref{eq: L5_2}) in the proof of Lemma~\ref{lm5} we have:
{\begin{align}
    \frac{p_T}{\eta_T} \mathbb{E}\left[ \|\bm{w}^{(T+1)} - \bm{w}^*\|^2 \right] &\leq \frac{p_0}{\eta_0}  \mathbb{E}\left[\|\bm{w}^{(1)} - \bm{w}^*\|^2\right]-\sum_{t=1}^Tp_t \left( \mathbb{E}\left[F( \bm{w}^{(t)})\right] 
    -  F( \bm{w}^{*})\right) \nonumber \\
    &\quad+\sum_{t=1}^{T} \left(\frac{2 L^2}{\mu}  + 2\eta_t  L^2 \right) 3\tau\left(\sigma^2 + 2\tau \frac{\zeta}{n}  \right)\frac{256 a^2}{\mu^2(a-1)^2}   + \sum_{t=1}^{T} p_t\eta_t \frac{\sigma^2}{n},\nonumber
\end{align}}
re-arranging term and dividing both sides by $S_T = \sum_{t=1}^T p_t > T^3$ yields:
{\begin{align}
    \frac{1}{S_T}\sum_{t=1}^Tp_t&\left( \mathbb{E}\left[F( \bm{w}^{(t)})\right] 
    -   F( \bm{w}^{*})\right)  \leq \frac{p_0}{S_T\eta_0}  \mathbb{E}\left[\|\bm{w}^{(1)} - \bm{w}^*\|^2\right] \nonumber\\
    & \quad +\frac{1}{S_T}\sum_{t=1}^{T} \left(\frac{2 L^2}{\mu}  + 2\eta_t  L^2 \right) 3\tau\left(\sigma^2 + 2\tau \frac{\zeta}{n}  \right)\frac{256 a^2}{\mu^2(a-1)^2}   + \frac{1}{S_T}\sum_{t=1}^{T} p_t\eta_t \frac{\sigma^2}{n} \nonumber\\
    &\leq O\left(\frac{\mu\mathbb{E}\left[\|\bm{w}^{(1)} - \bm{w}^*\|^2\right]}{T^3}  \right) +O\left(\frac{\kappa^2 \tau\left(\sigma^2 + 2\tau \frac{\zeta}{n}  \right)}{\mu T^2}\right) +O\left(\frac{\kappa^2\tau\left(\sigma^2 + 2\tau \frac{\zeta}{n}  \right)\ln T}{\mu T^3}\right)  + O\left(\frac{\sigma^2}{nT}\right).\nonumber
\end{align}}
Recall that $\bm{\hat{w}} = \frac{1}{nS_T}\sum_{t=1}^T \sum_{j=1}^n\bm{w}_j^{(t)} $ and convexity of $F$, we can conclude that:
{\begin{align}
      \mathbb{E}\left[F( \bm{\hat{w}})\right] 
    -   F( \bm{w}^{*})  \leq O\left(\frac{\mu\mathbb{E}\left[\|\bm{w}^{(1)} - \bm{w}^*\|^2\right]}{T^3}  \right) +O\left(\frac{\kappa^2 \tau\left(\sigma^2 + 2\tau \frac{\zeta}{n}  \right)}{\mu T^2}\right) +O\left(\frac{\kappa^2\tau\left(\sigma^2 + 2\tau \frac{\zeta}{n}  \right)\ln T}{\mu T^3}\right)  + O\left(\frac{\sigma^2}{nT}\right).\nonumber
\end{align}}
\end{proof}

\subsubsection{Proof of Theorem~\ref{localAPFLo}} \label{proof thm5}
\begin{proof}
Recall that we defined virtual sequences $\{\bm{w}^{(t)}\}_{t=1}^T$ where $\bm{w}^{(t)} = \frac{1}{n} \sum_{i=1}^n \bm{w}_i^{(t)}$ and $\hat{\bm{v}}_i^{(t)} = \alpha_i \bm{v}_i^{(t)} + (1-\alpha_i) \bm{w}^{(t)}$, then by the updating rule we have:
{\small\begin{align}
    \mathbb{E}\left[\|\hat{\bm{v}} _i^{(t+1)} -\bm{v}_i^{*} \|^2\right] &=  \mathbb{E}\left[\left\| \hat{\bm{v}} _i^{(t)}- \alpha_i^2 \eta_t \nabla f_i(\bar{\bm{v}}_i^{(t)}) - (1-\alpha_i)\eta_t\frac{1}{n}\sum_{j=1}^n \nabla f_j(\bm{w}^{(t)}_j) - \bm{v}_i^{*} \right\|^2\right]\nonumber\\ 
    &\quad +  \mathbb{E}\left[\left\|\alpha_i^2 \eta_t (\nabla f_i(\bar{\bm{v}}_i^{(t)}) - \nabla f_i(\bar{\bm{v}}_i^{(t)};\xi_i^t)) + (1-\alpha_i)\eta_t\frac{1}{n}\sum_{j\in U_t}\left(\nabla f_j(\bm{w}^{(t)}_j)- \nabla f_j(\bm{w}^{(t)}_j;\xi^t_j)\right)\right\|^2\right] \nonumber \\
    & \leq  \mathbb{E}\left[\| \hat{\bm{v}} _i^{(t)}- \bm{v}_i^{*} \|^2\right] -2 \mathbb{E}\left[ \left\langle \alpha_i^2 \eta_t \nabla f_i(\bar{\bm{v}}_i^{(t)})+ (1-\alpha_i)\eta_t\frac{1}{n}\sum_{j=1}^n\nabla f_j(\bm{w}^{(t)}_j),\hat{\bm{v}} _i^{(t)}- \bm{v}_i^{*} \right\rangle \right] \nonumber\\ 
    & \quad  +\eta_t^2 \mathbb{E}\left[\left\| \alpha_i^2 \nabla f_i( \bar{\bm{v}} _i^{(t)})+ (1-\alpha_i)\frac{1}{n}\sum_{j=1}^n\nabla f_j(\bm{w}^{(t)}_j)\right\|^2\right]+  \alpha_i^2\eta_t^2 \sigma^2+ (1-\alpha_i)^2\eta_t^2 \frac{\sigma^2}{n} \nonumber \\
     & =  \mathbb{E}\left[\|\hat{\bm{v}} _i^{(t)}- \bm{v}_i^{*} \|^2\right]\underbrace{ -2(\alpha_i^2+1-\alpha_i)\eta_t \mathbb{E}\left[\left\langle  \nabla f_i(\bar{\bm{v}}_i^{(t)}),   \hat{\bm{v}}_i^{(t)}- \bm{v}_i^{*} \right\rangle\right]}_{T_1}  \nonumber\\
     & \quad  \underbrace{- 2\eta_t (1-\alpha_i) \mathbb{E}\left[ \left\langle  \frac{1}{n}\sum_{j=1}^n\nabla f_j(\bm{w}^{(t)}_j) - \nabla f_i(\bar{\bm{v}}_i^{(t)}),   \hat{\bm{v}}_i^{(t)}- \bm{v}_i^{*} \right\rangle\right]}_{T_2}  \nonumber\\
     & \quad  +\eta_t^2\underbrace{ \mathbb{E}\left[\| \alpha_i^2 \nabla f_i(\bar{\bm{v}}_i^{(t)})+ (1-\alpha_i)\frac{1}{n}\sum_{j=1}^n\nabla f_j(\bm{w}^{(t)}_j)\|^2\right]}_{T_3}+  \alpha_i^2\eta_t^2 \sigma^2+ (1-\alpha_i)^2\eta_t^2 \frac{\sigma^2}{n}\label{proof thm5 0}.
\end{align}}
Now, we bound the term $T_1$ as follows:
\begin{align}
    T_1 &= -2\eta_t(\alpha_i^2+1-\alpha_i)\mathbb{E}\left[\left\langle   \nabla f_i(\hat{\bm{v}}_i^{(t)}),  \hat{\bm{v}} _i^{(t)}- \bm{v}_i^{*} \right\rangle\right] \nonumber\\
    & \quad -2\eta_t(\alpha_i^2+1-\alpha_i) \mathbb{E}\left[\left\langle  \nabla f_i(\bar{\bm{v}}_i^{(t)}) - \nabla f_i(\hat{\bm{v}}_i^{(t)}),    \hat{\bm{v}} _i^{(t)} - \bm{v}_i^{*}\right \rangle\right] \nonumber\\
    &\leq -2\eta_t(\alpha_i^2+1-\alpha_i)\left( \mathbb{E}\left[f_i( \hat{\bm{v}} _i^{(t)})\right]- f_i(\boldsymbol{v} _i^{*}) + \frac{\mu}{2}  \mathbb{E}\left[\| \hat{\bm{v}} _i^{(t)} - \boldsymbol{v} _i^{*}\|^2\right]\right) \nonumber\\
    & \quad   + (\alpha_i^2+1-\alpha_i)\eta_t \left(\frac{8L^2}{\mu(1-8(\alpha_i - \alpha_i^2))} \mathbb{E}\left[\| \hat{\bm{v}} _i^{(t)}- \boldsymbol{\bar{v}} _i^{(t)} \|^2\right] + \frac{\mu(1-8(\alpha_i - \alpha_i^2))}{8} \mathbb{E}\left[\|  \hat{\bm{v}} _i^{(t)}- \bm{v}_i^{*}\|^2\right]\right)    \nonumber\\
    &\leq -2\eta_t(\alpha_i^2+1-\alpha_i)\left( \mathbb{E}\left[f_i( \hat{\bm{v}} _i^{(t)})\right]- f_i(\boldsymbol{v} _i^{*}) + \frac{\mu}{2}  \mathbb{E}\left[\| \hat{\bm{v}} _i^{(t)} - \boldsymbol{v} _i^{*}\|^2\right]\right) \nonumber\\
    & \quad    +  \eta_t \left(\frac{8L^2 (1-\alpha_i)^2}{\mu(1-8(\alpha_i - \alpha_i^2))} \mathbb{E}\left[\|\boldsymbol{w}^{(t)} - \boldsymbol{w} _i^{(t)} \|^2\right] + \frac{\mu(1-8(\alpha_i - \alpha_i^2))}{8} \mathbb{E}\left[\| \hat{\bm{v}} _i^{(t)}- \bm{v}_i^{*}\|^2\right]\right) \nonumber\\
    &\leq -2\eta_t(\alpha_i^2+1-\alpha_i)\left( \mathbb{E}\left[f_i( \hat{\bm{v}} _i^{(t)})\right]- f_i(\boldsymbol{v} _i^{*}) \right) - \frac{7\mu\eta_t}{8} \mathbb{E}\left[ \| \hat{\bm{v}} _i^{(t)} - \boldsymbol{v} _i^{*}\|^2\right] \nonumber\\
    & \quad  +  \frac{8\eta_tL^2 (1-\alpha_i)^2}{\mu(1-8(\alpha_i - \alpha_i^2))} \mathbb{E}\left[\|\boldsymbol{w}^{(t)} - \boldsymbol{w} _i^{(t)} \|^2\right],\label{Thm3 1}
\end{align}
where we use the fact $(\alpha_i^2+1-\alpha_i) \leq 1$. Note that, because we set $\alpha_i \geq \max\{ 1-\frac{1}{4\sqrt{6}\kappa}, 1-\frac{1}{4\sqrt{6}\kappa \sqrt{\mu}}\}$, and hence $1- 8(\alpha_i-\alpha_i^2) \geq 0$, so in the second inequality we can use the arithmetic-geometry inequality.

Next, we turn to bounding the term $T_2$ in~(\ref{proof thm5 0}):
{\begin{align}
    T_2 &= - 2\eta_t (1-\alpha_i)\mathbb{E}\left[ \left\langle  \frac{1}{n}\sum_{j=1}^n\nabla f_j(\bm{w}^{(t)}_j) - \nabla f_i(\bar{\bm{v}}_i^{(t)}),    \hat{\bm{v}} _i^{(t)}- \bm{v}_i^{*} \right\rangle\right] \nonumber\\
    & \leq  \eta_t (1-\alpha_i) \left( \frac{2(1-\alpha_i)}{\mu} \mathbb{E}\left[\left\| \nabla f_i(\bar{\bm{v}}_i^{(t)}) - \frac{1}{n}\sum_{j=1}^n\nabla f_j(\bm{w}^{(t)}_j)\right\|^2\right] +\frac{\mu}{2(1-\alpha_i)}\mathbb{E}\left[\| \hat{\bm{v}}^{(t)}_i- \bm{v}_i^{*}\|^2\right] \right) \nonumber\\
   & \leq   \frac{6(1-\alpha_i)^2\eta_t}{\mu}\times \nonumber\\
   &\left( \mathbb{E}\left[\left\| \nabla f_i(\bar{\bm{v}}_i^{(t)}) - \nabla f_i(\hat{\bm{v}}_i^{(t)})\right\|^2\right] +\mathbb{E}\left[\left\|\nabla f_i(\hat{\bm{v}}_i^{(t)}) - \nabla F(\bm{w}^{(t)})\right\|^2\right]+   \mathbb{E}\left[\left\| \nabla F(\bm{w}^{(t)}) -  \frac{1}{n}\sum_{j=1}^n\nabla f_j(\bm{w}^{(t)}_j)\right\|^2\right]\right)\nonumber\\
   & \quad + \frac{\eta_t\mu}{2}\mathbb{E}\left[\| \hat{\bm{v}}^{(t)}_i- \bm{v}_i^{*}\|^2 \right]\nonumber\\
    & \leq   \frac{6(1-\alpha_i)^2\eta_t}{\mu} \nonumber\left(L^2 \mathbb{E}\left[\left\| \boldsymbol{w}^{(t)} - \boldsymbol{w} _i^{(t)}\right\|^2\right] +\mathbb{E}\left[\left\|\nabla f_i(\hat{\bm{v}}_i^{(t)}) - \nabla F(\bm{w}^{(t)})\right\|^2\right]+  \frac{1}{n}\sum_{j=1}^n L^2\mathbb{E}\left[\left\| \boldsymbol{w}^{(t)} - \boldsymbol{w} _j^{(t)}\right\|^2\right]\right)\nonumber\\
   & \quad  + \frac{\eta_t\mu}{2}\mathbb{E}\left[\|\hat{\bm{v}}^{(t)}_i- \bm{v}_i^{*}\|^2\right].\label{Thm3 2}
\end{align}}
And finally, we bound the term $T_3$ in~(\ref{proof thm5 0}) as follows:
{\begin{align}
    T_3 &= \mathbb{E}\left[\left\| \alpha_i^2 \nabla f_i(\bar{\bm{v}}_i^{(t)})+ (1-\alpha_i)\frac{1}{n}\sum_{j=1}^n\nabla f_j(\bm{w}^{(t)}_j)\right\|^2\right] \nonumber\\
    &\leq 2(\alpha_i^2+1-\alpha_i)^2\mathbb{E}\left[\| \nabla f_i(\bar{\bm{v}}_i^{(t)})\|^2\right] +2\mathbb{E}\left[\left\|  (1-\alpha_i)\left(\frac{1}{n}\sum_{j=1}^n\nabla f_j(\bm{w}^{(t)}_j)-\nabla f_i(\bar{\bm{v}}_i^{(t)})\right)\right\|^2\right] \nonumber\\
    &\leq 2\left(2(\alpha_i^2+1-\alpha_i)^2\mathbb{E}\left[\| \nabla f_i(\hat{\bm{v}}_i^{(t)}) - \nabla f_i^*\|^2\right]+2(\alpha_i^2+1-\alpha_i)^2\mathbb{E}\left[\|\nabla f_i(\bar{\bm{v}}^{(t)}_i)-\nabla f_i(\hat{\bm{v}}_i^{(t)})\|^2\right]\right)  \nonumber\\
    & \quad+2(1-\alpha_i)^2\mathbb{E}\left[\left\| \frac{1}{n}\sum_{j=1}^n\nabla f_j(\bm{w}^{(t)}_j)-\nabla f_i(\bar{\bm{v}}_i^{(t)})\right\|^2\right] \nonumber\\
    &\leq 8L(\alpha_i^2+1-\alpha_i)\left (\mathbb{E}\left[f_i(\hat{\bm{v}}^{(t)}_i)\right] -  f_i^*\right)+4(1-\alpha_i)^2L^2\mathbb{E}\left[\|\boldsymbol{w}^{(t)} - \boldsymbol{w} _i^{(t)}\|^2\right] \nonumber\\
    & \quad +6(1-\alpha_i)^2\left(L^2\mathbb{E}\left[\left\| \boldsymbol{w}^{(t)} - \boldsymbol{w} _i^{(t)}\right\|^2\right] +\mathbb{E}\left[\left\|\nabla f_i(\hat{\bm{v}}_i^{(t)}) - \nabla F(\bm{w}^{(t)})\right\|^2\right]+  \frac{1}{n}\sum_{j=1}^n L^2\mathbb{E}\left[\left\| \boldsymbol{w}^{(t)} - \boldsymbol{w} _j^{(t)}\right\|^2\right]\right). 
    \label{Thm3 3}
\end{align}}
Now, using Lemma \ref{deviation}, $(1-\alpha_i)^2 \leq 1 $ and  plugging back $T_1$, $T_2$, and $T_3$ from~(\ref{Thm3 1}),~(\ref{Thm3 2}), and~(\ref{Thm3 3}) into~(\ref{proof thm5 0}), yields:
{\begin{align}
   &\mathbb{E}\left[\| \boldsymbol{\hat{v}} _i^{(t+1)}- \bm{v}_i^{*} \|^2\right]\nonumber\\
   &\leq \left(1-\frac{3\mu \eta_t}{8}\right)\mathbb{E}\left[\| \boldsymbol{\hat{v}} _i^{(t)}- \bm{v}_i^{*} \|^2\right] -2(\eta_t-4\eta_t^2L)(\alpha_i^2+1-\alpha_i)\left(\mathbb{E}\left[f_i(\boldsymbol{\hat{v}} _i^{(t)})\right]- f_i(\boldsymbol{v} _i^{*}) \right) \nonumber\\
   &  +  \alpha_i^2\eta_t^2 \sigma^2+ (1-\alpha_i)^2\eta_t^2 \frac{\sigma^2}{n} \nonumber\\
    &   + \left(\frac{8\eta_tL^2 (1-\alpha_i)^2}{\mu(1-8(\alpha_i - \alpha_i^2))} +\frac{6(1-\alpha_i)^2\eta_t L^2}{\mu} + 10(1-\alpha_i)^2\eta_t^2L^2 \right)\mathbb{E}\left[\left\| \boldsymbol{w}^{(t)} - \boldsymbol{w} _i^{(t)}\right\|^2\right]\nonumber\\ & + \left( \frac{6(1-\alpha_i)^2\eta_t L^2}{\mu} + 6(1-\alpha_i)^2\eta_t^2L^2 \right) \frac{1}{n}\sum_{j=1}^n  \mathbb{E}\left[\left\| \boldsymbol{w}^{(t)} - \boldsymbol{w} _j^{(t)}\right\|^2\right] \nonumber\\
     & +  \left( \frac{6\eta_t}{\mu}+ 6\eta_t^2 \right)    (1-\alpha_i)^2\mathbb{E}\left[\left\|\nabla F(\bm{w}^{(t)})-\nabla f_i(\hat{\bm{v}}_i^{(t)})\right\|^2\right] ,\nonumber\\
   &\leq \left(1-\frac{3\mu \eta_t}{8}\right)\mathbb{E}\left[\| \boldsymbol{\hat{v}} _i^{(t)}- \bm{v}_i^{*} \|^2\right] -2(\eta_t-4\eta_t^2L)(\alpha_i^2+1-\alpha_i)\left(\mathbb{E}\left[f_i(\boldsymbol{\hat{v}} _i^{(t)})\right]- f_i(\boldsymbol{v} _i^{*}) \right) \nonumber\\
   & +  \alpha_i^2\eta_t^2 \sigma^2+ (1-\alpha_i)^2\eta_t^2 \frac{\sigma^2}{n} \nonumber\\
    &   + \left(\frac{8\eta_tL^2 (1-\alpha_i)^2}{\mu(1-8(\alpha_i - \alpha_i^2))} +\frac{6(1-\alpha_i)^2\eta_t L^2}{\mu} + 10(1-\alpha_i)^2\eta_t^2L^2 \right)3\tau \left(\sigma^2   + (\zeta_i + \frac{\zeta}{n}) \tau  \right)\eta_{t-1}^2\nonumber\\ 
    & + \left( \frac{6(1-\alpha_i)^2\eta_t L^2}{\mu} + 6(1-\alpha_i)^2\eta_t^2L^2 \right)3\tau \left(\sigma^2   + 2 \frac{\zeta}{n}  \tau  \right)\eta_{t-1}^2 \nonumber\\
     &  +  \underbrace{\left( \frac{6\eta_t}{\mu}+ 6\eta_t^2 \right)    (1-\alpha_i)^2\mathbb{E}\left[\left\|\nabla F(\bm{w}^{(t)})-\nabla f_i(\hat{\bm{v}}_i^{(t)})\right\|^2\right]}_{T_4},\label{Thm3_41}
\end{align}}
where using Lemma~\ref{v-w} we can bound $T_4$ as:
\begin{align}
   T_4 &\leq \frac{6\eta_t}{\mu}   (1-\alpha_i)^2\left(2L^2\mathbb{E}\left[\|\hat{\bm{v}}_i^{(t)} - \bm{v}^*\|^2\right] + 6\zeta_i + 6L^2\mathbb{E}\left[\|\bm{w}^{(t)} - \bm{w}^*\|^2\right] +  6L^2\Delta_i\right) \nonumber\\
   & \quad + 6\eta_t^2(1-\alpha_i)^2   \left(2L^2\mathbb{E}\left[\|\hat{\bm{v}}_i^{(t)} - \bm{v}^*\|^2\right] + 6\zeta_i + 6L^2\mathbb{E}\left[\|\bm{w}^{(t)} - \bm{w}^*\|^2\right] +  6L^2\Delta_i\right).
\end{align}
Note that we choose $\alpha_i \geq \max\{ 1-\frac{1}{4\sqrt{6}\kappa}, 1-\frac{1}{4\sqrt{6}\kappa \sqrt{\mu}}\}$, hence $\frac{12L^2(1-\alpha_i)^2}{\mu} \leq \frac{\mu}{8}$ and  $12L^2(1-\alpha_i)^2 \leq \frac{\mu}{8}$, thereby we have:
\begin{align}
   T_4 \leq \frac{\mu\eta_t}{4}\|\boldsymbol{\hat{v}} _i^{(t)}- \bm{v}^*\|^2 
  +36\eta_t\left(\frac{1}{\mu}+\eta_t\right)   (1-\alpha_i)^2   \left( \zeta_i + L^2\mathbb{E}\left[\|\bm{w}^{(t)} - \bm{w}^*\|^2\right] +  L^2\Delta_i \right).   \nonumber
 \end{align}
Now, using Lemma~\ref{lm5} we have:
\begin{align} 
   T_4 &\leq \frac{\mu\eta_t}{4}\mathbb{E}\left[\|\boldsymbol{\hat{v}} _i^{(t)}- \bm{v}^*\|^2\right]  
   +36\eta_t\left(\frac{1}{\mu}+\eta_t\right)   (1-\alpha_i)^2 \nonumber\\
   &\left( \zeta_i + L^2\left( \frac{a^3}{ (t+a-1)^3}  \mathbb{E}\left[\|\bm{w}^{(1)} - \bm{w}^*\|^2\right] \right.\right.\nonumber\\
   & \left.\left. \quad +\left(t+16\left(\frac{1}{a+1}+\ln (t+a)\right)\right) \frac{1536\tau\left(\sigma^2 + 2\tau \frac{\zeta}{n}  \right)L^2 }{\mu^4(t+a-1)^3}+\frac{128\sigma^2 t (t+2a)}{n\mu^2(t+a-1)^3}\right) + L^2 \Delta_i\right).\label{Thm3 4} 
\end{align}
By plugging back $T_4$ from~(\ref{Thm3 4}) in~(\ref{Thm3_41}) and using the fact $-(\eta_t-4\eta_t^2L) \leq -\frac{1}{2}\eta_t$, and $(\alpha_i^2+1-\alpha_i)\geq \frac{3}{4}$, we have:
{\begin{equation}\label{eq:dri}
\begin{split}
   \mathbb{E}\left[\| \boldsymbol{\hat{v}} _i^{(t+1)}- \bm{v}_i^{*} \|^2\right] 
   &\leq (1-\frac{\mu \eta_t}{8})\mathbb{E}\left[\| \boldsymbol{\hat{v}} _i^{(t)}- \bm{v}_i^{*} \|^2\right] -\frac{3\eta_t}{4}\left(\mathbb{E}\left[f_i(\boldsymbol{\hat{v}} _i^{(t)})\right]- f_i(\boldsymbol{v} _i^{*}) \right) + \alpha_i^2\eta_t^2 \sigma^2+ (1-\alpha_i)^2\eta_t^2 \frac{\sigma^2}{n} \nonumber\\
    &  + \left(\frac{8\eta_tL^2 (1-\alpha_i)^2}{\mu(1-8(\alpha_i - \alpha_i^2))} +\frac{6(1-\alpha_i)^2\eta_t L^2}{\mu} + 10(1-\alpha_i)^2\eta_t^2L^2 \right)3\tau \left(\sigma^2   + (\zeta_i + \frac{\zeta}{n}) \tau  \right)\eta_{t-1}^2\nonumber\\ 
    & + \left( \frac{6(1-\alpha_i)^2\eta_t L^2}{\mu} + 6(1-\alpha_i)^2\eta_t^2L^2 \right)3\tau \left(\sigma^2   + 2 \frac{\zeta}{n}  \tau  \right)\eta_{t-1}^2  \nonumber\\
     & \quad +36\eta_t\left(\frac{1}{\mu}+\eta_t\right)   (1-\alpha_i)^2   \left( \zeta_i + L^2\left( \frac{a^3\mathbb{E}\left[\|\bm{w}^{(1)} - \bm{w}^*\|^2\right]}{ (t-1+a)^3} \right. \right. \nonumber\\ 
      & \quad\left. \left. +\left(t+16\left(\frac{1}{a+1}+\ln (t+a)\right)\right) \frac{1536\tau\left(\sigma^2 + 2\tau \frac{\zeta}{n}  \right)L^2 }{\mu^4(t+a-1)^3}+\frac{128\sigma^2 t (t+2a)}{n\mu^2(t-1+a)^3}+  \Delta_i\right) \right).\nonumber
     \end{split}
\end{equation}}
Note that $(1-\frac{\mu \eta_t}{8}) \frac{p_t}{\eta_t} \leq \frac{p_{t-1}}{\eta_{t-1}}$ where $p_t = (t+a)^2$, so, we multiply $\frac{p_t}{\eta_t}$ on both sides, and re-arrange the terms:
{\begin{align}
   &\frac{3p_t}{4}\left(\mathbb{E}\left[f_i(\boldsymbol{\hat{v}} _i^{(t)})\right]- f_i(\boldsymbol{v} _i^{*}) \right) \nonumber\\
    & \leq\frac{p_{t-1}}{\eta_{t-1}}\mathbb{E}\left[\| \boldsymbol{\hat{v}} _i^{(t)}- \bm{v}_i^{*} \|^2\right] -\frac{p_t}{\eta_t}\mathbb{E}\left[\| \boldsymbol{\hat{v}} _i^{(t+1)}- \bm{v}_i^{*} \|^2\right]+ p_t\eta_t\left(\alpha_i^2 \sigma^2+(1-\alpha_i)^2 \frac{\sigma^2}{n}\right)\nonumber\\ 
     & \quad + \left(\frac{8L^2 (1-\alpha_i)^2}{\mu(1-8(\alpha_i - \alpha_i^2))} +\frac{6(1-\alpha_i)^2 L^2}{\mu} + 10(1-\alpha_i)^2\eta_tL^2 \right)3\tau \left(\sigma^2   + (\zeta_i + \frac{\zeta}{n}) \tau  \right)p_t\eta_{t-1}^2\nonumber\\ 
    & \quad+ \left( \frac{6(1-\alpha_i)^2 L^2}{\mu} + 6(1-\alpha_i)^2\eta_t L^2 \right)3\tau \left(\sigma^2   + 2 \frac{\zeta}{n}  \tau  \right)p_t\eta_{t-1}^2\nonumber\\
     &  \quad+36p_t\left(\frac{1}{\mu}+\eta_t\right)   (1-\alpha_i)^2    L^2\left(\frac{a^3\mathbb{E}\left[\|\bm{w}^{(1)} - \bm{w}^*\|^2\right]}{ (t-1+a)^3} +\left(t+16\Theta(\ln t)\right) \frac{1536\tau\left(\sigma^2 + 2\tau \frac{\zeta}{n}  \right)L^2 }{\mu^4(t+a-1)^3}+\frac{128\sigma^2 t (t+2a)}{n\mu^2(t-1+a)^3} \right) \nonumber\\
      & \quad +36p_t\left(\frac{1}{\mu}+\eta_t\right)   (1-\alpha_i)^2   \left( \zeta_i +  L^2\Delta_i \right)\nonumber
\end{align}}
By applying the telescoping sum and dividing both sides by $S_T = \sum_{t=1}^T p_t \geq T^3$ we have:
{\small\begin{align}
   & f_i(\boldsymbol{\hat{v}} _i)- f_i(\boldsymbol{v} _i^{*}) \nonumber\\
   &\leq \frac{1}{S_T} \sum_{t=1}^Tp_t(f_i(\boldsymbol{\hat{v}} _i^{(t)})- f_i(\boldsymbol{v} _i^{*}) )     \nonumber  \\
   &\leq \frac{4p_{0}\mathbb{E}\left[\| \boldsymbol{\hat{v}} _i^{(1)}- \bm{v}_i^{*} \|^2\right] }{3\eta_{0}S_T}+ \frac{1}{S_T}\frac{4}{3} \sum_{t=1}^T p_t\eta_t\left(\alpha_i^2 \sigma^2+(1-\alpha_i)^2 \frac{\sigma^2}{n}\right)  \nonumber\\
   & \quad + \frac{1}{S_T}\frac{4}{3} \sum_{t=1}^T\left(\frac{8L^2 (1-\alpha_i)^2}{\mu(1-8(\alpha_i - \alpha_i^2))} +\frac{6(1-\alpha_i)^2 L^2}{\mu} + 10(1-\alpha_i)^2\eta_tL^2 \right)3\tau \left(\sigma^2   + (\zeta_i + \frac{\zeta}{n}) \tau  \right)p_t\eta_{t-1}^2\nonumber\\ 
    & \quad + \frac{1}{S_T}\frac{4}{3} \sum_{t=1}^T\left( \frac{6(1-\alpha_i)^2 L^2}{\mu} + 6(1-\alpha_i)^2\eta_t L^2 \right)3\tau \left(\sigma^2   + 2 \frac{\zeta}{n}  \tau  \right)p_t\eta_{t-1}^2\nonumber\\
     &\quad + 48(1-\alpha_i)^2   \frac{ L^2}{S_T} \sum_{t=1}^Tp_t\left(\frac{1}{\mu}+\eta_t\right)  \left(\frac{a^3\mathbb{E}\left[\|\bm{w}^{(1)} - \bm{w}^*\|^2\right]}{ (t-1+a)^3} +\left(t+16\Theta(\ln t)\right) \frac{1536\tau\left(\sigma^2 + 2\tau \frac{\zeta}{n}  \right)L^2 }{\mu^4(t+a-1)^3}+\frac{128\sigma^2 t (t+2a)}{n\mu^2(t-1+a)^3} \right) \nonumber\\
      & \quad +48(1-\alpha_i)^2   \left( \zeta_i +  L^2\Delta_i \right)\frac{1}{S_T} \sum_{t=1}^Tp_t\left(\frac{1}{\mu}+\eta_t\right)   \nonumber\\
      &\leq \frac{4p_{0}\mathbb{E}\left[\| \boldsymbol{\hat{v}} _i^{(1)}- \bm{v}_i^{*} \|^2\right] }{3\eta_{0}S_T} +\frac{32 T(T+a)}{3\mu S_T}\left(\alpha_i^2 \sigma^2+(1-\alpha_i)^2 \frac{\sigma^2}{n}\right)  \nonumber\\
      & \quad +  \frac{4}{3}  \left(\frac{8L^2 (1-\alpha_i)^2T}{\mu(1-8(\alpha_i - \alpha_i^2))S_T} +\frac{6(1-\alpha_i)^2 L^2T}{\mu S_T} + \frac{10(1-\alpha_i)^2 L^2 \Theta(\ln T) }{\mu S_T}\right)3\tau \left(\sigma^2   + (\zeta_i + \frac{\zeta}{n}) \tau  \right)\frac{256a^2}{\mu^2 (a-1)^2}\nonumber\\ 
     & \quad+ \frac{4}{3}  \left( \frac{6(1-\alpha_i)^2 L^2 T}{\mu S_T} + \frac{6(1-\alpha_i)^2 L^2 \Theta(\ln T)}{\mu S_T} \right)3\tau \left(\sigma^2   + 2 \frac{\zeta}{n}  \tau  \right)\frac{256a^2}{\mu^2 (a-1)^2}\nonumber\\
     &\quad + 48(1-\alpha_i)^2    L^2\frac{a^2}{(a-1)^2S_T} \left(\frac{a^3\Theta(\ln T) \mathbb{E}\left[\|\bm{w}^{(1)} - \bm{w}^*\|^2\right]}{\mu}  + \left(\frac{T}{a}+\Theta(\ln T)\right) \frac{1536L^2\tau\left(\sigma^2 + 2\tau \frac{\zeta}{n}  \right)}{\mu^5}+\frac{64(2a+1)\sigma^2 T(T+a) }{na\mu^3} \right) \nonumber\\
     &\quad + 48(1-\alpha_i)^2    L^2\frac{a^2}{(a-1)^2S_T}  \left(\frac{16a^3\pi^2 \mathbb{E}\left[\|\bm{w}^{(1)} - \bm{w}^*\|^2\right]}{6\mu}  + \left( \Theta(\ln T) +O(1)\right) \frac{1536L^2\tau\left(\sigma^2 + 2\tau \frac{\zeta}{n}  \right)}{\mu^5}+\frac{2048(2a+1)\sigma^2  }{na\mu^3}T \right) \nonumber\\
      & \quad +48(1-\alpha_i)^2   \left( \zeta_i +  L^2\Delta_i \right)\frac{1}{S_T} \left(\frac{S_T}{\mu}+\frac{8T(T+2a)}{\mu}\right)   \nonumber\\
      & = O\left(\frac{\mu  }{T^3}\right) +\alpha_i^2O\left(\frac{\sigma^2}{\mu T}\right)+(1-\alpha_i)^2O\left(\frac{\zeta_i}{\mu} +  \kappa L\Delta_i\right)\nonumber\\
      & \quad+ (1-\alpha_i)^2\left(O\left(\frac{\kappa L \ln T}{T^3}\right) +O\left(\frac{\kappa^2\sigma^2}{\mu nT}\right)+ O\left(\frac{\kappa^2 \tau^2(\zeta_i + \frac{\zeta}{n})+\kappa^2 \tau\sigma^2 }{\mu T^2} \right)+ O\left(\frac{\kappa^4 \tau\left(\sigma^2 + 2\tau \frac{\zeta}{n}  \right) }{\mu T^2}\right)  \right). \nonumber  
\end{align}}
where we use the convergence of $\sum_{t=1}^{\infty} \frac{\ln t}{t^2} \xrightarrow{} O(1)$, and $\sum_{t=1}^{\infty} \frac{1}{t^2} \xrightarrow{} \frac{\pi^2}{6}$.

\end{proof}

\subsection{Proof with Sampling}
In this section we will provide the formal proof of the Theorem~\ref{localpartial}.  The proof pipeline is similar to what we did in Appendix \ref{proof thm5}. The only difference is that we use sampling method here, hence, we will introduce the variance depending on sampling size $K$.
Now we first begin with the proof of some technique lemmas.

\subsubsection{Proof of Useful Lemmas}
\begin{lemma} \label{lm: local-global gradient gap with sampling}
For Algorithm~\ref{algorithm:LDAPFL}, at each iteration, the gap between local gradient and global gradient is bounded by
{\small\begin{equation}
    \mathbb{E}\left[\|\nabla f_i(\hat{\bm{v}}_i^{(t)}) - \frac{1}{K} \sum_{j\in U_t}\nabla f_j(\bm{w}^{(t)})\|^2\right] \leq 2L^2\mathbb{E}\left[\|\hat{\bm{v}}_i^{(t)} - \bm{v}^*\|^2\right] + 6\left(2\zeta_i + 2\frac{\zeta}{K}  \right)  + 6L^2\mathbb{E}\left[\|\bm{w}^{(t)} - \bm{w}^*\|^2\right] +  6L^2\Delta_i.\nonumber
\end{equation}}
\end{lemma}
\begin{proof}
From the smoothness assumption and by applying the Jensen's inequality we  have:
{\begin{align}
&\mathbb{E}\left[\|\nabla f_i(\hat{\bm{v}}_i^{(t)}) - \frac{1}{K} \sum_{j\in U_t}\nabla f_j(\bm{w}^{(t)})\|^2\right] \nonumber \\
&\leq 2\mathbb{E}\left[\|\nabla f_i(\hat{\bm{v}}_i^{(t)}) - \nabla f_i(\bm{v}_i^{*})\|^2\right] + 2\mathbb{E}\left[\|\nabla f_i(\bm{v}_i^{*}) - \frac{1}{K} \sum_{j\in U_t}\nabla f_j(\bm{w}^{(t)})\|^2\right]\nonumber\\
   &\leq 2L^2\mathbb{E}\left[\|\hat{\bm{v}}_i^{(t)} - \bm{v}^*\|^2\right] + 6\mathbb{E}\left[\|\nabla f_i(\bm{v}_i^{*}) - \nabla f_i(\bm{w}^{*})\|^2\right] \nonumber \\
   &+ 6\mathbb{E}\left[\|\nabla f_i(\bm{w}^{*}) - \frac{1}{K} \sum_{j\in U_t}\nabla f_j (\bm{w}^{*})\|^2\right] + 6\mathbb{E}\left[\|\nabla \frac{1}{K} \sum_{j\in U_t}\nabla f_j (\bm{w}^{*}) - \frac{1}{K} \sum_{j\in U_t}\nabla f_j (\bm{w}^{(t)})\|^2\right]  \nonumber\\
   & \leq 2L^2\mathbb{E}\left[\|\hat{\bm{v}}_i^{(t)} - \bm{v}^*\|^2\right] +  6L^2\mathbb{E}\left[\|\bm{v}^{*}_i - \bm{w}^*\|^2\right]   + 6\left(2\zeta_i + 2\frac{1}{K}\sum_{j\in U_t} \zeta_j \right)  + 6L^2\mathbb{E}\left[\|\bm{w}^{(t)} - \bm{w}^*\|^2\right] \nonumber\\
  &\leq 2L^2\mathbb{E}\left[\|\hat{\bm{v}}_i^{(t)} - \bm{v}^*\|^2\right] +  6L^2\Delta_i +  6\left(2\zeta_i + 2\frac{\zeta}{K}  \right) + 6L^2\mathbb{E}\left[\|\bm{w}^{(t)} - \bm{w}^*\|^2\right].\nonumber
\end{align}}
\end{proof}

\begin{lemma}[Local model deviation with sampling]\label{lm: deviation with sampling}
For Algorithm~\ref{algorithm:LDAPFL}, at each iteration, the deviation between each local version of the global model $\bm{w}^{(t)}_i$ and the global model $\bm{w}^{(t)}$ is bounded by:
\begin{align}
     &\mathbb{E}\left[\|\bm{w}^{(t)} - \bm{w}_i^{(t)}\|^2\right] \leq 3\tau \sigma^2 \eta_{t-1}^2 + 3(\zeta_i + \frac{\zeta}{K}) \tau^2   \eta_{t-1}^2,\nonumber\\
    &\frac{1}{K}\sum_{i\in U_t} \mathbb{E}\left[\|\bm{w}^{(t)} - \bm{w}_i^{(t)}\|^2\right] \leq 3\tau \sigma^2 \eta_{t-1}^2 + 6 \tau^2 \frac{\zeta}{K} \eta_{t-1}^2.\nonumber
\end{align}
where $\frac{\zeta}{K} = \frac{1}{K} \sum_{i=1}^n \zeta_i$.
\end{lemma}
\begin{proof}
According to Lemma 8 in~\cite{woodworth2020minibatch}:
\begin{align}
  \mathbb{E}\left[\|\bm{w}^{(t)} - \bm{w}_i^{(t)}\|^2\right] &\leq \frac{1}{K} \sum_{j\in U_t} \mathbb{E}\left[\|\bm{w}_j^{(t)} - \bm{w}_i^{(t)}\|^2\right] \nonumber\\
    &\leq 3\left(\sigma^2 + \zeta_i\tau+ \frac{\zeta}{K}\tau \right)\sum_{p=t_c}^{t-1}\eta_p^2 \prod_{q=p+1}^{t-1}\left(1 - \mu \eta_q \right) \nonumber\\ 
    \frac{1}{n}\sum_{i=1}^n \mathbb{E}\left[\|\bm{w}^{(t)} - \bm{w}_i^{(t)}\|^2\right] &\leq \frac{1}{n^2}\sum_{i=1}^n \sum_{j=1}^n \mathbb{E}\left[\|\bm{w}_j^{(t)} - \bm{w}_i^{(t)}\|^2\right] \nonumber\\
    &\leq 3\left(\sigma^2 + 2\tau \frac{\zeta}{K}  \right)\sum_{p=t_c}^{t-1}\eta_p^2 \prod_{q=p+1}^{t-1}\left(1 - \mu \eta_q \right) \nonumber. 
\end{align}
Then the rest of the proof follows Lemma~\ref{deviation}.
\end{proof}

 \begin{lemma} (Convergence of Global Model) \label{lm6}
Let $\bm{w}^{(t)} = \frac{1}{K} \sum_{j\in U_t}  \bm{w}_j^{(t)}$. Assume each client's objective function satisfies Assumption \ref{assumption: strong convexity}-\ref{assumption: bounded grad}. then, using Algorithm~\ref{algorithm:LDAPFL} by choosing learning rate as $\eta_t = \frac{16}{\mu(t+a)}$ and letting $\kappa = L/\mu$,  we have:
{\begin{align}
   \mathbb{E}\left[\|\bm{w}^{(T+1)} - \bm{w}^*\|^2\right] &\leq \frac{a^3}{ (T+a)^3}  \mathbb{E}\left[\|\bm{w}^{(1)} - \bm{w}^*\|^2\right] \nonumber\\
   & \quad  +\left(T+16\left(\frac{1}{a+1}+\ln (T+a)\right)\right) \frac{1536a^2\tau\left(\sigma^2 + 2\tau \frac{\zeta}{K}  \right)L^2 }{(a-1)^2\mu^4(T+a)^3}+\frac{128\sigma^2 T (T+2a)}{K\mu^2(T+a)^3}.\nonumber
\end{align}}
\end{lemma}

\begin{proof}
First, we note that from the updating rule we have
\begin{align}
   \bm{w}^{(t+1)} - \bm{w}^* = \bm{w}^{(t)} - \bm{w}^* - \eta_t \frac{1}{K} \sum_{j\in U_t} \nabla f_j( \bm{w}_j^{(t)}; \xi_j^t).
\end{align}
Now, making both sides squared and according to the strong convexity we have:
\begin{align}
   &\mathbb{E}\left[\| \bm{w}^{(t+1)} - \bm{w}^*\|^2 \right]\nonumber\\
   &\leq \mathbb{E}\left[\| \bm{w}^{(t)} - \bm{w}^*\|^2 \right]- 2\eta_t \mathbb{E}\left[\left \langle \frac{1}{K}\sum_{j\in U_t}\nabla f_j( \bm{w}_j^{(t)}), \bm{w}^{(t)} - \bm{w}^*  \right\rangle\right]   + \eta_t^2\mathbb{E}\left[\left\|\frac{1}{K}\sum_{j\in U_t} \nabla f_j( \bm{w}_j^{(t)})\right\|^2 \right] + \eta_t^2 \frac{\sigma^2}{K}\nonumber\\ 
   &  \leq (1-\mu\eta_t)\mathbb{E}\left[\| \bm{w}^{(t)} - \bm{w}^*\|^2\right] -  (2\eta_t-2L\eta_t^2) \mathbb{E}\left[ F( \bm{w}^{(t)}) - F( \bm{w}^{*})\right] + \eta_t^2\frac{ \sigma^2}{K}\nonumber \\ 
   &\quad + \eta_t^2\frac{1}{K}\sum_{j\in U_t} L^2\mathbb{E}\left[\left\| \bm{w}_j^{(t)} -\bm{w}^{(t)} \right\|^2 \right] - 2\eta_t \mathbb{E}\left[\left \langle \frac{1}{K}\sum_{j\in U_t}\nabla f_j( \bm{w}_j^{(t)})-\nabla f_j( \bm{w}^{(t)}), \bm{w}^{(t)} - \bm{w}^*  \right\rangle\right] \nonumber\\
    &  \leq (1-\mu\eta_t)\mathbb{E}\left[\| \bm{w}^{(t)} - \bm{w}^*\|^2\right] \underbrace{-(2\eta_t-4L\eta_t^2)}_{\leq -\eta_t} \mathbb{E}\left[ F( \bm{w}^{(t)}) - F( \bm{w}^{*})\right] + \eta_t^2\frac{ \sigma^2}{K}\nonumber \\ 
   &\quad + 2\eta_t^2  L^2\frac{1}{K}\sum_{j\in U_t}\mathbb{E}\left[\left\| \bm{w}_j^{(t)} -\bm{w}^{(t)} \right\|^2 \right]  +\frac{2\eta_tL^2}{\mu} \frac{1}{K}\sum_{j\in U_t}\mathbb{E}\left[\left \|\bm{w}_j^{(t)}- \bm{w}^{(t)}\right\|^2\right] + \frac{\mu\eta_t}{2}\mathbb{E}\left[\| \bm{w}^{(t)} - \bm{w}^* \|^2\right].\label{lm6 0}
\end{align}
Then, merging the term, multiplying both sides with $\frac{p_t}{\eta_t}$, and do the telescoping sum yields: 
\begin{align}
  \frac{p_T}{\eta_T} \mathbb{E}\left[ \|\bm{w}^{(T+1)} - \bm{w}^*\|^2 \right] & \leq \frac{p_0}{\eta_0}  \mathbb{E}\left[\|\bm{w}^{(1)} - \bm{w}^*\|^2\right] - \mathbb{E}[ F( \bm{w}^{(t)}) - F( \bm{w}^{*})] \nonumber\\
   & \quad +\sum_{t=1}^{T}\left(\frac{2 L^2 }{\mu}  +2 \eta_t  L^2 \right)p_t\frac{1}{K}\sum_{j\in U_t}\mathbb{E}\left[\left\| \bm{w}_j^{(t)} -\bm{w}^{(t)} \right\|^2 \right]  + \sum_{t=1}^{T} p_t\eta_t \frac{\sigma^2}{K}.\label{eq: L6_1}
\end{align}

Plugging Lemma~\ref{lm: deviation with sampling} into (\ref{eq: L6_1}) yields:
\begin{align}
  \frac{p_T}{\eta_T} \mathbb{E}\left[ \|\bm{w}^{(T+1)} - \bm{w}^*\|^2 \right] & \leq \frac{p_0}{\eta_0}  \mathbb{E}\left[\|\bm{w}^{(1)} - \bm{w}^*\|^2\right] - \mathbb{E}[ F( \bm{w}^{(t)}) - F( \bm{w}^{*})] \nonumber\\
   & \quad +\sum_{t=1}^{T}\left(\frac{2L^2 }{\mu}  +2\eta_t  L^2 \right)3p_t\eta_{t-1}^2\tau\left( \sigma^2  + 2 \tau\frac{\zeta}{K} \right) + \sum_{t=1}^{T} p_t\eta_t \frac{\sigma^2}{K}.\label{eq: L6_2}
\end{align}

Then, by re-arranging the terms will conclude the proof as 
\begin{equation*}
\begin{aligned}
   \mathbb{E}\left[\|\bm{w}^{(T+1)} - \bm{w}^*\|^2\right] & \leq \frac{a^3}{ (T+a)^3}  \mathbb{E}\left[\|\bm{w}^{(1)} - \bm{w}^*\|^2\right] \nonumber\\ & \quad +\left(T+16\left(\frac{1}{a+1}+\ln (T+a)\right)\right) \frac{1536 a^2 L^2\tau\left( \sigma^2  + 2 \tau\frac{\zeta}{K} \right)}{(a-1)^2\mu^4(T+a)^3}+\frac{128\sigma^2 T (T+2a)}{K\mu^2(T+a)^3}.\nonumber
\end{aligned}
\end{equation*}

 \end{proof}

\subsubsection{Proof of Theorem~\ref{Thm: Global Convergence w sampling}} \label{proof thm2}

\begin{proof}
According to (\ref{eq: L6_2}) we have:
\begin{align}
  \frac{p_T}{\eta_T} \mathbb{E}\left[ \|\bm{w}^{(T+1)} - \bm{w}^*\|^2 \right] & \leq \frac{p_0}{\eta_0}  \mathbb{E}\left[\|\bm{w}^{(1)} - \bm{w}^*\|^2\right] - \mathbb{E}[ F( \bm{w}^{(t)}) - F( \bm{w}^{*})] \nonumber\\
   & \quad +\sum_{t=1}^{T}\left(\frac{2 L^2 }{\mu}  +2\eta_t  L^2 \right)3p_t\eta_{t-1}^2\tau\left( \sigma^2  + 2 \tau\frac{\zeta}{K} \right) + \sum_{t=1}^{T} p_t\eta_t \frac{\sigma^2}{K}.\label{eq: L6_2}
\end{align}
By re-arranging the terms and dividing both sides by $S_T = \sum_{t=1}^T p_t > T^3$ yields:
{\small\begin{align}
    &\frac{1}{S_T}\sum_{t=1}^Tp_t\left( \mathbb{E}\left[F( \bm{w}^{(t)})\right] 
    -   F( \bm{w}^{*})\right) \nonumber \\
    &\qquad\leq \frac{p_0}{S_T\eta_0}  \mathbb{E}\left[\|\bm{w}^{(1)} - \bm{w}^*\|^2\right]  +\frac{1}{S_T}\sum_{t=1}^{T}\left(\frac{2  L^2 }{\mu}  +2\eta_t  L^2 \right)3p_t\eta_{t-1}^2\tau\left( \sigma^2  + 2 \tau\frac{\zeta}{K} \right)+ \frac{1}{S_T}\sum_{t=1}^{T} p_t\eta_t \frac{ \sigma^2}{K}\nonumber\\
    & \qquad\leq O\left(\frac{\mu\mathbb{E}\left[\|\bm{w}^{(1)} - \bm{w}^*\|^2\right]}{T^3}  \right) +O\left(\frac{\kappa^2 \tau\left(\sigma^2 + 2\tau \frac{\zeta}{K}  \right)}{\mu T^2}\right) +O\left(\frac{\kappa^2\tau\left(\sigma^2 + 2\tau \frac{\zeta}{K}  \right)\ln T}{\mu T^3}\right)  + O\left(\frac{\sigma^2}{KT}\right).\nonumber
\end{align}}
Recalling that $\bm{\hat{w}} = \frac{1}{nS_T}\sum_{t=1}^T \sum_{j=1}^n\bm{w}_j^{(t)} $, from the  convexity of $F(\cdot)$, we can conclude that
{\small\begin{align}
      \mathbb{E}\left[F( \bm{\hat{w}})\right] 
    -   F( \bm{w}^{*})  \leq O\left(\frac{\mu\mathbb{E}\left[\|\bm{w}^{(1)} - \bm{w}^*\|^2\right]}{T^3}  \right) +O\left(\frac{\kappa^2 \tau\left(\sigma^2 + 2\tau \frac{\zeta}{K}  \right)}{\mu T^2}\right) +O\left(\frac{\kappa^2\tau\left(\sigma^2 + 2\tau \frac{\zeta}{K}  \right)\ln T}{\mu T^3}\right)  + O\left(\frac{\sigma^2}{KT}\right).\nonumber
\end{align}}
\end{proof}

\subsubsection{Proof of Theorem~\ref{localpartial}} \label{proof thm3}
In this section we provide the proof of Theorem~\ref{localpartial}. The main difference in this case is that only a subset of local models get updated each period due to partial participation of devices, i.e., $K$ out of all $n$ devices that are sampled uniformly at random. To generalize the proof,  we will use an indicator function to model this stochastic update, and show that while the stochastic gradient is unbiased, but it introduces some extra variance that can be taken into account by properly tuning the hyper-parameters.
\begin{proof}
Recall that we defined virtual sequences of $\{\bm{w}^{(t)}\}_{t=1}^T$ where $\bm{w}^{(t)} = \frac{1}{K} \sum_{j\in U_t} \bm{w}_i^{(t)}$ and $\hat{\bm{v}}_i^{(t)} = \alpha_i \bm{v}_i^{(t)} + (1-\alpha_i) \bm{w}^{(t)}$. We also define an indicator variable to denote whether $i$th client was selected at iteration $t$:
\begin{align}
   \mathbb{I}_i^t = \left\{\begin{array}{l}1  \quad if  \ i  \in U_t \\0  \quad else \end{array}\right.\nonumber
\end{align}
obviously, $\mathbb{E}\left[\mathbb{I}_i^t\right] =\frac{K}{n}$.
 We start by writing the updating rule:
\begin{align}
  \hat{\bm{v}} _i^{(t+1)}  = \hat{\bm{v}} _i^{(t)}- \alpha_i^2  \mathbb{I}_i^t \eta_t \nabla f_i(\bar{\bm{v}}_i^{(t)};\xi_i^t) - (1-\alpha_i)\eta_t\frac{1}{K}\sum_{j\in U_t} \nabla f_j(\bm{w}^{(t)}_j;\xi_i^t). \nonumber
\end{align}
Now, subtracting $\bm{v}_i^*$ on both sides, taking the square of norm and expectation, yields:
{\begin{align}
     &\mathbb{E}\left[\|\hat{\bm{v}} _i^{(t+1)} -\bm{v}_i^{*} \|^2\right] \nonumber\\
     &= \mathbb{E}\left[\left\| \hat{\bm{v}} _i^{(t)}- \alpha_i^2\mathbb{I}_i^t \eta_t \nabla f_i(\bar{\bm{v}}_i^{(t)}) - (1-\alpha_i)\eta_t\frac{1}{K}\sum_{j \in U_t} \nabla f_j(\bm{w}^{(t)}_j) - \bm{v}_i^{*} \right\|^2 \right]\nonumber\\ 
    & + \mathbb{E}\left[\left\|\alpha_i^2 \mathbb{I}_i^t\eta_t \left(\nabla f_i(\bar{\bm{v}}_i^{(t)}) -  \nabla f_i(\bar{\bm{v}}_i^{(t)};\xi_i^t)\right) + (1-\alpha_i)\eta_t\left(\frac{1}{K}\sum_{j \in U_t}\nabla f_j(\bm{w}^{(t)}_j)-\frac{1}{K}\sum_{j\in U_t} \nabla f_j(\bm{w}^{(t)}_j;\xi^t)\right)\right\|^2 \right]\nonumber \\
    & = \mathbb{E}\left[\| \hat{\bm{v}} _i^{(t)}- \bm{v}_i^{*} \|^2 \right]- 2\left\langle \frac{K}{n} \alpha_i^2 \eta_t \nabla f_i(\bar{\bm{v}}_i^{(t)})+ (1-\alpha_i)\eta_t\frac{1}{K}\sum_{j \in U_t}\nabla f_j(\bm{w}^{(t)}_j),\hat{\bm{v}} _i^{(t)}- \bm{v}_i^{*} \right\rangle  \nonumber\\ 
    &  +\eta_t^2\mathbb{E}\left[\left\| \alpha_i^2\mathbb{I}_i^t \nabla f_i( \bar{\bm{v}} _i^{(t)})+ (1-\alpha_i)\frac{1}{K}\sum_{j \in U_t}\nabla f_j(\bm{w}^{(t)}_j)\right\|^2 \right] + \alpha_i^2\eta_t^2\frac{2 K^2 \sigma^2 }{n^2} +(1-\alpha_i)^2\eta_t^2 \frac{2 \sigma^2 }{K}.\nonumber \\
    & = \mathbb{E}\left[\| \hat{\bm{v}} _i^{(t)}- \bm{v}_i^{*} \|^2 \right]\underbrace{- 2\eta_t\left\langle \left(\frac{K}{n} \alpha_i^2 +1 -\alpha_i \right) \nabla f_i(\bar{\bm{v}}_i^{(t)}) ,\hat{\bm{v}} _i^{(t)}- \bm{v}_i^{*} \right\rangle}_{T_1}  \nonumber\\ 
    & \quad  \underbrace{- 2\eta_t (1-\alpha_i) \mathbb{E}\left[ \left\langle \frac{1}{K}\sum_{j \in U_t}\nabla f_j(\bm{w}^{(t)}_j) - \nabla f_i(\bar{\bm{v}}_i^{(t)}),   \hat{\bm{v}}_i^{(t)}- \bm{v}_i^{*} \right\rangle\right]}_{T_2}\nonumber\\
      &  +\eta_t^2\underbrace{\mathbb{E}\left[\left\| \alpha_i^2\mathbb{I}_i^t \nabla f_i( \bar{\bm{v}} _i^{(t)})+ (1-\alpha_i)\frac{1}{K}\sum_{j \in U_t}\nabla f_j(\bm{w}^{(t)}_j)\right\|^2 \right]}_{T_3} + \alpha_i^2\eta_t^2\frac{2 K^2 \sigma^2 }{n^2} +(1-\alpha_i)^2\eta_t^2 \frac{2 \sigma^2 }{K}.\nonumber 
\end{align}}

Now we switch to bound $T_1$:
\begin{align}
    T_1 &= -2\eta_t(\frac{K}{n}\alpha_i^2+1-\alpha_i)\mathbb{E}\left[\left\langle   \nabla f_i(\hat{\bm{v}}_i^{(t)}),  \hat{\bm{v}} _i^{(t)}- \bm{v}_i^{*} \right\rangle\right] \nonumber\\
    & \quad -2\eta_t(\frac{K}{n}\alpha_i^2+1-\alpha_i) \mathbb{E}\left[\left\langle  \nabla f_i(\bar{\bm{v}}_i^{(t)}) - \nabla f_i(\hat{\bm{v}}_i^{(t)}),    \hat{\bm{v}} _i^{(t)} - \bm{v}_i^{*}\right \rangle\right] \nonumber\\
    &\leq -2\eta_t(\frac{K}{n}\alpha_i^2+1-\alpha_i)\left( \mathbb{E}\left[f_i( \hat{\bm{v}} _i^{(t)})\right]- f_i(\boldsymbol{v} _i^{*}) + \frac{\mu}{2}  \mathbb{E}\left[\| \hat{\bm{v}} _i^{(t)} - \boldsymbol{v} _i^{*}\|^2\right]\right) \nonumber\\
    & \quad   + (\frac{K}{n}\alpha_i^2+1-\alpha_i)\eta_t \left(\frac{8L^2}{\mu(1-8(\alpha_i - \alpha_i^2\frac{K}{n}))} \mathbb{E}\left[\| \hat{\bm{v}} _i^{(t)}- \boldsymbol{\bar{v}} _i^{(t)} \|^2\right] + \frac{\mu(1-8(\alpha_i - \alpha_i^2\frac{K}{n}))}{8} \mathbb{E}\left[\|  \hat{\bm{v}} _i^{(t)}- \bm{v}_i^{*}\|^2\right]\right)    \nonumber\\
    &\leq -2\eta_t(\frac{K}{n}\alpha_i^2+1-\alpha_i)\left( \mathbb{E}\left[f_i( \hat{\bm{v}} _i^{(t)})\right]- f_i(\boldsymbol{v} _i^{*}) + \frac{\mu}{2}  \mathbb{E}\left[\| \hat{\bm{v}} _i^{(t)} - \boldsymbol{v} _i^{*}\|^2\right]\right) \nonumber\\
    & \quad    +  \eta_t \left(\frac{8L^2 (1-\alpha_i)^2}{\mu(1-8(\alpha_i - \frac{K}{n}\alpha_i^2))} \mathbb{E}\left[\|\boldsymbol{w}^{(t)} - \boldsymbol{w} _i^{(t)} \|^2\right] + \frac{\mu(1-8(\alpha_i - \frac{K}{n}\alpha_i^2))}{8} \mathbb{E}\left[\| \hat{\bm{v}} _i^{(t)}- \bm{v}_i^{*}\|^2\right]\right) \nonumber\\
    &\leq -2\eta_t(\frac{K}{n}\alpha_i^2+1-\alpha_i)\left( \mathbb{E}\left[f_i( \hat{\bm{v}} _i^{(t)})\right]- f_i(\boldsymbol{v} _i^{*}) \right) - \frac{7\mu\eta_t}{8} \mathbb{E}\left[ \| \hat{\bm{v}} _i^{(t)} - \boldsymbol{v} _i^{*}\|^2\right] \nonumber\\
    & \quad  +  \frac{8\eta_tL^2 (1-\alpha_i)^2}{\mu(1-8(\alpha_i - \alpha_i^2))} \mathbb{E}\left[\|\boldsymbol{w}^{(t)} - \boldsymbol{w} _i^{(t)} \|^2\right],
\end{align}
For $T_2$, we use the same approach as we did in (\ref{Thm3 2}); To deal with $T_3$, we also employ the similar technique in (\ref{Thm3 3}):
\begin{align}
    T_3 &= \mathbb{E}\left[\left\| \alpha_i^2 \mathbb{I}_i^t \nabla f_i(\bar{\bm{v}}_i^{(t)})+ (1-\alpha_i)\frac{1}{K}\sum_{j\in U_t} \nabla f_j(\bm{w}^{(t)}_j)\right\|^2\right] \nonumber\\
    &\leq 2(\frac{K}{n}\alpha_i^2+1-\alpha_i)^2\mathbb{E}\left[\| \nabla f_i(\bar{\bm{v}}_i^{(t)})\|^2\right] +2\mathbb{E}\left[\left\|  (1-\alpha_i)\left(\frac{1}{K}\sum_{j\in U_t}\nabla f_j(\bm{w}^{(t)}_j)-\nabla f_i(\bar{\bm{v}}_i^{(t)})\right)\right\|^2\right] \nonumber\\
    &\leq 2\left(2(\frac{K}{n}\alpha_i^2+1-\alpha_i)^2\mathbb{E}\left[\| \nabla f_i(\hat{\bm{v}}_i^{(t)}) - \nabla f_i^*\|^2\right]+2(\frac{K}{n}\alpha_i^2+1-\alpha_i)^2\mathbb{E}\left[\|\nabla f_i(\bar{\bm{v}}^{(t)}_i)-\nabla f_i(\hat{\bm{v}}_i^{(t)})\|^2\right]\right)  \nonumber\\
    & \quad+2(1-\alpha_i)^2\mathbb{E}\left[\left\| \frac{1}{K}\sum_{j\in U_t}\nabla f_j(\bm{w}^{(t)}_j)-\nabla f_i(\bar{\bm{v}}_i^{(t)})\right\|^2\right] \nonumber\\
    &\leq 8L(\frac{K}{n}\alpha_i^2+1-\alpha_i)\left (\mathbb{E}\left[f_i(\hat{\bm{v}}^{(t)}_i)\right] -  f_i^*\right)+4(1-\alpha_i)^2L^2\mathbb{E}\left[\|\boldsymbol{w}^{(t)} - \boldsymbol{w} _i^{(t)}\|^2\right] \nonumber\\
    & \quad +6(1-\alpha_i)^2\left(L^2\mathbb{E}\left[\left\| \boldsymbol{w}^{(t)} - \boldsymbol{w} _i^{(t)}\right\|^2\right] +\mathbb{E}\left[\left\|\nabla f_i(\hat{\bm{v}}_i^{(t)}) - \nabla F(\bm{w}^{(t)})\right\|^2\right]+  \frac{1}{K}\sum_{j\in U_t} L^2\mathbb{E}\left[\left\| \boldsymbol{w}^{(t)} - \boldsymbol{w} _j^{(t)}\right\|^2\right]\right). 
\end{align}

Then plugging $T_1, T_2, T_3$ back, we obtain the similar formulation as the without sampling case in~(\ref{proof thm5 0}). Thus:
{\small\begin{align}
   &\mathbb{E}\left[\| \boldsymbol{\hat{v}} _i^{(t+1)}- \bm{v}_i^{*} \|^2\right]\nonumber\\
   &\leq \left(1-\frac{3\mu \eta_t}{8}\right)\mathbb{E}\left[\| \boldsymbol{\hat{v}} _i^{(t)}- \bm{v}_i^{*} \|^2\right] -2(\eta_t-4\eta_t^2L)\left(\alpha_i^2\frac{K}{n}+1-\alpha_i\right)\left(\mathbb{E}\left[f_i(\boldsymbol{\hat{v}} _i^{(t)})\right]- f_i(\boldsymbol{v} _i^{*}) \right) \nonumber\\
   &  + \alpha_i^2\eta_t^2\frac{2 K \sigma^2 }{n } +(1-\alpha_i)^2\eta_t^2 \frac{2 \sigma^2 }{K} \nonumber\\
    &   + \left(\frac{8\eta_tL^2 (1-\alpha_i)^2}{\mu(1-8(\alpha_i - \alpha_i^2\frac{K}{n}))} +\frac{6(1-\alpha_i)^2\eta_t L^2}{\mu} + 10(1-\alpha_i)^2\eta_t^2L^2 \right)\mathbb{E}\left[\left\| \boldsymbol{w}^{(t)} - \boldsymbol{w} _i^{(t)}\right\|^2\right]\nonumber\\ & + \left( \frac{6(1-\alpha_i)^2\eta_t L^2}{\mu} + 6(1-\alpha_i)^2\eta_t^2L^2 \right) \frac{1}{K} \sum_{j\in U_t} \mathbb{E}\left[\left\| \boldsymbol{w}^{(t)} - \boldsymbol{w} _j^{(t)}\right\|^2\right] \nonumber\\
     & +  \left( \frac{6\eta_t}{\mu}+ 6\eta_t^2 \right)    (1-\alpha_i)^2\mathbb{E}\left[\left\|\frac{1}{K} \sum_{j\in U_t}\nabla f_j(\bm{w}^{(t)})-\nabla f_i(\hat{\bm{v}}_i^{(t)})\right\|^2\right].  \label{eq: th3_1} 
\end{align}}
we firstly examine the lower bound of $\alpha_i^2\frac{K}{n}+1-\alpha_i$. Notice that: $\alpha_i^2\frac{K}{n}+1-\alpha_i = \frac{K}{n}((\alpha_i-\frac{n}{2K})^2 + \frac{n}{K}- \frac{n^2}{4 K^2})$.
\paragraph{Case 1: $\frac{n}{2K}\geq 1$} The lower bound is attained when $\alpha_i = 1$: $\alpha_i^2\frac{K}{n}+1-\alpha_i \geq \frac{K}{n}$.
\paragraph{Case 2: $\frac{n}{2K} < 1$} The lower bound is attained when $\alpha_i = \frac{n}{2K}$: $\alpha_i^2\frac{K}{n}+1-\alpha_i \geq 1 - \frac{n}{4K} > \frac{1}{2}$.

So $\alpha_i^2\frac{K}{n}+1-\alpha_i \geq b := \min\{\frac{K}{n}, \frac{1}{2}\}$ always holds.

 Now we plug it and Lemma~\ref{lm: deviation with sampling} back to (\ref{eq: th3_1}):
{\small\begin{align}
   &\mathbb{E}\left[\| \boldsymbol{\hat{v}} _i^{(t+1)}- \bm{v}_i^{*} \|^2\right]\nonumber\\
   &\leq \left(1-\frac{3\mu \eta_t}{8}\right)\mathbb{E}\left[\| \boldsymbol{\hat{v}} _i^{(t)}- \bm{v}_i^{*} \|^2\right] -b\eta_t\left(\mathbb{E}\left[f_i(\boldsymbol{\hat{v}} _i^{(t)})\right]- f_i(\boldsymbol{v} _i^{*}) \right) \nonumber\\
   &  + \alpha_i^2\eta_t^2\frac{2 K \sigma^2 }{n } +(1-\alpha_i)^2\eta_t^2 \frac{2 \sigma^2 }{K} \nonumber\\
    &   + \left(\frac{8\eta_tL^2 (1-\alpha_i)^2}{\mu(1-8(\alpha_i - \alpha_i^2\frac{K}{n}))} +\frac{6(1-\alpha_i)^2\eta_t L^2}{\mu} + 10(1-\alpha_i)^2\eta_t^2L^2 \right)3\tau \eta_{t-1}^2 \left(\sigma^2   + (\zeta_i + \frac{\zeta}{K}) \tau \right )\nonumber\\
    & + \left( \frac{6(1-\alpha_i)^2\eta_t L^2}{\mu} + 6(1-\alpha_i)^2\eta_t^2L^2 \right) 3\tau \eta_{t-1}^2 \left(\sigma^2   + 2 \frac{\zeta}{K}  \tau  \right) \nonumber\\
     & +  \left( \frac{6\eta_t}{\mu}+ 6\eta_t^2 \right)    (1-\alpha_i)^2\mathbb{E}\left[\left\|\frac{1}{K} \sum_{j\in U_t}\nabla f_j(\bm{w}^{(t)})-\nabla f_i(\hat{\bm{v}}_i^{(t)})\right\|^2\right]. \label{eq: proof_alpha_condition 1}   
\end{align}}

Plugging Lemma~\ref{lm: local-global gradient gap with sampling} yields:
{\small\begin{align}
   &\mathbb{E}\left[\| \boldsymbol{\hat{v}} _i^{(t+1)}- \bm{v}_i^{*} \|^2\right]\nonumber\\
   &\leq \left(1-\frac{3\mu \eta_t}{8}\right)\mathbb{E}\left[\| \boldsymbol{\hat{v}} _i^{(t)}- \bm{v}_i^{*} \|^2\right] -b\eta_t\left(\mathbb{E}\left[f_i(\boldsymbol{\hat{v}} _i^{(t)})\right]- f_i(\boldsymbol{v} _i^{*}) \right) \nonumber\\
   &  + \alpha_i^2\eta_t^2\frac{2 K \sigma^2 }{n} +(1-\alpha_i)^2\eta_t^2 \frac{2 \sigma^2 }{K} \nonumber\\
    &   + \left(\frac{8\eta_tL^2 (1-\alpha_i)^2}{\mu(1-8(\alpha_i - \alpha_i^2\frac{K}{n}))} +\frac{6(1-\alpha_i)^2\eta_t L^2}{\mu} + 10(1-\alpha_i)^2\eta_t^2L^2 \right)3\tau \eta_{t-1}^2 \left(\sigma^2   + (\zeta_i + \frac{\zeta}{K}) \tau \right )\nonumber\\
    & + \left( \frac{6(1-\alpha_i)^2\eta_t L^2}{\mu} + 6(1-\alpha_i)^2\eta_t^2L^2 \right) 3\tau \eta_{t-1}^2 \left(\sigma^2   + 2 \frac{\zeta}{K}  \tau  \right) \nonumber\\
     & +  \left( \frac{6\eta_t}{\mu}+ 6\eta_t^2 \right)    (1-\alpha_i)^2 \left[2L^2\mathbb{E}\left[\|\hat{\bm{v}}_i^{(t)} - \bm{v}^*\|^2\right] + 6\left(2\zeta_i + 2\frac{\zeta}{K}  \right)  + 6L^2\mathbb{E}\left[\|\bm{w}^{(t)} - \bm{w}^*\|^2\right] +  6L^2\Delta_i \right].   \nonumber
\end{align}}

Then following the same procedure in Appendix~\ref{proof thm5}, together with the application of Lemma \ref{lm6} we can conclude that:
{\begin{align}
    & f_i(\boldsymbol{\hat{v}} _i)- f_i(\boldsymbol{v} _i^{*}) \nonumber\\
   &\leq \frac{1}{S_T} \sum_{t=1}^Tp_t(f_i(\boldsymbol{\hat{v}} _i^{(t)})- f_i(\boldsymbol{v} _i^{*}) )       \nonumber\\
   &\leq \frac{ p_{0}\mathbb{E}\left[\| \boldsymbol{\hat{v}} _i^{(1)}- \bm{v}_i^{*} \|^2\right]}{ b\eta_{0}S_T}+ \frac{1}{bS_T}  \sum_{t=1}^T p_t\eta_t\left(\alpha_i^2\eta_t^2\frac{2 K \sigma^2 }{n} +(1-\alpha_i)^2\eta_t^2 \frac{2 \sigma^2 }{K}\right) \nonumber\\
   &  + \frac{1}{bS_T} \sum_{t=1}^T(1-\alpha_i)^2L^2 \left(\frac{8  }{\mu(1-8(\alpha_i - \alpha_i^2\frac{K}{n}))} +\frac{6  }{\mu} + 10 \eta_t  \right)3\tau p_t\eta_{t-1}^2 \left(\sigma^2   + (\zeta_i + \frac{\zeta}{K}) \tau \right )\nonumber\\
     &  + \frac{1}{bS_T} \sum_{t=1}^T(1-\alpha_i)^2L^2 \left( \frac{6  }{\mu} + 10 \eta_t  \right)3\tau p_t\eta_{t-1}^2 \left(\sigma^2   + 2 \frac{\zeta}{K}  \tau \right )\nonumber\\
     & + 48(1-\alpha_i)^2   \frac{ L^2}{bS_T} \sum_{t=1}^Tp_t  \nonumber \\
     & \left(\frac{1}{\mu}+\eta_t\right)  \left(\frac{a^3}{ (t-1+a)^3}  \mathbb{E}\left[\|\bm{w}^{(1)} - \bm{w}^*\|^2\right]   +\left(t+16\left(\frac{1}{a+1}+\ln (t+a)\right)\right) \frac{1536a^2\tau\left(\sigma^2 + 2\tau \frac{\zeta}{K}  \right)L^2 }{(a-1)^2\mu^4(t-1+a)^3}+\frac{128\sigma^2 t (t+2a)}{K\mu^2(t-1+a)^3} \right) \nonumber\\
      & +48(1-\alpha_i)^2   \left( \zeta_i +  L^2\Delta_i \right)\frac{1}{bS_T} \sum_{t=1}^Tp_t\left(\frac{1}{\mu}+\eta_t\right).  \nonumber \\
      & = O\left(\frac{\mu  }{bT^3}\right) +\alpha_i^2O\left(\frac{\sigma^2}{\mu b T}\right)+(1-\alpha_i)^2O\left(\frac{\zeta_i}{\mu b} +  \frac{\kappa L\Delta_i}{b}\right)\nonumber\\
      & \quad+ (1-\alpha_i)^2\left(O\left(\frac{\kappa L \ln T}{bT^3}\right) +O\left(\frac{\kappa^2\sigma^2}{\mu bKT}\right)+ O\left(\frac{\kappa^2 \tau^2(\zeta_i + \frac{\zeta}{K})+\kappa^2 \tau\sigma^2 }{\mu bT^2} \right)+ O\left(\frac{\kappa^4 \tau\left(\sigma^2 + 2\tau \frac{\zeta}{K}  \right) }{\mu bT^2}\right)  \right). \nonumber
\end{align}}
\end{proof}
\section{Proof of Convergence without Assumption on $\alpha_i$}\label{app:cond_alpha}

In this section, we present  the proof of Theorem~\ref{without constraint}.
\begin{proof}
According to (\ref{eq: proof_alpha_condition 1}):
 \begin{align}
   &\mathbb{E}\left[\| \boldsymbol{\hat{v}} _i^{(t+1)}- \bm{v}_i^{*} \|^2\right]\nonumber\\
   &\leq \left(1-\frac{3\mu \eta_t}{8}\right)\mathbb{E}\left[\| \boldsymbol{\hat{v}} _i^{(t)}- \bm{v}_i^{*} \|^2\right] -b\eta_t\left(\mathbb{E}\left[f_i(\boldsymbol{\hat{v}} _i^{(t)})\right]- f_i(\boldsymbol{v} _i^{*}) \right) \nonumber\\
   &  \quad + \alpha_i^2\eta_t^2\frac{2 K \sigma^2 }{n } +(1-\alpha_i)^2\eta_t^2 \frac{2 \sigma^2 }{K} \nonumber\\
    &  \quad + \left(\frac{8\eta_tL^2 (1-\alpha_i)^2}{\mu(1-8(\alpha_i - \alpha_i^2\frac{K}{n}))} +\frac{6(1-\alpha_i)^2\eta_t L^2}{\mu} + 10(1-\alpha_i)^2\eta_t^2L^2 \right)3\tau \eta_{t-1}^2 \left(\sigma^2   + (\zeta_i + \frac{\zeta}{K}) \tau \right )\nonumber\\
    &\quad + \left( \frac{6(1-\alpha_i)^2\eta_t L^2}{\mu} + 6(1-\alpha_i)^2\eta_t^2L^2 \right) 3\tau \eta_{t-1}^2 \left(\sigma^2   + 2 \frac{\zeta}{K}  \tau  \right) \nonumber\\
     &\quad +  \left( \frac{6\eta_t}{\mu}+ 6\eta_t^2 \right)    (1-\alpha_i)^2\mathbb{E}\left[\left\|\frac{1}{K} \sum_{j\in U_t}\nabla f_j(\bm{w}^{(t)})-\nabla f_i(\hat{\bm{v}}_i^{(t)})\right\|^2\right].  \nonumber
\end{align} 
Here, we directly use the bound $\mathbb{E}\left[\left\|\frac{1}{K} \sum_{j\in U_t}\nabla f_j(\bm{w}^{(t)})-\nabla f_i(\hat{\bm{v}}_i^{(t)})\right\|^2\right] \leq 2G^2$. Then we have:
 \begin{align}
   &\mathbb{E}\left[\| \boldsymbol{\hat{v}} _i^{(t+1)}- \bm{v}_i^{*} \|^2\right]\nonumber\\
   &\leq \left(1-\frac{\mu \eta_t}{4}\right)\mathbb{E}\left[\| \boldsymbol{\hat{v}} _i^{(t)}- \bm{v}_i^{*} \|^2\right] -b\eta_t\left(\mathbb{E}\left[f_i(\boldsymbol{\hat{v}} _i^{(t)})\right]- f_i(\boldsymbol{v} _i^{*}) \right) \nonumber\\
   & \quad + \alpha_i^2\eta_t^2\frac{2 K \sigma^2 }{n } +(1-\alpha_i)^2\eta_t^2 \frac{2 \sigma^2 }{K} \nonumber\\
    & \quad  + \left(\frac{8\eta_tL^2 (1-\alpha_i)^2}{\mu(1-8(\alpha_i - \alpha_i^2\frac{K}{n}))} +\frac{6(1-\alpha_i)^2\eta_t L^2}{\mu} + 10(1-\alpha_i)^2\eta_t^2L^2 \right)3\tau \eta_{t-1}^2 \left(\sigma^2   + (\zeta_i + \frac{\zeta}{K}) \tau \right )\nonumber\\
    & \quad + \left( \frac{6(1-\alpha_i)^2\eta_t L^2}{\mu} + 6(1-\alpha_i)^2\eta_t^2L^2 \right) 3\tau \eta_{t-1}^2 \left(\sigma^2   + 2 \frac{\zeta}{K}  \tau  \right) \nonumber\\
     & \quad +  \left( \frac{12\eta_t}{\mu}+ 12\eta_t^2 \right)    (1-\alpha_i)^2 G^2.  \nonumber
\end{align} 

Then following the same procedure in Appendix~\ref{proof thm5},  we can conclude that:
{\begin{align}
    & f_i(\boldsymbol{\hat{v}} _i)- f_i(\boldsymbol{v} _i^{*}) \nonumber\\
   &\leq \frac{1}{S_T} \sum_{t=1}^Tp_t(f_i(\boldsymbol{\hat{v}} _i^{(t)})- f_i(\boldsymbol{v} _i^{*}) )       \nonumber\\
   &\leq \frac{ p_{0}\mathbb{E}\left[\| \boldsymbol{\hat{v}} _i^{(1)}- \bm{v}_i^{*} \|^2\right]}{ b\eta_{0}S_T}+ \frac{1}{bS_T}  \sum_{t=1}^T p_t\eta_t\left(\alpha_i^2\eta_t^2\frac{2 K \sigma^2 }{n} +(1-\alpha_i)^2\eta_t^2 \frac{2 \sigma^2 }{K}\right) \nonumber\\
   & \quad + \frac{1}{bS_T} \sum_{t=1}^T(1-\alpha_i)^2L^2 \left(\frac{8  }{\mu(1-8(\alpha_i - \alpha_i^2\frac{K}{n}))} +\frac{6  }{\mu} + 10 \eta_t  \right)3\tau p_t\eta_{t-1}^2 \left(\sigma^2   + (\zeta_i + \frac{\zeta}{K}) \tau \right )\nonumber\\
     & \quad + \frac{1}{bS_T} \sum_{t=1}^T(1-\alpha_i)^2L^2 \left( \frac{6  }{\mu} + 10 \eta_t  \right)3\tau p_t\eta_{t-1}^2 \left(\sigma^2   + 2 \frac{\zeta}{K}  \tau \right )\nonumber\\
      & \quad +12(1-\alpha_i)^2  G^2\frac{1}{bS_T} \sum_{t=1}^Tp_t\left(\frac{1}{\mu}+\eta_t\right).  \nonumber \\
      & = O\left(\frac{\mu  }{bT^3}\right) +\alpha_i^2O\left(\frac{\sigma^2}{\mu b T}\right)+(1-\alpha_i)^2O\left(\frac{G^2}{\mu b} \right)\nonumber\\
      & \quad+ (1-\alpha_i)^2\left(O\left(\frac{\kappa L \ln T}{bT^3}\right) +O\left(\frac{\kappa^2\sigma^2}{\mu bKT}\right)+ O\left(\frac{\kappa^2 \tau^2(\zeta_i + \frac{\zeta}{K})+\kappa^2 \tau\sigma^2 }{\mu bT^2} \right)+ O\left(\frac{\kappa^4 \tau\left(\sigma^2 + 2\tau \frac{\zeta}{K}  \right) }{\mu bT^2}\right)  \right). \nonumber
\end{align}}
\end{proof}

\section{Proof of Nonconvex Loss Convergence} \label{app: nonconvex}
In this section we will provide the proof of convergence results on nonconvex functions. As usual, let us first introduce several useful lemmas.
\subsection{Proof of Technical Lemmas}
\begin{lemma} \label{lm: nonconvex lm1}
Under Theorem~\ref{Thm: nonconvex}'s assumptions, the following statement holds true:
\begin{align}
    \mathbb{E}\left[f_i(\hat{\bm{v}}^{(t+1)}_i) \right] & \leq \mathbb{E} \left[ f_i(\hat{\bm{v}}^{(t)}_i) \right] -\frac{\eta}{2}\mathbb{E} \left[\left \| \nabla f_i (\hat{\bm{v}}^{(t)}_i) \right\|^2\right]  + \frac{\eta^2 L}{2}\left(\alpha_i^4 \sigma^2 + (1-\alpha_i)^2 \frac{\sigma^2}{n}\right)  +  2\alpha_i^4 (1-\alpha_i)^2\eta L^2 \mathbb{E} \left[\left \|  \bm{w}^{(t)}_i  -  \bm{w}^{(t)}\right\|^2  \right] \nonumber\\
    & + (1-\alpha_i)^2 \eta \frac{1}{n}\sum_{j=1}^n\mathbb{E} \left[\left \| \bm{w}^{(t)} - \bm{w}^{(t)}_j  \right\|^2  \right]+ 4\eta (1-\alpha_i^2)^2  \zeta_i  + 8\eta L^2 (1-\alpha_i^2)^2 \Gamma + 8(\alpha_i-\alpha_i^2)^2 \eta \mathbb{E} \left[\left \| \nabla F(\bm{w}^{(t)}) \right\|^2  \right]\nonumber.
\end{align}
\begin{proof}
According to the updating rule and smoothness of $f_i$, we have:
\begin{align}
    f_i(\hat{\bm{v}}^{(t+1)}_i) \leq  f_i(\hat{\bm{v}}^{(t)}_i) + \left \langle \nabla f_i (\hat{\bm{v}}^{(t)}_i), \hat{\bm{v}}^{(t+1)}_i - \hat{\bm{v}}^{(t)}_i \right \rangle + \frac{L}{2}\left \| \hat{\bm{v}}^{(t+1)}_i - \hat{\bm{v}}^{(t)}_i \right\|^2.\nonumber
\end{align}
Taking expectation on both sides yields:
\begin{align}
    \mathbb{E} \left[f_i(\hat{\bm{v}}^{(t+1)}_i)\right] &\leq \mathbb{E} \left[ f_i(\hat{\bm{v}}^{(t)}_i) \right]+  \mathbb{E} \left[\left \langle \nabla f_i (\hat{\bm{v}}^{(t)}_i), \hat{\bm{v}}^{(t+1)}_i - \hat{\bm{v}}^{(t)}_i \right \rangle \right] + \frac{L}{2} \mathbb{E} \left[\left \| \hat{\bm{v}}^{(t+1)}_i - \hat{\bm{v}}^{(t)}_i \right\|^2\right]\nonumber\\
    & \leq \mathbb{E} \left[ f_i(\hat{\bm{v}}^{(t)}_i) \right] -  \eta \mathbb{E} \left[\left \langle \nabla f_i (\hat{\bm{v}}^{(t)}_i),    \alpha_i^2 \nabla f_i (\bar{\bm{v}}_{i}^{(t)}) + (1-\alpha_i) \frac{1}{n}\sum_{j=1}^n \nabla f_j (\bm{w}^{(t)}_j)    \right \rangle \right]\nonumber\\
    & \quad+ \frac{\eta^2 L}{2} \mathbb{E} \left[\left \| \alpha_i^2 \nabla f_i (\bar{\bm{v}}_{i}^{(t)}) + (1-\alpha_i) \frac{1}{n}\sum_{j=1}^n \nabla f_j (\bm{w}^{(t)}_j) \right\|^2\right] + \frac{\eta^2 L}{2}\left(\alpha_i^4 \sigma^2 + (1-\alpha_i)^2 \frac{\sigma^2}{n}\right).\nonumber
\end{align}
 Using the identity: $\langle \bm{a}, \bm{b} \rangle = \frac{1}{2}\|\bm{a}\|^2+\frac{1}{2}\|\bm{b}\|^2-\frac{1}{2}\|\bm{a}-\bm{b}\|^2$ we have:

 \begin{align}
    \mathbb{E} \left[f_i(\hat{\bm{v}}^{(t+1)}_i)\right] 
     & \leq \mathbb{E} \left[ f_i(\hat{\bm{v}}^{(t)}_i) \right] -\frac{\eta}{2}\mathbb{E} \left[\left \| \nabla f_i (\hat{\bm{v}}^{(t)}_i) \right\|^2\right] -  \underbrace{\left( \frac{\eta}{2} -\frac{\eta^2 L}{2}\right)}_{\geq 0}\mathbb{E} \left[\left \| \alpha_i^2 \nabla f_i (\bar{\bm{v}}_{i}^{(t)}) + (1-\alpha_i) \frac{1}{n}\sum_{j=1}^n \nabla f_j (\bm{w}^{(t)}_j) \right \|^2 \right] \nonumber\\
    & \quad+ \frac{\eta}{2}\mathbb{E} \left[\left \| \nabla f_i (\hat{\bm{v}}^{(t)}_i) - \alpha_i^2 \nabla f_i (\bar{\bm{v}}_{i}^{(t)}) - (1-\alpha_i) \frac{1}{n}\sum_{j=1}^n \nabla f_j (\bm{w}^{(t)}_j) \right\|^2\right] + \frac{\eta^2 L}{2}\left(\alpha_i^4 \sigma^2 + (1-\alpha_i)^2 \frac{\sigma^2}{n}\right)\nonumber\\ 
     & \leq \mathbb{E} \left[ f_i(\hat{\bm{v}}^{(t)}_i) \right] -\frac{\eta}{2}\mathbb{E} \left[\left \| \nabla f_i (\hat{\bm{v}}^{(t)}_i) \right\|^2\right]  + \frac{\eta^2 L}{2}\left(\alpha_i^4 \sigma^2 + (1-\alpha_i)^2 \frac{\sigma^2}{n}\right) \nonumber\\
    & \quad+  \eta \mathbb{E} \left[\left \|  \alpha_i^2\left(\nabla f_i (\hat{\bm{v}}^{(t)}_i) - \nabla f_i (\bar{\bm{v}}_{i}^{(t)}) \right) -(1-\alpha_i) \nabla F(\bm{w}^{(t)})  + (1-\alpha_i^2)\nabla f_i (\hat{\bm{v}}^{(t)}_i) \right\|^2  \right] \nonumber \\
     &\quad + (1-\alpha_i)^2 \eta \mathbb{E} \left[\left \|   \nabla F(\bm{w}^{(t)}) - \frac{1}{n}\sum_{j=1}^n \nabla f_j (\bm{w}^{(t)}_j)  \right\|^2  \right] \nonumber \\
    & \leq \mathbb{E} \left[ f_i(\hat{\bm{v}}^{(t)}_i) \right] -\frac{\eta}{2}\mathbb{E} \left[\left \| \nabla f_i (\hat{\bm{v}}^{(t)}_i) \right\|^2\right]  + \frac{\eta^2 L}{2}\left(\alpha_i^4 \sigma^2 + (1-\alpha_i)^2 \frac{\sigma^2}{n}\right) \nonumber\\
    & \quad+ (1-\alpha_i)^2 \eta \frac{1}{n}\sum_{j=1}^n\mathbb{E} \left[\left \| \bm{w}^{(t)} - \bm{w}^{(t)}_j  \right\|^2  \right]+  2\eta \mathbb{E} \left[\left \|  \alpha_i^2\left(\nabla f_i (\hat{\bm{v}}^{(t)}_i) - \nabla f_i (\bar{\bm{v}}_{i}^{(t)}) \right)   \right\|^2  \right] \nonumber \\
    &\quad +  2\eta \mathbb{E} \left[\left \|(1-\alpha_i^2)\nabla f_i (\hat{\bm{v}}^{(t)}_i)-(1-\alpha_i^2) \nabla F(\hat{\bm{v}}^{(t)}_i) + (1-\alpha_i^2) \nabla F(\hat{\bm{v}}^{(t)}_i) - (1-\alpha_i) \nabla F(\bm{w}^{(t)}) \right\|^2  \right] \nonumber 
\end{align}
Using the smoothness of $f$ and $F$, together with applying Jensen's inequality on the last term yields:
 \begin{align}
    \mathbb{E} \left[f_i(\hat{\bm{v}}^{(t+1)}_i)\right]   & \leq \mathbb{E} \left[ f_i(\hat{\bm{v}}^{(t)}_i) \right] -\frac{\eta}{2}\mathbb{E} \left[\left \| \nabla f_i (\hat{\bm{v}}^{(t)}_i) \right\|^2\right]  + \frac{\eta^2 L}{2}\left(\alpha_i^4 \sigma^2 + (1-\alpha_i)^2 \frac{\sigma^2}{n}\right)  \nonumber\\
    & \quad+ (1-\alpha_i)^2 \eta \frac{1}{n}\sum_{j=1}^n\mathbb{E} \left[\left \| \bm{w}^{(t)} - \bm{w}^{(t)}_j  \right\|^2  \right]+  2\alpha_i^4\eta L^2 \mathbb{E} \left[\left \| \hat{\bm{v}}^{(t)}_i  -  \bar{\bm{v}}_{i}^{(t)}  \right\|^2  \right] \nonumber \\
    & \quad+ 4\eta \mathbb{E} \left[\left \|(1-\alpha_i^2)\nabla f_i (\hat{\bm{v}}^{(t)}_i)-(1-\alpha_i^2) \nabla F(\hat{\bm{v}}^{(t)}_i)\right\|^2  \right] + 4\eta \mathbb{E} \left[\left \|(1-\alpha_i^2) \nabla F(\hat{\bm{v}}^{(t)}_i) - (1-\alpha_i) \nabla F(\bm{w}^{(t)}) \right\|^2  \right] \nonumber  \\
     & \leq \mathbb{E} \left[ f_i(\hat{\bm{v}}^{(t)}_i) \right] -\frac{\eta}{2}\mathbb{E} \left[\left \| \nabla f_i (\hat{\bm{v}}^{(t)}_i) \right\|^2\right]  + \frac{\eta^2 L}{2}\left(\alpha_i^4 \sigma^2 + (1-\alpha_i)^2 \frac{\sigma^2}{n}\right)\nonumber\\  
     &\quad+ (1-\alpha_i)^2 \eta \frac{1}{n}\sum_{j=1}^n\mathbb{E} \left[\left \| \bm{w}^{(t)} - \bm{w}^{(t)}_j  \right\|^2  \right] +  2\alpha_i^4 (1-\alpha_i)^2\eta L^2 \mathbb{E} \left[\left \|  \bm{w}^{(t)}_i  -  \bm{w}^{(t)}\right\|^2  \right] \nonumber \\
    & \quad+ 4\eta (1-\alpha_i^2)^2  \zeta_i + 8\eta \mathbb{E} \left[\left \|(1-\alpha_i^2) \nabla F(\hat{\bm{v}}^{(t)}_i) - (1-\alpha_i^2) \nabla F(\bm{w}^{(t)}) \right\|^2  \right] + 8\eta \mathbb{E} \left[\left \| (\alpha_i-\alpha_i^2) \nabla F(\bm{w}^{(t)}) \right\|^2  \right]\nonumber  \\
    & \leq \mathbb{E} \left[ f_i(\hat{\bm{v}}^{(t)}_i) \right] -\frac{\eta}{2}\mathbb{E} \left[\left \| \nabla f_i (\hat{\bm{v}}^{(t)}_i) \right\|^2\right]  + \frac{\eta^2 L}{2}\left(\alpha_i^4 \sigma^2 + (1-\alpha_i)^2 \frac{\sigma^2}{n}\right)\nonumber\\ 
    &\quad+ (1-\alpha_i)^2 \eta \frac{1}{n}\sum_{j=1}^n\mathbb{E} \left[\left \| \bm{w}^{(t)} - \bm{w}^{(t)}_j  \right\|^2  \right] +  2\alpha_i^4 (1-\alpha_i)^2\eta L^2 \mathbb{E} \left[\left \|  \bm{w}^{(t)}_i  -  \bm{w}^{(t)}\right\|^2  \right]  + 4\eta (1-\alpha_i^2)^2  \zeta_i \nonumber\\
    &\quad + 8\eta L^2 (1-\alpha_i^2)^2  \mathbb{E} \left[\left \| \nabla F(\hat{\bm{v}}^{(t)}_i) - \nabla F(\bm{w}^{(t)}) \right\|^2  \right]+ 8(\alpha_i-\alpha_i^2)^2 \eta \mathbb{E} \left[\left \| \nabla F(\bm{w}^{(t)}) \right\|^2  \right]\nonumber  \\
    & \leq \mathbb{E} \left[ f_i(\hat{\bm{v}}^{(t)}_i) \right] -\frac{\eta}{2}\mathbb{E} \left[\left \| \nabla f_i (\hat{\bm{v}}^{(t)}_i) \right\|^2\right]  + \frac{\eta^2 L}{2}\left(\alpha_i^4 \sigma^2 + (1-\alpha_i)^2 \frac{\sigma^2}{n}\right)  +  2\alpha_i^4 (1-\alpha_i)^2\eta L^2 \mathbb{E} \left[\left \|  \bm{w}^{(t)}_i  -  \bm{w}^{(t)}\right\|^2  \right] \nonumber\\
    &\quad+ (1-\alpha_i)^2 \eta \frac{1}{n}\sum_{j=1}^n\mathbb{E} \left[\left \| \bm{w}^{(t)} - \bm{w}^{(t)}_j  \right\|^2  \right] + 4\eta (1-\alpha_i^2)^2  \zeta_i  + 8\eta L^2 (1-\alpha_i^2)^2 \Gamma + 8(\alpha_i-\alpha_i^2)^2 \eta \mathbb{E} \left[\left \| \nabla F(\bm{w}^{(t)}) \right\|^2  \right]\nonumber.
\end{align}

\end{proof}

\end{lemma}

\begin{lemma}\label{lm: nonconvex lm2}
Under Theorem~\ref{Thm: nonconvex}'s assumptions, the following statement holds true:
\begin{align}
  \frac{1}{T}\sum_{t=1}^T  \mathbb{E} \left[\left \| \nabla F(\bm{w}^{(t)})\right\|^2\right]    & \leq \frac{2}{\eta T}\mathbb{E} \left[ F(\bm{w}^{(1)}) \right]+  L^2\frac{1}{T}\sum_{t=1}^T  \frac{1}{n} \sum_{j=1}^n \mathbb{E} \left[\left \|\ \bm{w}^{(t)}_j - \bm{w}^{(t)}  \right \|^2 \right]+ \frac{\eta  L \sigma^2}{ n}. \nonumber
\end{align}
\begin{proof}
According to the updating rule and smoothness of $f_i$, we have:
\begin{align}
    F( \bm{w}^{(t+1)}) \leq  F(\bm{w}^{(t)}) + \left \langle \nabla F(\bm{w}^{(t)}), \bm{w}^{(t+1)} - \bm{w}^{(t)} \right \rangle + \frac{L}{2}\left \| \bm{w}^{(t+1)} - \bm{w}^{(t)} \right\|^2.\nonumber
\end{align}
Taking expectation on both sides yields:
\begin{align}
    \mathbb{E} \left[F( \bm{w}^{(t+1)})\right] &\leq \mathbb{E} \left[  F(\bm{w}^{(t)}) \right]+  \mathbb{E} \left[\left \langle \nabla F(\bm{w}^{(t)}), \bm{w}^{(t+1)} - \bm{w}^{(t)} \right \rangle \right] + \frac{L}{2} \mathbb{E} \left[\left \| \bm{w}^{(t+1)} - \bm{w}^{(t)} \right\|^2\right]\nonumber\\
    & \leq \mathbb{E} \left[ F(\bm{w}^{(t)}) \right] - \eta \mathbb{E} \left[\left \langle \nabla F(\bm{w}^{(t)}) ,  \frac{1}{n}\sum_{j=1}^n \nabla f_j (\bm{w}^{(t)}_j) \right \rangle \right] + \frac{\eta^2 L}{2} \mathbb{E} \left[\left \|  \frac{1}{n}\sum_{j=1}^n \nabla f_j (\bm{w}^{(t)}_j)\right\|^2\right] + \frac{\eta^2 L \sigma^2}{2n}  .\nonumber
\end{align}
 Using the identity $\langle \bm{a},\bm{b}\rangle = -\frac{1}{2}\|\bm{a} - \bm{b}\|^2 + \frac{1}{2}\|\bm{a}\|^2 + \frac{1}{2}\|\bm{b}\|^2 $, we have:
 
\begin{align}
    \mathbb{E} \left[F( \bm{w}^{(t+1)})\right] & \leq \mathbb{E} \left[ F(\bm{w}^{(t)}) \right] -\frac{\eta}{2}  \mathbb{E} \left[\left \| \nabla F(\bm{w}^{(t)})\right\|^2\right] -\left(\frac{\eta}{2} -\frac{\eta^2 L}{2}\right)\mathbb{E} \left[\left \|\frac{1}{n}\sum_{j=1}^n \nabla f_j (\bm{w}^{(t)}_j) \right \|^2 \right]   + \frac{\eta^2 L \sigma^2}{2n} \nonumber\\
    & \quad + \frac{\eta L^2}{2} \frac{1}{n} \sum_{i=1}^n \mathbb{E} \left[\left \|\ \bm{w}^{(t)}_j - \bm{w}^{(t)}  \right \|^2 \right] \nonumber\\
    & \leq \mathbb{E} \left[ F(\bm{w}^{(t)}) \right] -\frac{\eta}{2}  \mathbb{E} \left[\left \| \nabla F(\bm{w}^{(t)})\right\|^2\right]   + \frac{\eta^2 L \sigma^2}{2n}  + \frac{\eta L^2}{2} \frac{1}{n} \sum_{i=1}^n \mathbb{E} \left[\left \|\ \bm{w}^{(t)}_j - \bm{w}^{(t)}  \right \|^2 \right]. \nonumber
\end{align}
Re-arranging terms and doing the telescoping sum from $t = 1$ to $T$:
\begin{align}
  \frac{1}{T}\sum_{t=1}^T  \mathbb{E} \left[\left \| \nabla F(\bm{w}^{(t)})\right\|^2\right]    & \leq \frac{2}{\eta T}\mathbb{E} \left[ F(\bm{w}^{(1)}) \right]+  L^2\frac{1}{T}\sum_{t=1}^T  \frac{1}{n} \sum_{j=1}^n \mathbb{E} \left[\left \|\ \bm{w}^{(t)}_j - \bm{w}^{(t)}  \right \|^2 \right]+ \frac{\eta  L \sigma^2}{ n}. \nonumber
\end{align}

\end{proof}

\end{lemma}

\begin{lemma}\label{lm: nonconvex deviation}
Under Theorem~\ref{Thm: nonconvex}'s assumptions, the following statement holds true:
\begin{align} 
 &\frac{1}{T}\sum_{t=1}^{T}\frac{1}{n}\sum_{i=1}^n\mathbb{E}\left[\left\|\bm{w}^{(t)}-\bm{w}^{(t)}_i \right\|^2 \right]  \leq  10\tau^2 \eta^2 \left( \sigma^2 + \frac{\sigma^2}{n} +\frac{\zeta}{n} \right) ,\nonumber\\
  & \frac{1}{T}\sum_{t=1}^{T} \mathbb{E}\left[\left\|\bm{w}^{(t)}-\bm{w}^{(t)}_i \right\|^2 \right] \leq   200L^2\tau^4\eta^4\left( \sigma^2 + \frac{\sigma^2}{n}+ \frac{\zeta}{n}  \right)  +  20\tau^2 \eta^2 \left( \sigma^2 + \frac{\sigma^2}{n} + \zeta_i \right).\nonumber
\end{align}
\begin{proof}

 Similarly, for the second statement, we define $\gamma_t = \frac{1}{n}\sum_{i=1}^n\mathbb{E}\left[\left\|\bm{w}^{(t)}-\bm{w}^{(t)}_i \right\|^2 \right] $, and let $t_c$ be the latest synchronization stage. Then we have:

 \begin{align}
 \gamma_t &= \frac{1}{n}\sum_{i=1}^n\mathbb{E}\left[\left\| \bm{w}^{t_c} - \sum_{j=t_c}^{t}\frac{\eta}{n}\sum_{k=1}^n \nabla f_k(\bm{w}_k^{(j)};\xi_k^j)  -\left(\bm{w}^{t_c}  -\sum_{j=t_c}^{t}\eta  \nabla  f_i(\bm{w}_i^{(j)};\xi_i^j) \right) \right\|^2\right]\nonumber\\
 & =\tau \sum_{j=t_c}^{t}\frac{\eta^2}{n}\sum_{i=1}^n\mathbb{E}\left[\left\|   \frac{1}{n}\sum_{k=1}^n \nabla f_k(\bm{w}_k^{(j)};\xi_k^j)   -  \nabla  f_i(\bm{w}_i^{(j)};\xi_i^j)\right\|^2\right]\nonumber\\ 
  &  = \tau \sum_{j=t_c}^{t}\frac{\eta^2}{n} \sum_{i=1}^n\mathbb{E}\left[\left\|   \frac{1}{n}\sum_{k=1}^n \nabla f_k(\bm{w}_k^{(j)} ;\xi_k^j)  - \nabla f_k(\bm{w}_k^{(j)} ) +\nabla f_k(\bm{w}_k^{(j)}  ) -\nabla f_k(\bm{w}^{(j)} )  \right.\right. \nonumber\\
  & \quad \quad \left.\left. + \nabla f_k(\bm{w}^{(j)} ) -\nabla f_i(\bm{w}^{(j)} )+\nabla f_i(\bm{w}^{(j)} ) -\nabla f_i(\bm{w}_i^{(j)} ) +\nabla f_i(\bm{w}_i^{(j)} ) - \nabla f_i(\bm{w}_i^{(j)};\xi_i^t )  \right\|^2\right] \nonumber\\
  &  \leq \tau \sum_{j=t_c}^{t+\tau}5\eta^2 \left( \sigma^2 + \frac{\sigma^2}{n} +2L^2\gamma^j + \frac{\zeta}{n} \right). \nonumber  
 \end{align} 
 Summing over $t$ from $t_c$ to $t_c + \tau$ yields:
 \begin{align}
  \sum_{t=t_c}^{t_c + \tau} \gamma_t &\leq    \sum_{t=t_c}^{t_c+\tau}\sum_{j=t_c}^{t_c+\tau}5\tau\eta ^2 \left( \sigma^2 + \frac{\sigma^2}{n} +2L^2\gamma^j + \frac{\zeta}{n} \right)  \nonumber\\
 & \leq  10L^2   \tau^2    \eta^2\sum_{j=r\tau}^{(r+1)\tau} \gamma^j + 5\tau^3 \eta^2 \left( \sigma^2 + \frac{\sigma^2}{n}  \right)+5 \tau^3\eta^2 \frac{\zeta}{n} .  
 \end{align}
 Since $\eta \leq \frac{1}{2\sqrt{5}\tau L}$, we have $10L^2\tau^2 \eta^2\leq \frac{1}{2}$, hence by re-arranging the terms we have:
  \begin{align}
  \sum_{t=t_c}^{t_c + \tau} \gamma_t   \leq   10\tau^3 \eta^2 \left( \sigma^2 + \frac{\sigma^2}{n}  \right)+10\tau^3\eta^2\frac{\zeta}{n}.\nonumber
 \end{align}
 Summing over all synchronization stages $t_c$, and dividing both sides by $T$ can conclude the proof of the first statement:
  \begin{align}
 & \frac{1}{T}\sum_{t=1}^{T} \gamma_t   \leq   10\tau^2 \eta^2 \left( \sigma^2 + \frac{\sigma^2}{n}  \right)+10\tau^2\eta^2\frac{\zeta}{n}.\label{eq: lm nonconvex deviation_1}
 \end{align}
 
 To prove the second statement, let $\delta_t^i =\mathbb{E}\left[\left\|\bm{w}^{(t)}-\bm{w}^{(t)}_i \right\|^2 \right] $. Notice that:

 \begin{align}
 \delta_t^i &=  \mathbb{E}\left[\left\| \bm{w}^{t_c} - \sum_{j=t_c}^{t}\frac{\eta}{n}\sum_{k=1}^n \nabla f_k(\bm{w}_k^{(j)};\xi_k^j)  -\left(\bm{w}^{t_c}  -\sum_{j=t_c}^{t}\eta  \nabla  f_i(\bm{w}_i^{(j)};\xi_i^j) \right) \right\|^2\right]\nonumber\\
 & =\tau \sum_{j=t_c}^{t} \eta^2 \mathbb{E}\left[\left\|   \frac{1}{n}\sum_{k=1}^n \nabla f_k(\bm{w}_k^{(j)};\xi_k^j)   -  \nabla  f_i(\bm{w}_i^{(j)};\xi_i^j)\right\|^2\right]\nonumber\\ 
  &  = \tau \sum_{j=t_c}^{t} \eta^2  \mathbb{E}\left[\left\|   \frac{1}{n}\sum_{k=1}^n \nabla f_k(\bm{w}_k^{(j)} ;\xi_k^j)  - \nabla f_k(\bm{w}_k^{(j)} ) +\nabla f_k(\bm{w}_k^{(j)}  ) -\nabla f_k(\bm{w}^{(j)} )  \right.\right. \nonumber\\
  & \quad \quad \left.\left. + \nabla f_k(\bm{w}^{(j)} ) -\nabla f_i(\bm{w}^{(j)} )+\nabla f_i(\bm{w}^{(j)} ) -\nabla f_i(\bm{w}_i^{(j)} ) +\nabla f_i(\bm{w}_i^{(j)} ) - \nabla f_i(\bm{w}_i^{(j)};\xi_i^t )  \right\|^2\right] \nonumber\\
  &  \leq \tau \sum_{j=t_c}^{t+\tau}5\eta^2 \left( \sigma^2 + \frac{\sigma^2}{n} + L^2\gamma_j+ L^2\delta_j^i +  \zeta_i \right). \nonumber  
 \end{align}

 Summing over $t$ from $t_c$ to $t_c + \tau$ yields:
 \begin{align}
  \sum_{t=t_c}^{t_c + \tau} \gamma_t &\leq    \sum_{t=t_c}^{t_c+\tau}\sum_{j=t_c}^{t_c+\tau}5\tau\eta ^2 \left( \sigma^2 + \frac{\sigma^2}{n} + L^2\gamma_j+ L^2\delta_j^i + \zeta_i \right)  \nonumber\\
 & \leq  5L^2   \tau^2    \eta^2\sum_{j=t_c}^{t_c+\tau} \gamma_j+5L^2   \tau^2    \eta^2\sum_{j=t_c}^{t_c+\tau} \delta_j^i + 5\tau^3 \eta^2 \left( \sigma^2 + \frac{\sigma^2}{n}  \right)+5 \tau^3\eta^2\zeta_i .  \nonumber
 \end{align}
 Since $\eta \leq \frac{1}{2\sqrt{5}\tau L}$, we have $5L^2\tau^2 \eta^2\leq \frac{1}{4}$, hence by re-arranging the terms we have:
  \begin{align}
  \sum_{t=t_c}^{t_c + \tau} \delta_t^i   \leq  20L^2   \tau^2    \eta^2\sum_{j=t_c}^{t_c+\tau} \gamma_j +  20\tau^3 \eta^2 \left( \sigma^2 + \frac{\sigma^2}{n}  \right)  +20\tau^3\eta^2\zeta_i.\nonumber
 \end{align}
 Summing over all synchronization stages $t_c$, and dividing both sides by $T$ can conclude the proof of the first statement:
  \begin{align}
 & \frac{1}{T}\sum_{t=1}^{T} \delta_t^i  \leq   20L^2   \tau^2    \eta^2\frac{1}{T}\sum_{t=1}^{T} \gamma_t +  20\tau^2 \eta^2 \left( \sigma^2 + \frac{\sigma^2}{n}  \right)  +20\tau^2\eta^2\zeta_i.\nonumber
 \end{align}
 Using the result from (\ref{eq: lm nonconvex deviation_1}) to bound $2\frac{1}{T}\sum_{t=1}^{T} \gamma_t$ we have:
  \begin{align}
 & \frac{1}{T}\sum_{t=1}^{T} \delta_t^i  \leq   200L^2\tau^4\eta^4\left( \sigma^2 + \frac{\sigma^2}{n}+ \frac{\zeta}{n}  \right)  +  20\tau^2 \eta^2 \left( \sigma^2 + \frac{\sigma^2}{n} + \zeta_i \right).\nonumber
 \end{align}
\end{proof}
\end{lemma}

\subsection{Proof of Theorem~\ref{Thm: nonconvex_global}}\label{app: proof_nonconvex_global}
\begin{proof}
According to  Lemma~\ref{lm: nonconvex lm2}:
\begin{align}
  \frac{1}{T}\sum_{t=1}^T  \mathbb{E} \left[\left \| \nabla F(\bm{w}^{(t)})\right\|^2\right]    & \leq \frac{2}{\eta T}\mathbb{E} \left[ F(\bm{w}^{(1)}) \right]+  L^2\frac{1}{T}\sum_{t=1}^T  \frac{1}{n} \sum_{j=1}^n \mathbb{E} \left[\left \|\ \bm{w}^{(t)}_j - \bm{w}^{(t)}  \right \|^2 \right]+ \frac{\eta  L \sigma^2}{ n}. \nonumber
\end{align}
Then plugging in Lemma~\ref{lm: nonconvex deviation} will conclude the proof.

\end{proof}

\subsection{Proof of Theorem~\ref{Thm: nonconvex}}\label{app: proof_nonconvex_personalized}
\begin{proof}
According to Lemma~\ref{lm: nonconvex lm1}:
\begin{align}
    \mathbb{E}\left[f_i(\hat{\bm{v}}^{(t+1)}_i) \right] & \leq \mathbb{E} \left[ f_i(\hat{\bm{v}}^{(t)}_i) \right] -\frac{\eta}{2}\mathbb{E} \left[\left \| \nabla f_i (\hat{\bm{v}}^{(t)}_i) \right\|^2\right]  + \frac{\eta^2 L}{2}\left(\alpha_i^4 \sigma^2 + (1-\alpha_i)^2 \frac{\sigma^2}{n}\right)  +  2\alpha_i^4 (1-\alpha_i)^2\eta L^2 \mathbb{E} \left[\left \|  \bm{w}^{(t)}_i  -  \bm{w}^{(t)}\right\|^2  \right] \nonumber\\
    & \quad + (1-\alpha_i)^2 \eta \frac{1}{n}\sum_{j=1}^n\mathbb{E} \left[\left \| \bm{w}^{(t)} - \bm{w}^{(t)}_j  \right\|^2  \right]+ 4\eta (1-\alpha_i^2)^2  \zeta_i  + 8\eta (1-\alpha_i^2)^2 \Gamma + 8(\alpha_i-\alpha_i^2)^2 \eta \mathbb{E} \left[\left \| \nabla F(\bm{w}^{(t)}) \right\|^2  \right]\nonumber.
\end{align}
Re-arranging the terms, summing from $t=1$ to $T$, and dividing both sides with $T$ yields:
\begin{align}
  &\frac{1}{T} \sum_{t=1}^T \mathbb{E} \left[\left \| \nabla f_i (\hat{\bm{v}}^{(t)}_i) \right\|^2\right] \nonumber \\
  & \leq \frac{2\mathbb{E} \left[ f_i(\hat{\bm{v}}^{(1)}_i) \right]}{\eta T}    +  \eta L  \left(\alpha_i^4 \sigma^2 + (1-\alpha_i)^2 \frac{\sigma^2}{n}\right) \nonumber\\
    &\quad + 4\alpha_i^4 (1-\alpha_i)^2 L^2\frac{1}{T} \sum_{t=1}^T \mathbb{E} \left[\left \|\bm{w}_i^{(t)} - \bm{w}^{(t)}\right\|^2  \right] +  2 (1-\alpha_i)^2 L^2 \frac{1}{n}\sum_{j=1}^n \frac{1}{T} \sum_{t=1}^T \mathbb{E} \left[\left \|\bm{w}_j^{(t)} - \bm{w}^{(t)}\right\|^2  \right]   \nonumber \\
     & \quad+ 16 (1-\alpha_i)^2 L^2 \frac{1}{T} \sum_{t=1}^T  \mathbb{E} \left[\left \|  \nabla F (\bm{w}^{(t)}) \right\|^2  \right]+ 8(1-\alpha_i^2)^2  \zeta_i  + 16(1-\alpha_i^2)^2 \Gamma  \nonumber,
\end{align} 
Then, plug in Lemma~\ref{lm: nonconvex lm2} and~\ref{lm: nonconvex deviation} :
\begin{align}
  &\frac{1}{T} \sum_{t=1}^T \mathbb{E} \left[\left \| \nabla f_i (\hat{\bm{v}}^{(t)}_i) \right\|^2\right] \nonumber \\
  & \leq \frac{2\mathbb{E} \left[ f_i(\hat{\bm{v}}^{(1)}_i) \right]}{\eta T}    +   \eta L \left(\alpha_i^4 \sigma^2 + (1-\alpha_i)^2 \frac{\sigma^2}{n}\right)+8(1-\alpha_i^2)^2  \zeta_i  +16(1-\alpha_i^2)^2 \Gamma  \nonumber\\
    & \quad + 4\alpha_i^4 (1-\alpha_i)^2 L^2\frac{1}{T} \sum_{t=1}^T \mathbb{E} \left[\left \|\bm{w}_i^{(t)} - \bm{w}^{(t)}\right\|^2  \right] + 2(1-\alpha_i)^2 L^2\frac{1}{n}\sum_{j=1}^n \frac{1}{T} \sum_{t=1}^T \mathbb{E} \left[\left \|\bm{w}_j^{(t)} - \bm{w}^{(t)}\right\|^2  \right]   \nonumber \\
     & \quad +16(1-\alpha_i)^2 L^2\left ( \frac{2}{\eta T}\mathbb{E} \left[ F(\bm{w}^{(1)}) \right]+  L^2\frac{1}{T}\sum_{t=1}^T  \frac{1}{n} \sum_{j=1}^n \mathbb{E} \left[\left \|\ \bm{w}^{(t)}_j - \bm{w}^{(t)}  \right \|^2 \right]+ \frac{\eta  L \sigma^2}{ n} \right) \nonumber\\
       & \leq \frac{2\mathbb{E} \left[ f_i(\hat{\bm{v}}^{(1)}_i) \right]}{\eta T}    + \eta L\left(\alpha_i^4 \sigma^2 + (1-\alpha_i)^2 \frac{\sigma^2}{n}\right) +8 (1-\alpha_i^2)^2  \zeta_i  +16(1-\alpha_i^2)^2 \Gamma \nonumber\\
    &\quad +  4\alpha_i^4 (1-\alpha_i)^2 L^2  \left[200L^2\tau^4\eta^4\left( \sigma^2 + \frac{\sigma^2}{n}+ \frac{\zeta}{n}  \right)  +  20\tau^2 \eta^2 \left( \sigma^2 + \frac{\sigma^2}{n} + \zeta_i \right)\right] +    180 \tau^2 \eta^2(1-\alpha_i)^2 L^2  \left( \sigma^2 + \frac{\sigma^2}{n} +\frac{\zeta}{n} \right)   \nonumber \\
     &\quad +16 (1-\alpha_i)^2 L^2\left ( \frac{2}{\eta T}\mathbb{E} \left[ F(\bm{w}^{(1)}) \right] + \frac{\eta  L \sigma^2}{ n} \right) \nonumber .
\end{align}
Plugging in $\eta = \frac{1}{2\sqrt{5}\sqrt{T} L }$ we can conclude the proof:
\begin{align}
  &\frac{1}{T} \sum_{t=1}^T \mathbb{E} \left[\left \| \nabla f_i (\hat{\bm{v}}^{(t)}_i) \right\|^2\right] \nonumber \\
       & \leq  O \left( \frac{ L }{\sqrt{T}}  \right)  +  \alpha_i^4 O\left(\frac{\sigma^2}{\sqrt{T}}\right)  + (1-\alpha_i)^2O \left( \frac{\sigma^2}{n\sqrt{T}}\right)+  (1-\alpha_i)^2 \left ( \frac{L}{ \sqrt{T}}  + \frac{ \sigma^2}{ n\sqrt{T}} \right)+  (1-\alpha_i^2)^2 \left (  \zeta_i + \Gamma \right) \nonumber\\
    & \quad +   \alpha_i^4 (1-\alpha_i)^2  O \left( \frac{  \tau^4\left( \sigma^2 + \frac{\sigma^2}{n}+ \frac{\zeta}{n}  \right) }{T^2}  + \frac{  \tau^2\left( \sigma^2 + \frac{\sigma^2}{n} + \zeta_i \right) }{T} \right) +  (1-\alpha_i)^2     O\left(\frac{\tau^2 (\sigma^2 + \frac{\sigma^2}{n} +\frac{\zeta}{n})}{T}  \right)   \nonumber.
\end{align}
\end{proof}

\end{document}